\documentclass[runningheads]{llncs}
\usepackage{graphicx}
\usepackage{comment}
\usepackage{amsmath,amssymb} % define this before the line numbering.
\usepackage{color}

\usepackage{bbm}
\usepackage{multirow}
\usepackage{grffile}
\usepackage{xcolor}
\usepackage{booktabs}
\usepackage{algorithm}
\usepackage{algorithmic}
\usepackage{subcaption}
\newcommand{\mE}{\mathbb{E}}
\newtheorem{prop}{Proposition}

\usepackage{enumitem}
\setlist{nolistsep}
\setlist{nosep}

\makeatletter
\newenvironment{breakablealgorithm}
{% \begin{breakablealgorithm}
	\begin{center}
		\refstepcounter{algorithm}% New algorithm
		\hrule height.8pt depth0pt \kern2pt% \@fs@pre for \@fs@ruled
		\renewcommand{\caption}[2][\relax]{% Make a new \caption
			{\raggedright\textbf{\ALG@name~\thealgorithm} ##2\par}%
			\ifx\relax##1\relax % #1 is \relax
			\addcontentsline{loa}{algorithm}{\protect\numberline{\thealgorithm}##2}%
			\else % #1 is not \relax
			\addcontentsline{loa}{algorithm}{\protect\numberline{\thealgorithm}##1}%
			\fi
			\kern2pt\hrule\kern2pt
		}
	}{% \end{breakablealgorithm}
	\kern2pt\hrule\relax% \@fs@post for \@fs@ruled
\end{center}
}

\begin{document}

\pagestyle{plain}
\mainmatter

\title{Large-Scale Optimal Transport via Adversarial Training with Cycle-Consistency} % Replace with your title

\author{Guansong Lu\thanks{Equal contribution} \and
	Zhiming Zhou$^\star$ \and
	Jian Shen \and 
	Cheng Chen \and
	Weinan Zhang \and
	Yong Yu}
\institute{
	Shanghai Jiao Tong University, Shanghai, China\\
	\email{$\{$gslu, heyohai, rockyshen$\}$@apex.sjtu.edu.cn}, \email{jack\_chen1990@sjtu.edu.cn},
	\email{$\{$wnzhang, yyu$\}$@apex.sjtu.edu.cn},
}

\maketitle

\begin{abstract}

Recent advances in large-scale optimal transport have greatly extended its application scenarios in machine learning. However, existing methods either not explicitly learn the transport map or do not support general cost function. In this paper, we propose an end-to-end approach for large-scale optimal transport,  which directly solves the transport map and is compatible with general cost function. It models the transport map via stochastic neural networks and enforces the constraint on the marginal distributions via adversarial training. The proposed framework can be further extended towards learning Monge map or optimal bijection via adopting cycle-consistency constraint(s). We verify the effectiveness of the proposed method and demonstrate its superior performance against existing methods with large-scale real-world applications, including domain adaptation, image-to-image translation, and color transfer. 


\end{abstract}

\section{Introduction}

The idea of optimal transport (OT), dating back to 1781 \cite{monge}, has recently been widely studied. It establishes an optimal mapping between two distributions with a minimal transportation cost. The optimal mapping obtained via OT has a wide range of applications such as domain adaptation \cite{courty2017optimal}, image-to-image translation \cite{ot-cyclegan} and color transfer \cite{ferradans2014regularized}. Besides, the minimal transportation cost can be used to measure the distance between two distributions \cite{rubner2000earth}. 

Given the cost of moving one unit of mass from one point to another, the Monge formulation of OT aims to find an optimal deterministic mapping (each source has a single target), named Monge map, to transport the source distribution to the target distribution with a minimal transportation cost. However, the Monge formulation is not always feasible. The Kantorovich formulation relaxes the Monge formulation and minimizes the overall cost over transport plan, which is a joint distribution whose marginals are the source and target distributions respectively. Such a transport plan allows each point to be mapped to multiple targets and thus is no longer a deterministic mapping in general. 

Kantorovich formulation can be solved as a linear programming \cite{network_simplex}. However, linear programming solvers usually suffer from high time complexity. To solve large-scale OT efficiently, recent approaches turn to stochastic gradient algorithms. Among them, \cite{largescale_semidual,wgan} focused on solving the minimal transportation cost but failed to address the problem of learning an optimal mapping between the two distributions. \cite{spot} learns two mappings from a latent distribution to the two data distributions respectively. It enables sampling from the optimal transport plan, but the mapping between the two data distributions is not explicitly learned, and how a source is mapped to the target is hard to retrieve. \cite{large-scale} proposed a two-step approach to learn a transport map: first solves an optimal transport plan (in terms of the density of the joint distribution) and then fits the barycentric mapping of the optimal transport plan with a deep neural network. However, the pushforward distribution of such an estimated map may not well align with the target distribution and tends to have out of distribution samples, which leads to blurry results and poor performance in applications like domain adaptation and color transfer. 

In this paper, we propose an end-to-end approach for large-scale optimal transport. Given the source and target distribution, we directly model the transport map via a stochastic neural network, which learns how to map the samples in the source distribution to the target distribution. We use adversarial training to ensure that the pushforward distribution matches the target distribution, and at the same time, we minimize the overall transportation cost of the learned map. The training of this framework simply requires standard back-propagation and thus is compatible with any differentiable cost function. Moreover, directly modeling the transport map benefits applications like image-to-image translation and color transfer that require direct access to the transport map. 

Cycle-consistency is a prevalent idea in the domain of image translation. And we find that under the circumstance of OT, and within our framework, cycle-consistency is also very useful. The first interesting finding is that one-side cycle-consistency can softly link the Kantorovich formulation and the Monge formulation: introducing cycle-consistency constraint on and only on target samples in our aforementioned framework, which is Kantorovich-based, will regularize the model towards learning an optimal deterministic mapping, i.e., a Monge map, where each sample tends to have a single target. Note that although the strict Monge formulation is not always feasible, such a regularization-based formulation is generally feasible. On the other hand, the commonly used two-side cycle-consistency, if introduced, can further regularize the model towards establishing an optimal bijection, which means a bijection with minimized transportation cost. That is, when the proposed framework is combined with two-side cycle-consistency, it can be viewed as an end-to-end version of OT-CycleGAN \cite{ot-cyclegan}. 

The contributions of this paper can be summarized as follows:
\begin{itemize}
	\item We propose an end-to-end framework for solving large-scale optimal transport based on the Kantorovich formulation, which directly models the transport map and is compatible with general cost function. 
	\item We show that one-side cycle-consistency constraint can regularize the 
	
	Kantorovich-based formulation towards learning a deterministic mapping and thus build a soft link between Monge formulation and Kantorovich one. 
	\item Furthermore, if two-side cycle-consistency is introduced in the proposed OT framework, it will regularize the model towards learning an optimal bijection between the two data distributions with the transportation cost minimized, which can be viewed as an end-to-end version of OT-CycleGAN \cite{ot-cyclegan}. 
	\item We verify the effectiveness and demonstrate the superior performance of the proposed method with large-scale real-world applications, including domain adaptation, image-to-image translation and color transfer. 
\end{itemize}

\section{Preliminaries}

\subsection{Monge Formulation}

Given two distributions $\mu$ and $\nu$ defined in domain $X$ and $Y$ respectively, and a cost function $c: X \times Y \rightarrow \mathbb{R}^+$, the Monge formulation aims to find a deterministic mapping $f: X \rightarrow Y$ which transports the mass of the source distribution $\mu$ to the target distribution $\nu$ with the minimal transportation cost:
\begin{equation}
\label{monge_formulation}
\inf_{f: f\#\mu = \nu} \mathbb{E}_{x \sim \mu} [c(x,f(x))],
\end{equation}
where $f\#\mu = \nu$ denotes that the pushforward measure of $\mu$ under $f$ is $\nu$. The optimal mapping with the minimal transportation cost is referred to as a Monge map. Brenier \cite{brenier1991polar} proved the existence and uniqueness of the Monge map when the source distribution $\mu$ is continuous and $c=\|x-y\|^2$. The result was later generalized to more general cost, e.g., strictly convex and super linear \cite{ot_applymath}. However, in many cases, such a Monge map does not exist \cite{large-scale} and thus it is not always feasible to solve the Monge formulation.

\subsection{Kantorovich Formulation}

To make the OT problem generally feasible, Kantorovich \cite{kantorovich} relaxed the Monge formulation by optimizing over a joint distribution, where each source can be mapped to multiple targets with different probabilities, rather than deterministically mapped to a single target:
\begin{equation}
\label{kantorovich_formulation}
\inf_{\pi \in \prod(\mu, \nu)} \mathbb{E}_{(x,y) \sim \pi} [c(x,y)],
\end{equation}
where $\prod(\mu, \nu)$ denotes the set of all joint distributions defined on $X \times Y$ with the marginal distributions being $\mu$ and $\nu$ respectively, i.e., $\mathbb{E}_{y \sim \pi(x, y)} \mathbbm{1} = \mu(x)$, $\mathbb{E}_{x \sim \pi(x, y)} \mathbbm{1} = \nu(y)$. Such a joint distribution $\pi$ is called a transport plan, and the one with the minimal transportation cost is referred to as the optimal transport plan. With a transport plan, each source point $x$ is transported to possibly more than one target points $y$ according to the conditional distribution $\pi(y|x)$, which means a transport plan implies a stochastic mapping. 

\section{Kantorovich Solver}

Solving the OT problem in the Kantorovich formulation requires to model and optimize over the transport plan, which is a joint distribution $\pi(x, y)$ whose marginal distributions are $\mu(x)$ and $\nu(y)$ respectively. To achieve this, we model $\pi(x,y)$ via parameterizing the conditional distribution $\pi(y|x)$ as a stochastic neural network and enforce the constraint on the marginal distributions via adversarial training. 

\subsection{Stochastic Neural Network}

In the Kantorovich formulation, we only need to sample from the joint distribution $\pi(x,y)$ and does not need to access the probability densities. Given that we can easily sample from $\mu(x)$, we propose to model $\pi(x,y)$ via parameterizing the conditional distribution $\pi(y|x)$. Once $\pi(y|x)$ is modeled, sampling from the joint distribution $\pi(x,y)$ boils down to sampling $x$ from the marginal distribution $\mu$ and then sampling $y$ from the conditional distribution $\pi(y|x)$. 

Due to the excellent expression ability of deep neural networks, nowadays in deep learning, we usually model and optimize functions as deep neural networks. However, typically, deep neural network is deterministic. Hence, we introduce randomness into the neural network to achieve its ability of modeling stochasticity. Specifically, we model the mapping function as neural network $G_{xy}$ and augment the input with an independent random noise $z$, which gives $G_{xy}(x, z)$. The random noise follows some given distribution, e.g., Uniform or Gaussian. Then given a source sample $x$, the sampling procedure for conditional distribution $\pi(y|x)$ can be carried out through the following process:
\begin{equation}
y = G_{xy}(x, z), ~ z \sim p(z).
\end{equation}
For simplicity, we omit the input $z$ in the notation and use $G_{xy}(x)$ to denote the stochastic mapping from $x$ to $y$ and the distribution of $y$ conditioned on $x$. 

With $\pi(y|x)$ modeled as a stochastic neural network $G_{xy}(x)$, the OT problem defined in Eq. \eqref{kantorovich_formulation} can be rewritten as: 
\begin{equation}
\label{reformulate_hardconstraint}
\inf_{G_{xy}\#\mu = \nu} L_{opt}(G_{xy}) = \mathbb{E}_{x \sim \mu, y \sim G_{xy}(x)} [c(x,y)].
\end{equation}

\subsection{Adversarial Training}

The constraint in Eq. \eqref{kantorovich_formulation} requires the marginal distributions of the joint distribution $\pi$ to be $\mu$ and $\nu$. Since we can guarantee the source distribution is $\mu$, after the above reformulation, the constraint in Eq. \eqref{reformulate_hardconstraint} now becomes the pushforward measure of $\mu$ under $G_{xy}$ should be $\nu$. 

Imposing a constraint over a pushforward distribution can be difficult, but fortunately, recent progress of GANs \cite{gan} provide an efficient framework to minimize the divergence between two distributions. We therefore leverage GANs and enforce the constraint in Eq. \eqref{reformulate_hardconstraint} via an additional adversarial optimization over $G_{xy}$. It minimizes the divergence between the pushforward measure $G_{xy}\#\mu$ and the target distribution $\nu$. When the global minimum of such adversarial training is achieved, the two distributions would be identical. 

In GANs \cite{gan}, besides the mapping function $G_{xy}$ which can be regarded as the generator, there requires another function $D_{y}$ called the discriminator (also named critic). The discriminator tries to distinguish the samples from the target distribution $\nu$ and the samples from the generator $G_{xy}$, while the generator aims to map the source distribution $\mu$ to the target distribution $\nu$ such that the discriminator cannot distinguish which distribution a sample is from. GANs is usually formulated as a minimax game as below:
\begin{equation}
\begin{aligned}
\min\limits_{G_{xy}} \max\limits_{D_{y}} L_{gan} (G_{xy},D_{y}). 
\end{aligned}
\end{equation} 

The vanilla GANs \cite{gan} formulates $L_{gan} (G_{xy},D_{y})$ such that it equivalently minimizes the Jensen–Shannon (JS) divergence between $\nu$ and $G_{xy}\#\mu$. However, it has later been shown that optimizing the JS divergence will lead to various training issues. WGAN \cite{wgan} proposed to use the Wasserstein-1 distance instead and achieved superior training stability and results. Hence, we adopt the WGAN and define $L_{gan} (G_{xy},D_{y})$ as: 
\begin{equation}
\begin{aligned}
L_{gan} (G_{xy},D_{y}) &={\mE}_{y \sim \nu}[D_{y}(y)] -{\mE}_{x \sim \mu}[D_{y}(G_{xy}(x))],
\end{aligned}
\end{equation}
where $D_{y}$ is required to be 1-Lipschitz. To impose the Lipschitz constraint, we adopt the WGAN-GP \cite{wgangp} which introduces gradient penalty on $D_{y}$:
\begin{equation}
L_{gp}(D_{y})= {\mE}_{y \sim P_{\tilde{y}}} \, [ (\lVert \nabla D_{y}(y) \rVert_2 -1)^2],
\end{equation}
where $P_{\tilde{y}}$ is the distribution of uniformly distributed linear interpolations between $y \sim \nu$ and $y' \sim G_{xy}\#\mu$.

With the constraint of $G_{xy}\#\mu=\nu$ imposed via the adversarial training, the Kantorovich problem now becomes:
\begin{equation}
\begin{aligned}
\min\limits_{G_{xy}} \max\limits_{D_{y}} L_{opt}(G_{xy}) + \lambda_{gan} L_{gan}(G_{xy},D_{y}) + \lambda_{gp} L_{gp}(D_{y}),
\end{aligned}
\end{equation}
where $\lambda_{gan}$ and $\lambda_{gp}$ are the coefficients for GAN loss and gradient penalty respectively. 

The model can be trained by the standard back-propagation algorithm with two iterative steps similar to the GANs' training: first train $D_y$ for $n_{critic}$ iterations and then train $G_{xy}$ for one iteration. 

\begin{figure}[!t]
	\centering
	\begin{subfigure}[b]{1.\linewidth}
	    \begin{subfigure}[b]{0.25\linewidth}
    		\centering
    		\includegraphics[width=0.9\linewidth]{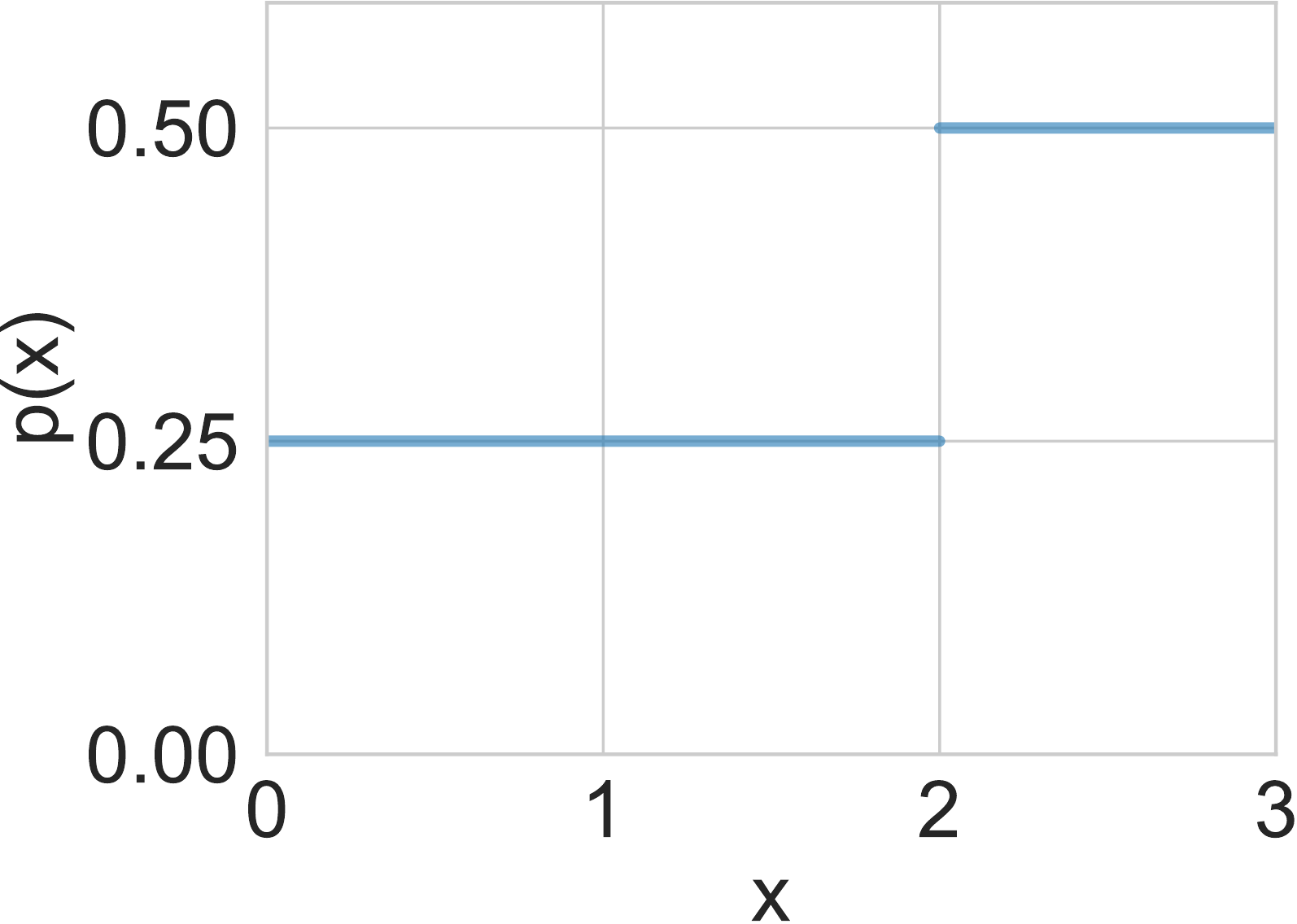}
    		\caption{Source }
    	\end{subfigure}%%
    	\begin{subfigure}[b]{0.25\linewidth}
    		\centering
    		\includegraphics[width=0.9\linewidth]{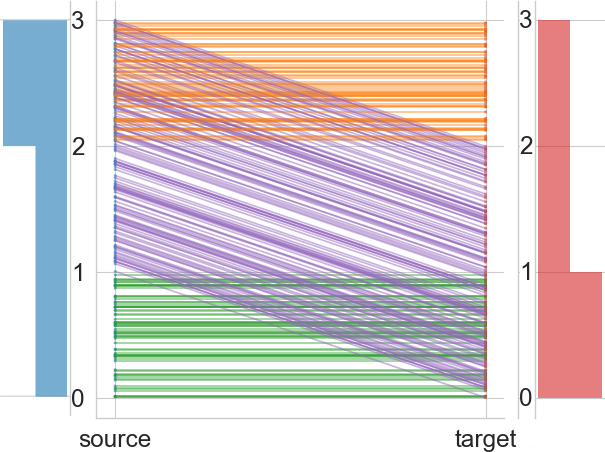}
    		\caption{Kantorovich}
    	\end{subfigure}%%
    	\begin{subfigure}[b]{0.25\linewidth}
    		\centering
    		\includegraphics[width=0.9\linewidth]{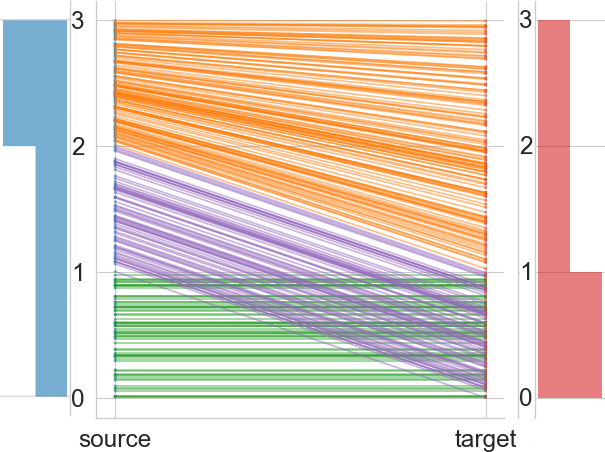}
    		\caption{Monge}
    	\end{subfigure}%%
    	\begin{subfigure}[b]{0.25\linewidth}
    		\centering
    		\includegraphics[width=0.9\linewidth]{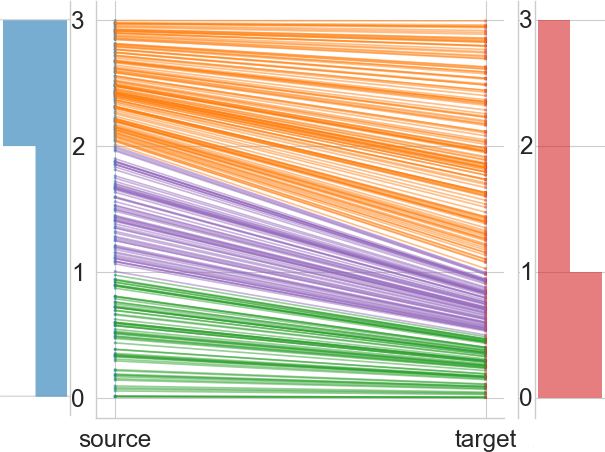}
    		\caption{Bijection}
    	\end{subfigure}%%
	\end{subfigure}
	\begin{subfigure}[b]{1.\linewidth}
	    \vspace{3pt} 
	    \begin{subfigure}[b]{0.25\linewidth}
    		\centering
    		\includegraphics[width=0.9\linewidth]{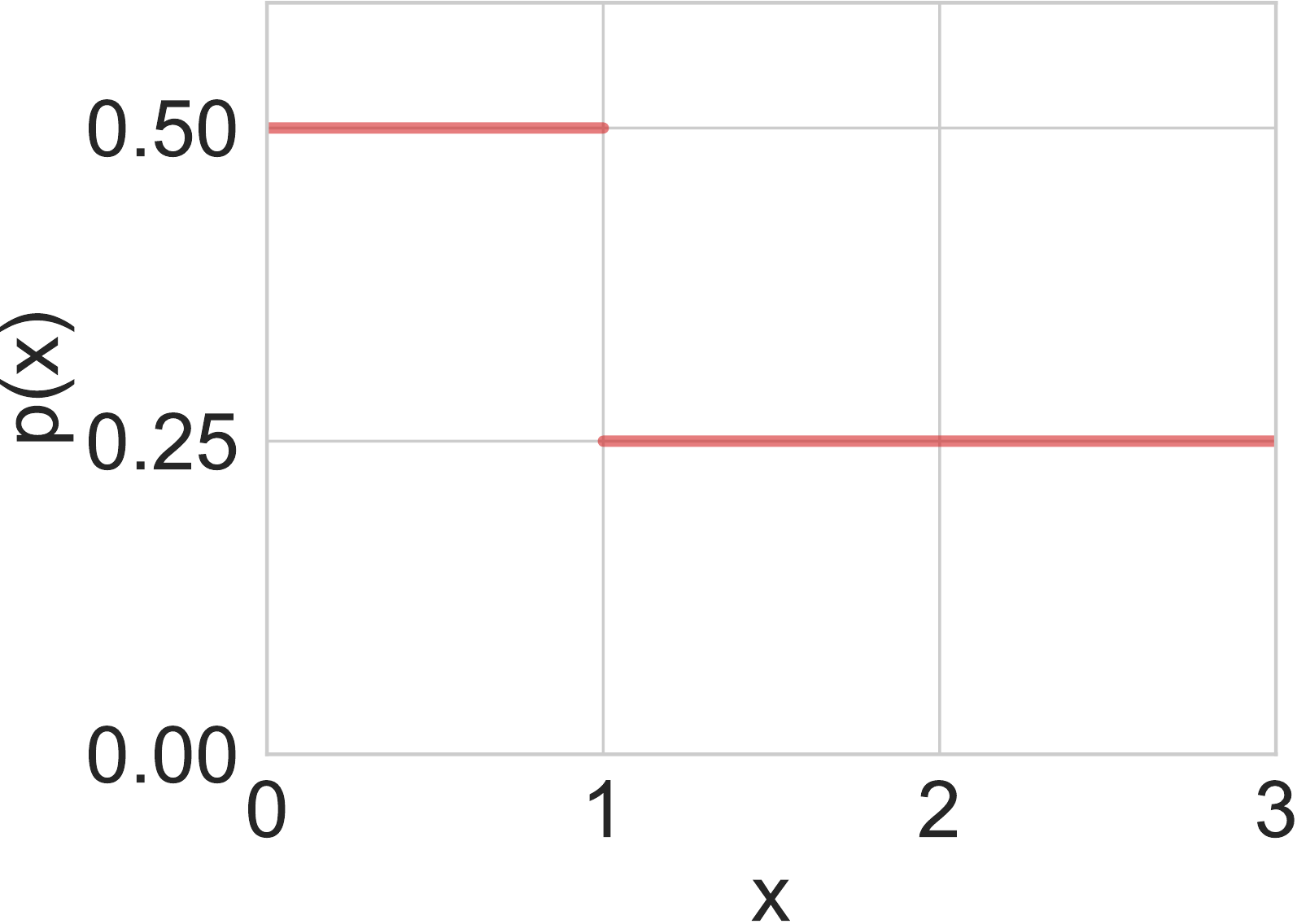}
    		\caption{Target }
    	\end{subfigure}%%
    	\begin{subfigure}[b]{0.25\linewidth}
    		\centering
    		\includegraphics[width=0.9\linewidth]{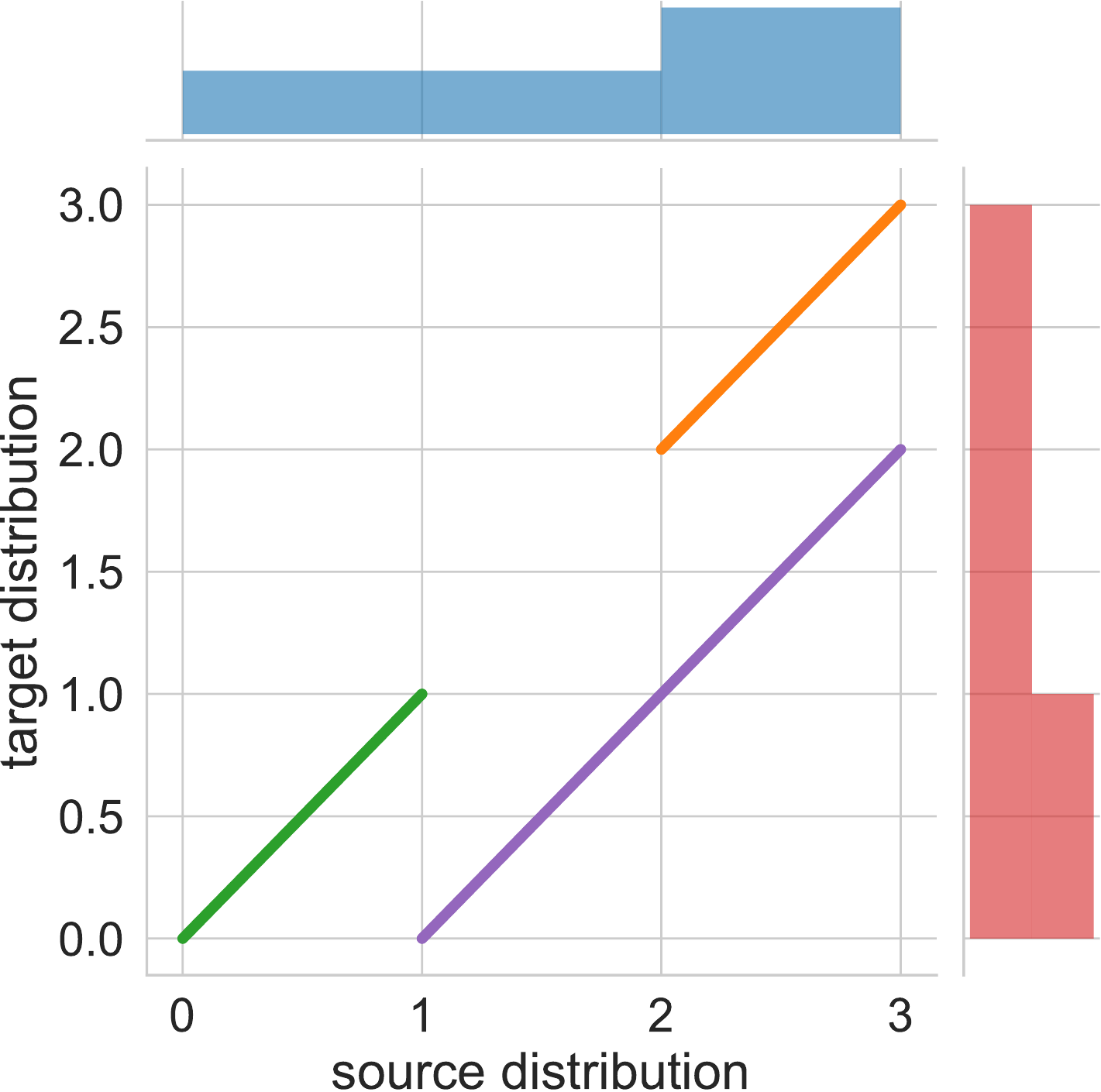}
    		\caption{Kantorovich}
    	\end{subfigure}%%
    	\begin{subfigure}[b]{0.25\linewidth}
    		\centering
    		\includegraphics[width=0.9\linewidth]{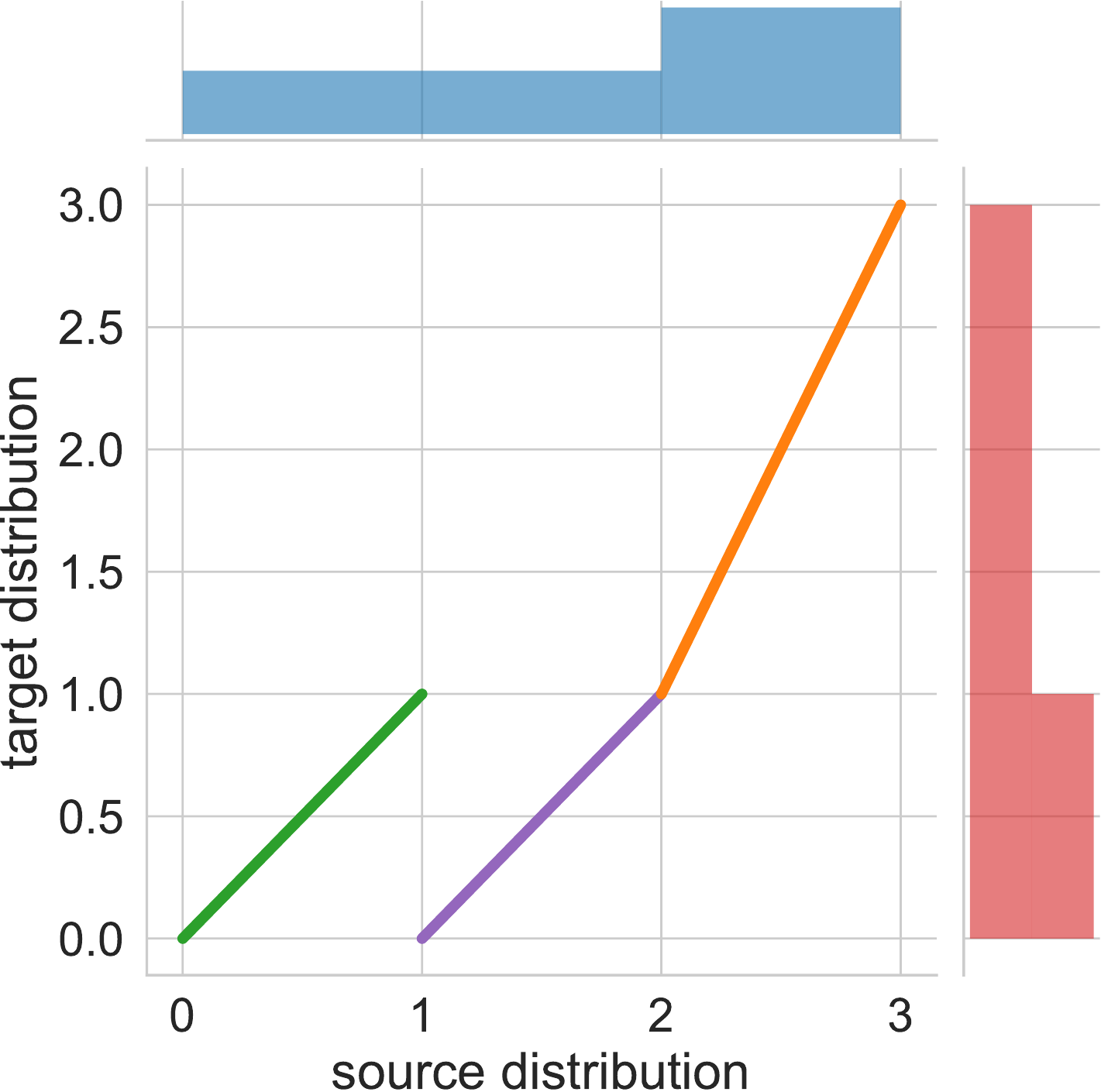}
    		\caption{Monge}
    	\end{subfigure}%%
    	\begin{subfigure}[b]{0.25\linewidth}
        	\centering
        	\includegraphics[width=0.9\linewidth]{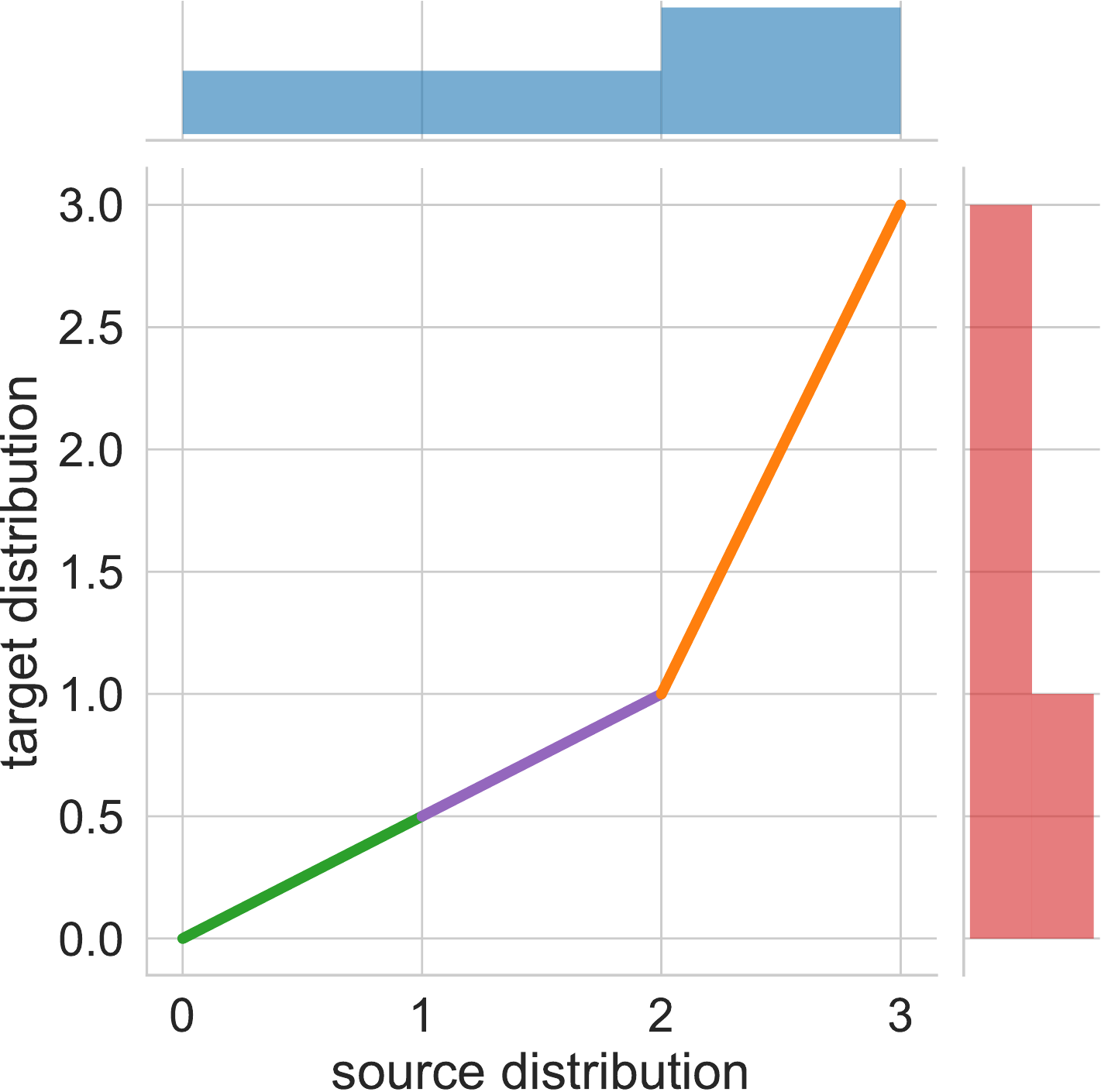}
        	\caption{Bijection}
    	\end{subfigure}%%
	\end{subfigure}
	\caption{Illustration for the difference between Kantorovich, Monge, and Bijection OT problem. (a) Source distribution; (e) Target distribution; $c(x, y) = \| x - y \|$. (b-d) are the line representations of optimal mapping of Kantorovich, Monge, and Bijection OT problem, respectively; (f-h) are the corresponding transport plan representations. A Kantorovich plan can be one-to-many and many-to-one, a Monge is required to be deterministic and nevertheless can be many-to-one, and a bijection is the most restrictive and further requires it to be one-to-one.} 
	\label{illustrate_monge_kantorovich}
	\vspace{-10pt}
\end{figure}

\section{Extensions with Cycle-Consistency}

Cycle-consistency deriving from the idea of dual learning \cite{duallearning} is first proposed for language translation to get rid of the requirement for paired examples and is later adopted for unsupervised image-to-image translation \cite{cyclegan,dualgan,disco}. Besides the mapping $G_{xy}$ from $X$ to $Y$, cycle-consistency introduces one more mapping $G_{yx}$ from $Y$ to $X$ and requires samples can be reconstructed after applying these two mappings sequentially, i.e., $G_{xy}(G_{yx}(y)) \approx y$ and $G_{yx}(G_{xy}(x)) \approx x$. 
The cycle-consistency loss can be formulated as:
\begin{equation}
\label{eq_one_side}
\begin{aligned}
L_{cycle}(\mu) & = \mathbb{E}_{x \sim \mu} \mathbb{E}_{y \sim G_{xy}(x)} \mathbb{E}_{\hat{x} \sim G_{yx}(y)}  [\|\hat{x} - x\|_2], \\
L_{cycle}(\nu) & = \mathbb{E}_{y \sim \nu} \mathbb{E}_{x \sim G_{yx}(y)} \mathbb{E}_{\hat{y} \sim G_{xy}(x)} [\|\hat{y} - y\|_2]. \end{aligned}
\end{equation} 

\subsection{One-Side Cycle-Consistency and Monge Solver}
\label{cycle-consistency loss}

In previous works \cite{cyclegan}, the two cycle-consistency constraints always appear together. Interestingly, we found that one-side cycle-consistency can regularize the Kantorovich solver towards learning a Monge map. Specifically, we have the following proposition: 
\begin{prop}
	\label{one-side cycle-consistency}
	Given two distributions $\mu$ and $\nu$ defined in domain $X$ and $Y$ respectively and two stochastic mappings $G_{xy}: X \rightarrow Y $ and $G_{yx}: Y \rightarrow X$. If $G_{yx}\#\nu = \mu$ and $L_{cycle}(\nu)=0$, then 
	\begin{enumerate}
		\item $G_{xy}$ becomes a deterministic mapping;
		\item $\forall ~ y_1$, $y_2$, if $y_1 \neq y_2$, then $p(G_{yx}(y_1)=G_{yx}(y_2))=0$.
	\end{enumerate}
\end{prop}

The formal proof is included in the Appendix. As an illustrative explanation, we show in Fig. \ref{illustrate_one_side} the two cases that will be punished by $L_{cycle}(\nu)$. Therefore, it can regularize the stochastic mapping $G_{xy}$ towards a deterministic mapping and regularize $G_{yx}$ away from mapping two different target samples to the same source sample. Symmetrical results can be obtained for the other side of cycle-consistency, i.e., $L_{cycle}(\mu)$. 

Based on Proposition \ref{one-side cycle-consistency}, we propose to adopt the one-side cycle-consistency $L_{cycle}(\nu)$ to further regularize our Kantorovich solver towards learning an optimal deterministic mapping, i.e., a Monge map. 

To apply the cycle-consistency constraint $L_{cycle}(\nu)$, we need to incorporate another mapping network $G_{yx}$. And to enforce $G_{yx}\#\nu = \mu$, i.e., the pushforward measure of $\nu$ under $G_{yx}$ is $\mu$, we introduce another critic network $D_{x}$ and train $G_{yx}$ adversarially as well. 

Combining all these components, we attain the overall objective for our Monge solver: 
\begin{equation}
\begin{aligned}
\min\limits_{G_{xy},G_{yx}} \max\limits_{D_{y},D_{x}} 
& \ L_{opt}(G_{xy}) + \lambda_{gan_{xy}} L_{gan}(G_{xy}, D_{y}) + \lambda_{gp_{xy}} L_{gp}(D_{y})\\
&  + \lambda_{cycle_{\nu}} L_{cycle}(\nu) + \lambda_{gan_{yx}} L_{gan}(G_{yx}, D_{x}) + \lambda_{gp_{yx}} L_{gp}(D_{x}), 
\end{aligned}
\end{equation}
where $\lambda_{cycle_{\nu}}$ is the coefficient for the cycle-consistency loss. For training, we iterative train $D_y$ and $D_x$ for $n_{critic}$ iterations and then train $G_{xy}$ and $G_{yx}$ for one iteration. 

\begin{figure}[!t]
	\centering
	\begin{subfigure}[b]{0.5\linewidth}
		\centering
		\includegraphics[width=0.7\linewidth]{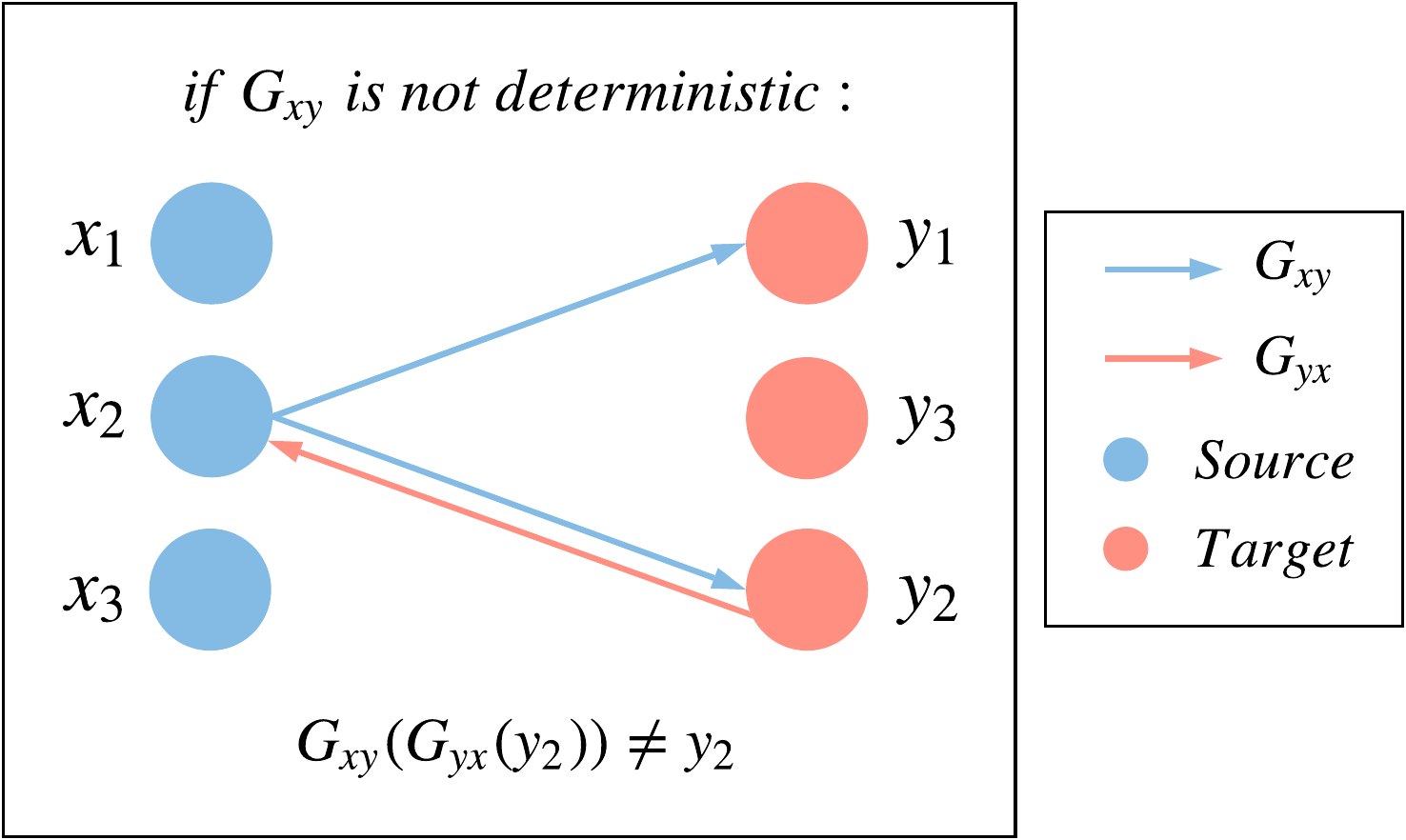}
		\label{one_side_cycle_deterministic} 
	\end{subfigure}%% 
	\begin{subfigure}[b]{0.5\linewidth}
		\centering
		\includegraphics[width=0.7\linewidth]{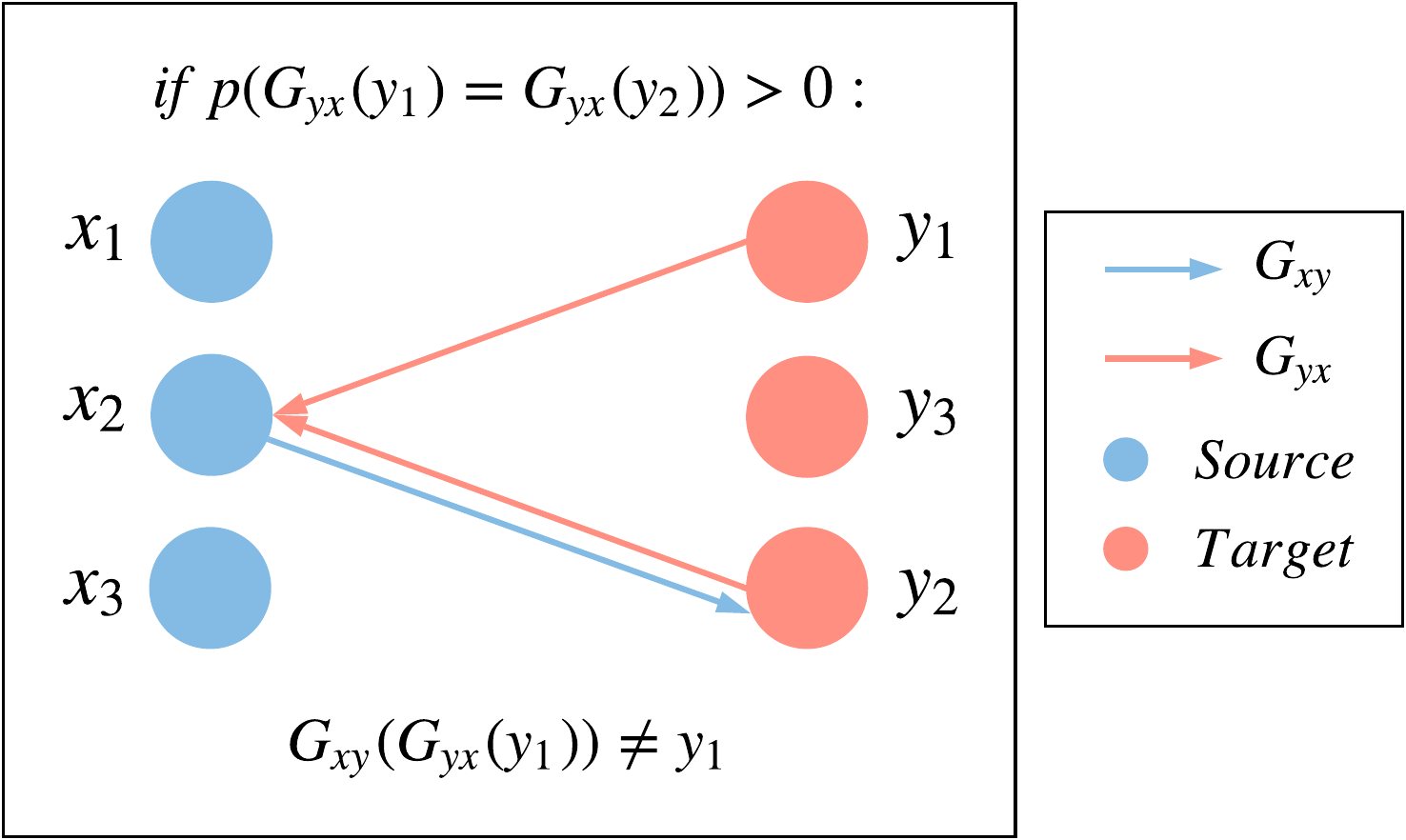}
		\label{one_side_cycle_not_collapse} 
	\end{subfigure}
	\caption{Illustration for Proposition \ref{one-side cycle-consistency}: the cases that will be punished by the one-side cycle-consistency constraint $L_{cycle}(\nu)$.}
	\label{illustrate_one_side}
	\vspace{-10pt}
\end{figure}

\subsection{Two-Side Cycle-Consistency and Optimal Bijection Solver}
\label{discussion two-side cycle-consistency}

It is known that two-side cycle-consistency can establish a one-to-one mapping (i.e., bijection) between two distributions \cite{cyclegan}. For completeness, we restate it as follows (the reorganized proof is also provided in the Appendix): 
\begin{prop}
	\label{two-side cycle-consistency}
	Given two distributions $\mu$ and $\nu$ defined in domain $X$ and $Y$ respectively and two stochastic mappings $G_{xy}: X \rightarrow Y$ and $G_{yx}: Y \rightarrow X$. If $G_{xy}\#\mu = \nu$, $G_{yx}\#\nu = \mu$, $L_{cycle}(\mu)=0$ and $L_{cycle}(\nu)=0$, then $G_{xy}$, $G_{yx}$ becomes bijections. 
\end{prop}

Two-side cycle-consistency can ensure the mapping is a bijection. However, the bijection between two distributions is generally not unique and in some applications one may prefer the bijection with best quality in some aspects.

Viewing the bijection as a transport between the two distributions, we can formulate the problem of seeking the bijection with best property as an optimal transport problem, where the transport plan/mapping is further required to be a bijection (just like Monge requires the mapping to be deterministic). We name such problem as optimal bijection transport (OBT). 

Actually, the problem of OBT has been considered by \cite{ot-cyclegan}. However, their method involves separated procedure to calculate the optimal transport and then use the solved optimal transport plan as a reference to train a CycleGAN, which falls in short in efficiency and accuracy. 

With our new perspective, to achieving the same goal, we can directly incorporate two-side cycle-consistency in our Kantorovich solver. As such, the overall objective of our optimal bijection solver is as follows: 
\begin{equation}
\begin{aligned}
\min\limits_{G_{xy},G_{yx}} \max\limits_{D_{y},D_{x}} 
& \ L_{opt}(G_{xy}) \\
& + \lambda_{cycle_{\mu}} L_{cycle}(\mu) + \lambda_{gan_{xy}} L_{gan}(G_{xy}, D_{y}) + \lambda_{gp_{xy}} L_{gp}(D_{y}) \\
& + \lambda_{cycle_{\nu}} L_{cycle}(\nu) + \lambda_{gan_{yx}} L_{gan}(G_{yx}, D_{x}) + \lambda_{gp_{yx}} L_{gp}(D_{x}),
\end{aligned}
\end{equation}
where $\lambda_{cycle_{\mu}}$ is the coefficient for the cycle-consistency loss on $\mu$. 

\subsection{Discussion}

So far, we have dealt with three OT problems that have different levels of restrictions on the transport plan. The Kantorovich formulation is relatively free and the transport plan can be one-to-many and many-to-one. The Monge formulation requires the source-to-target map to be deterministic (i.e., not one-to-many), but the map still can be many-to-one. And OBT further requires the map being bijection (i.e., one-to-one). 

In the objectives, the essential difference lies in they have no cycle-consistency, one-side cycle-consistency, or two-side cycle-consistency. And other components are the support of OT and cycle-consistency. 

Practically, we need to formulate the problem as the most suitable version of OT problem and use the corresponding solvers. For example, when only a source-to-target deterministic transport is required, a Monge solver might be sufficient and a bijection can be unnecessary and hence over constrained. And if a deterministic mapping does not benefit or it actually prefers stochastic mapping, we may simply use the Kantorovich solver. 

Finally, we should note that both our Monge solver and optimal bijection solver are regularization-based methods. There exists a trade-off between the OT objective and the cycle-consistency constraint, which can be tuned via $\lambda_{cycle}$. 

\section{Experiments}
\label{experiments}

In this section, we validate and study the performance of the proposed solvers. Synthetic experiments and large-scale real-world applications including domain adaptation, image-to-image translation, and color transfer, are considered. 

\subsection{Synthetic Experiments}
\label{toy_experiments}

\begin{figure}[!t]
	\centering
	\begin{subfigure}[b]{0.5\linewidth}
		\centering
		\hspace{25pt}
		\includegraphics[width=0.6\linewidth]{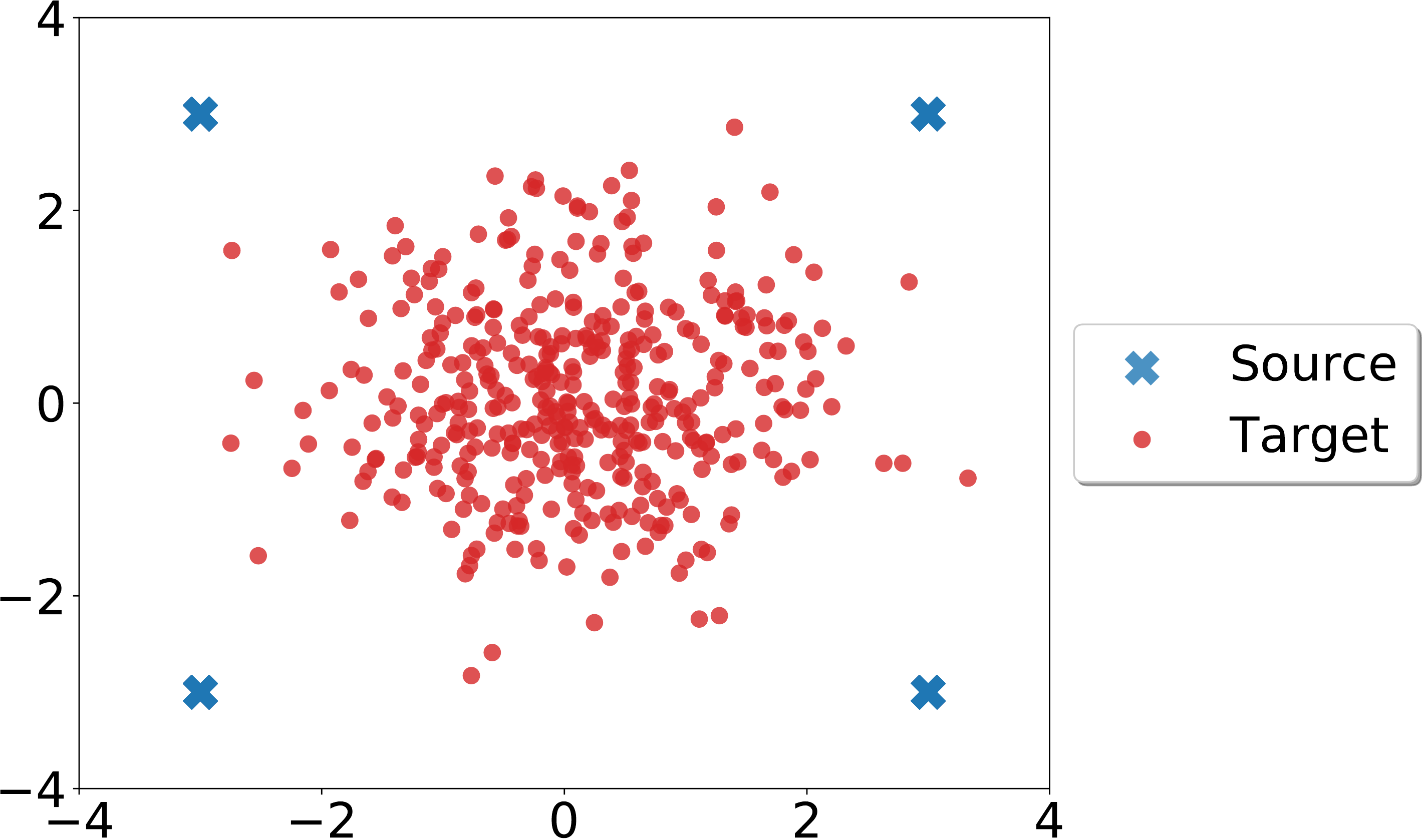} 
		\caption{Source and target samples} 
		\label{discrete2continuous_sample}
	\end{subfigure}%% 
	\begin{subfigure}[b]{0.5\linewidth}
		\centering
		\hspace{25pt}
		\includegraphics[width=0.6\linewidth]{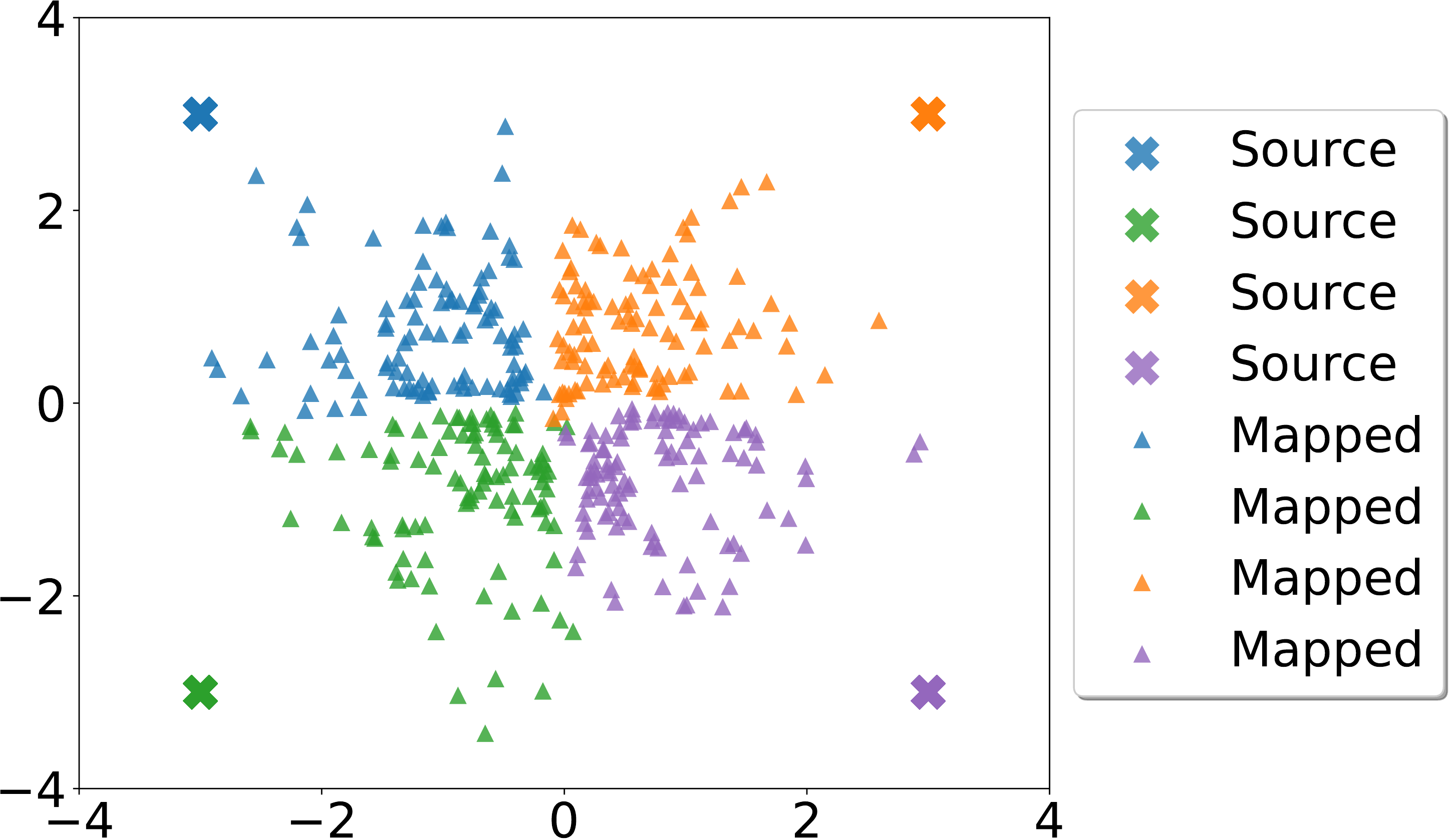} 
		\caption{Source and Mapped samples} 
		\label{discrete2continuous_mapping} 
	\end{subfigure}%% 
	\vspace{-3pt} 
	\caption{Verifying the Kantorovich solver. With stochastic neural network, each source samples are mapped stochastically to multiple samples. The target distribution is well recovered with adversarial training.} 
	\label{discrete2continuous_ot} 
\end{figure}

\begin{figure}[!t]
	\begin{subfigure}[b]{0.33\linewidth}
		\centering
		\includegraphics[width=0.6\linewidth]{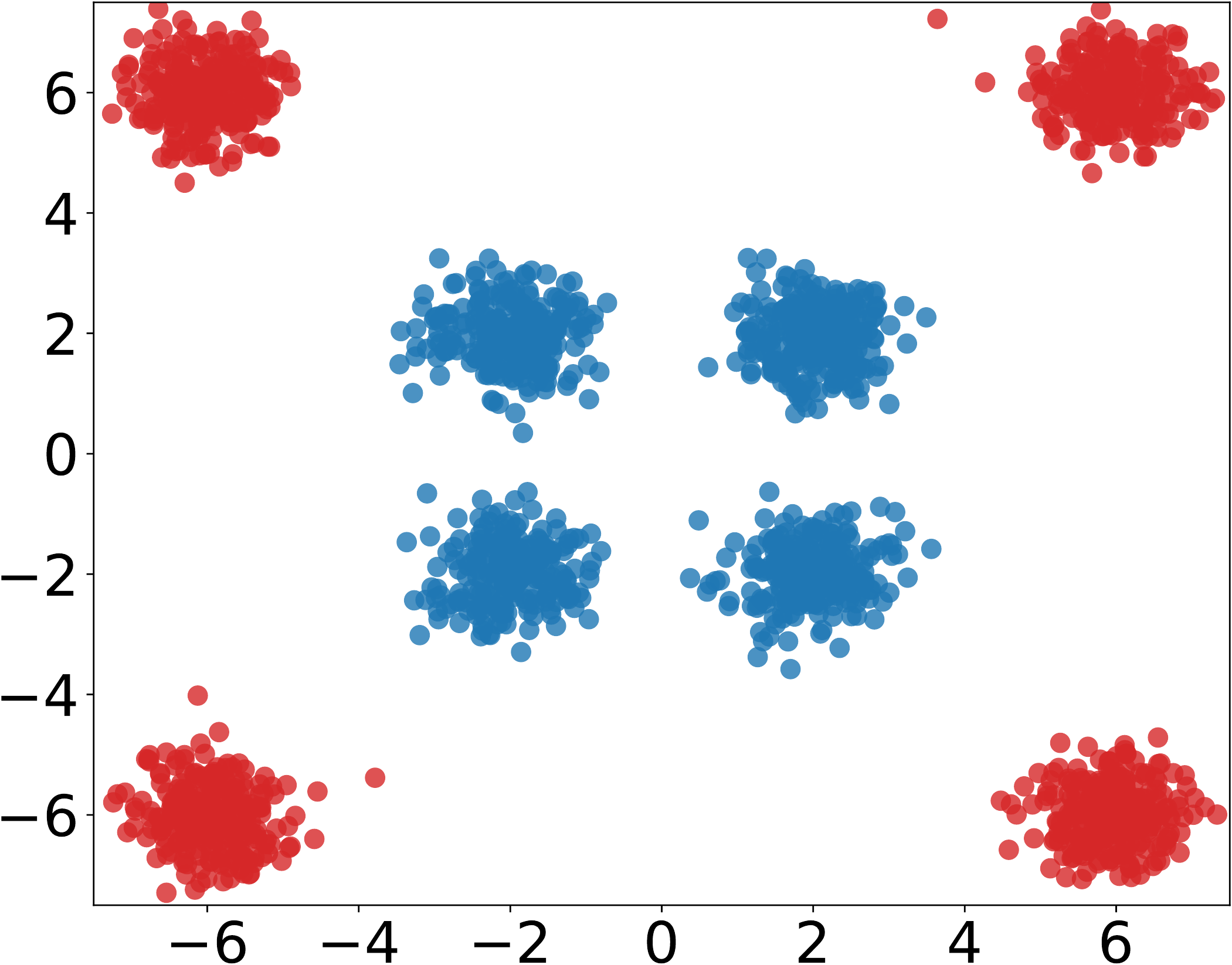} 
	\end{subfigure}%% 
	\begin{subfigure}[b]{0.33\linewidth}
		\centering
		\includegraphics[width=0.6\linewidth]{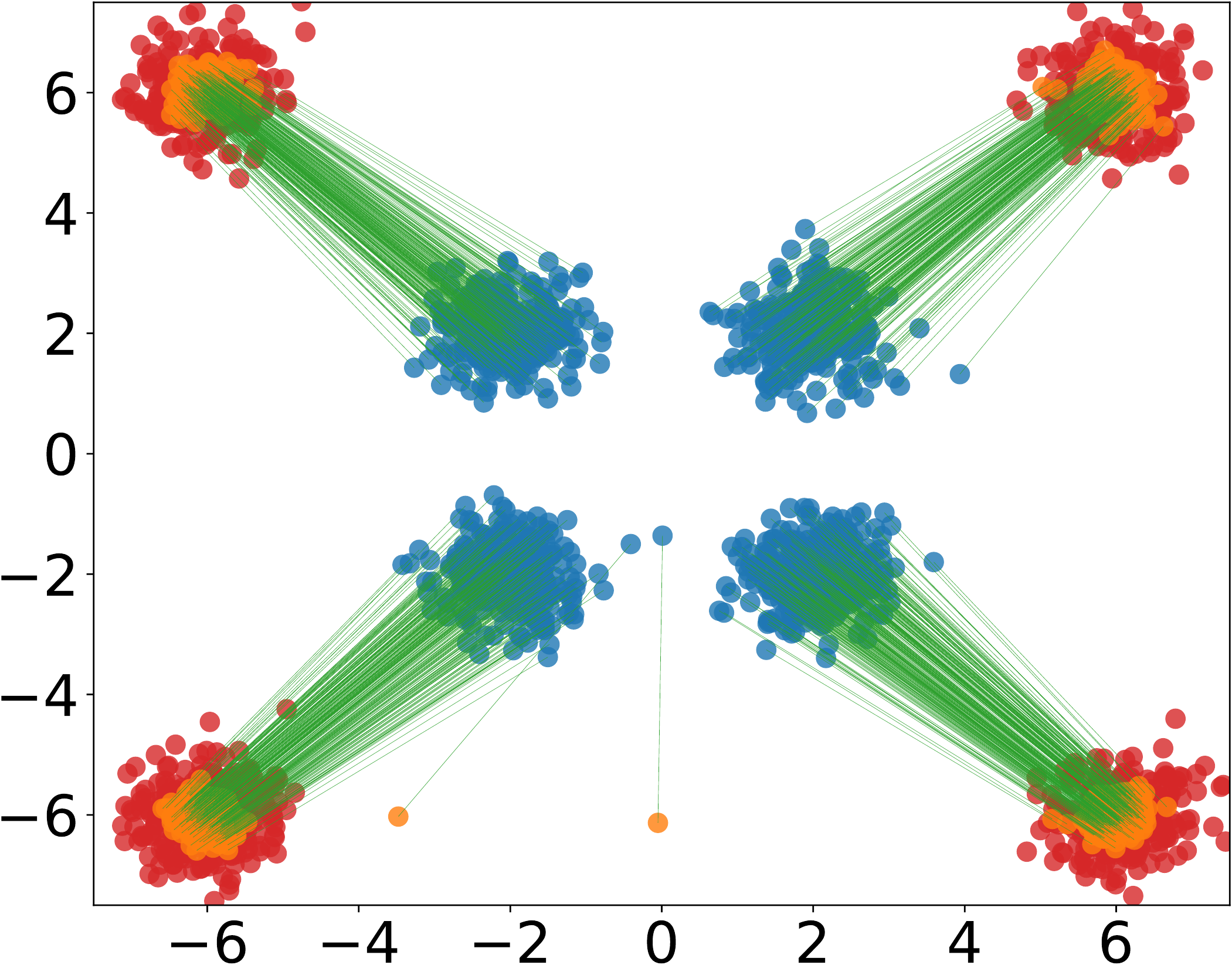} 
	\end{subfigure}%% 
	\begin{subfigure}[b]{0.33\linewidth}
		\centering
		\includegraphics[width=0.6\linewidth]{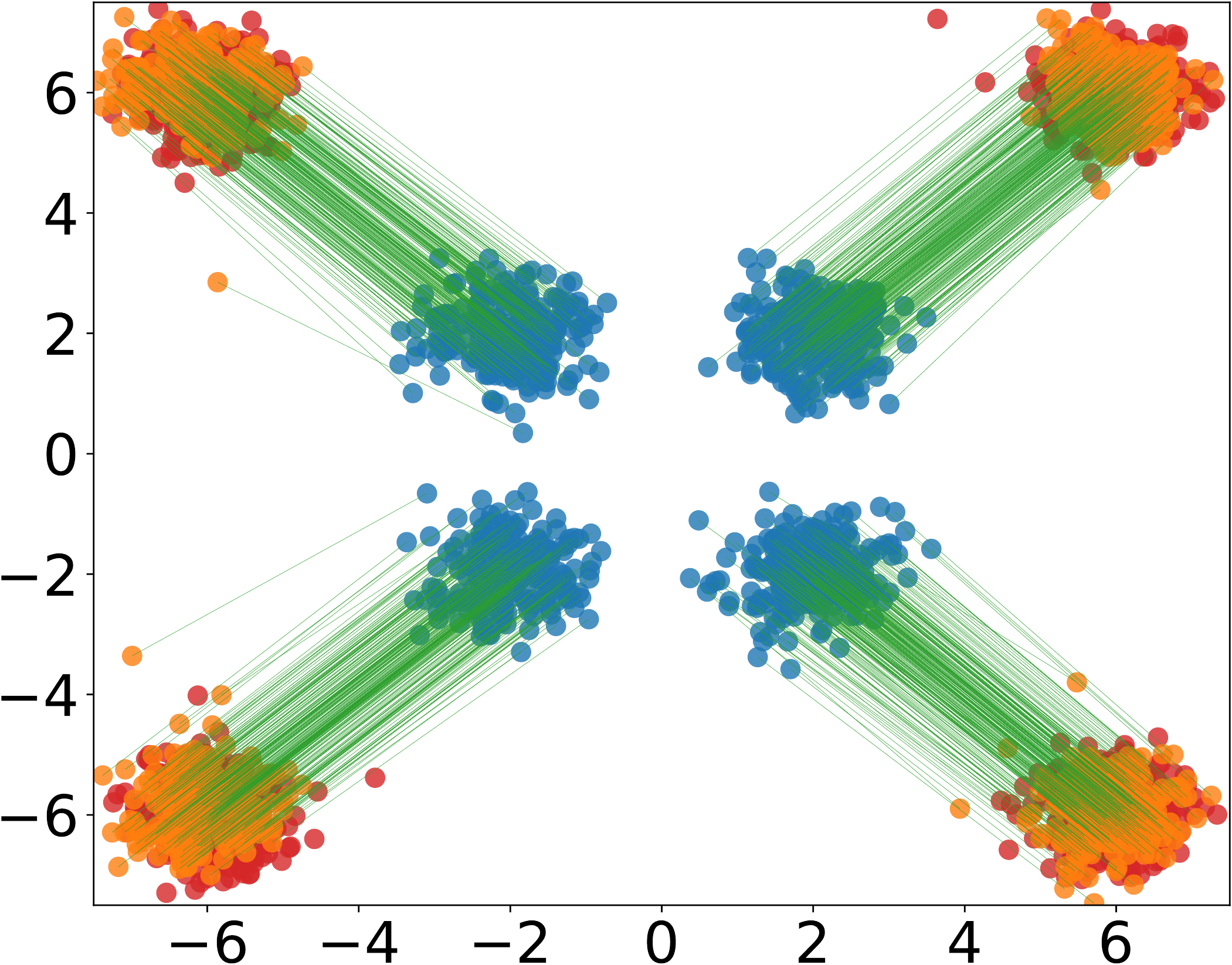} 
	\end{subfigure}%% 	
	
	\begin{subfigure}[b]{0.33\linewidth}
		\centering
		\includegraphics[width=0.6\linewidth]{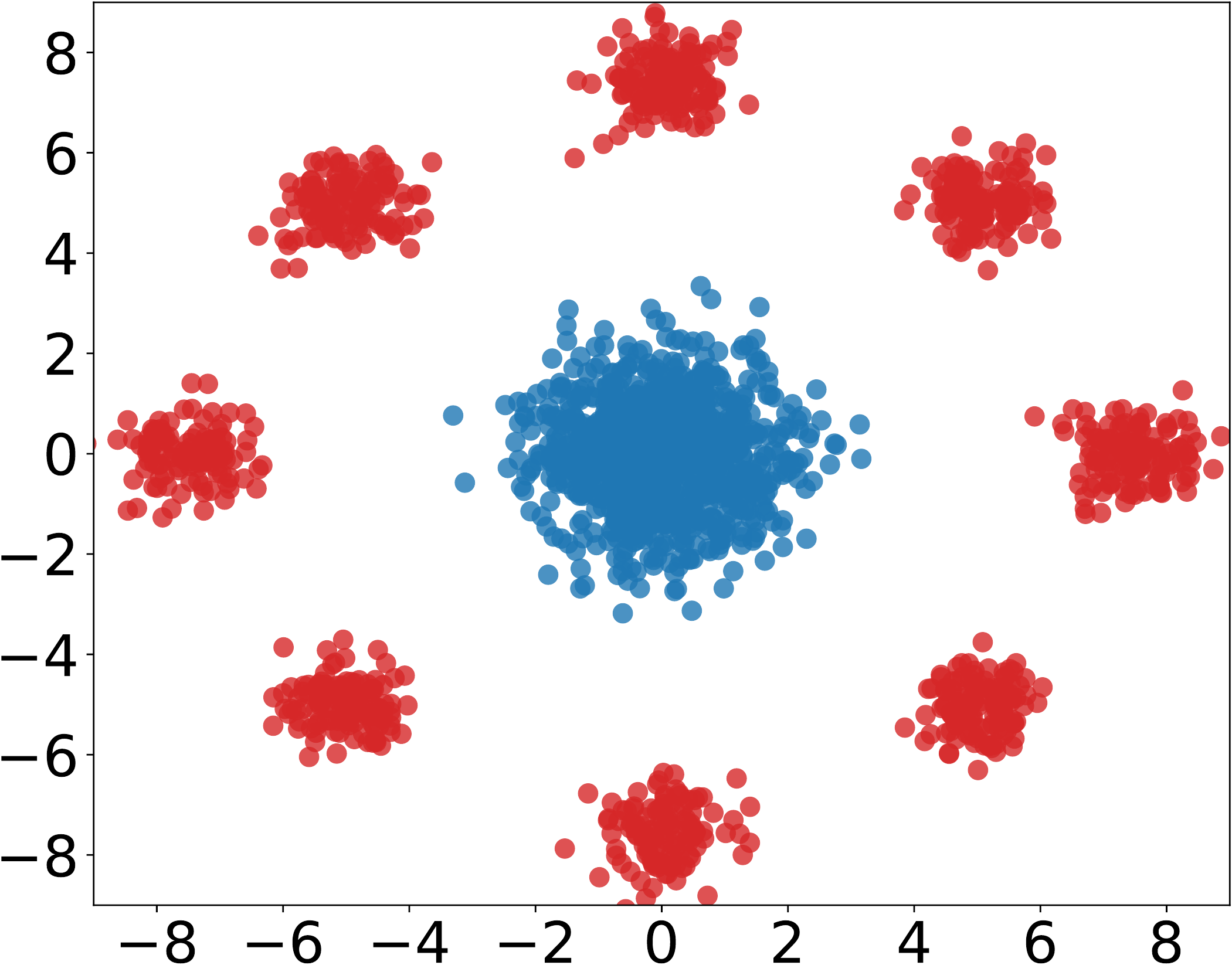} 
	\end{subfigure}%% 
	\begin{subfigure}[b]{0.33\linewidth}
		\centering
		\includegraphics[width=0.6\linewidth]{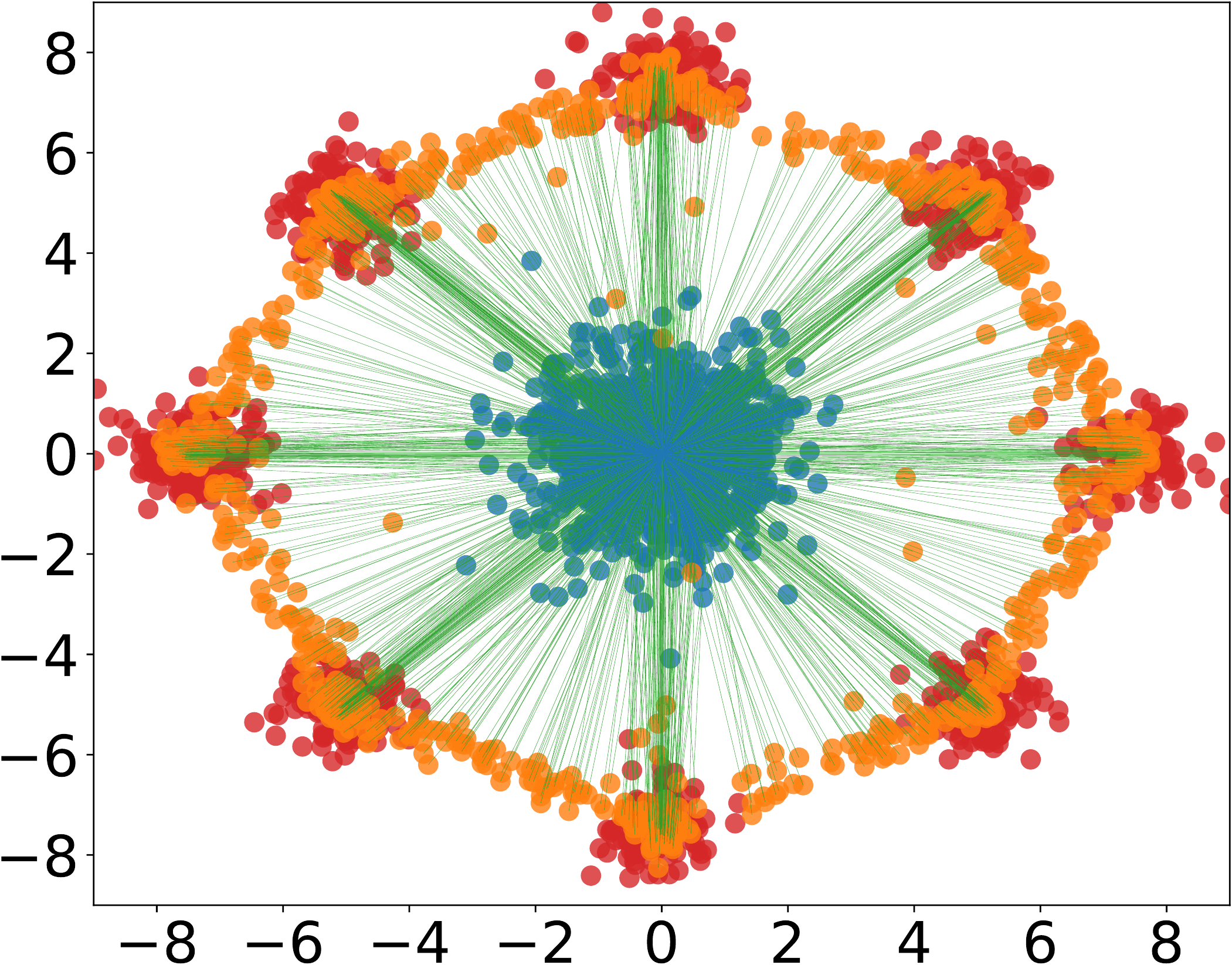} 
	\end{subfigure}%% 
	\begin{subfigure}[b]{0.33\linewidth}
		\centering
		\includegraphics[width=0.6\linewidth]{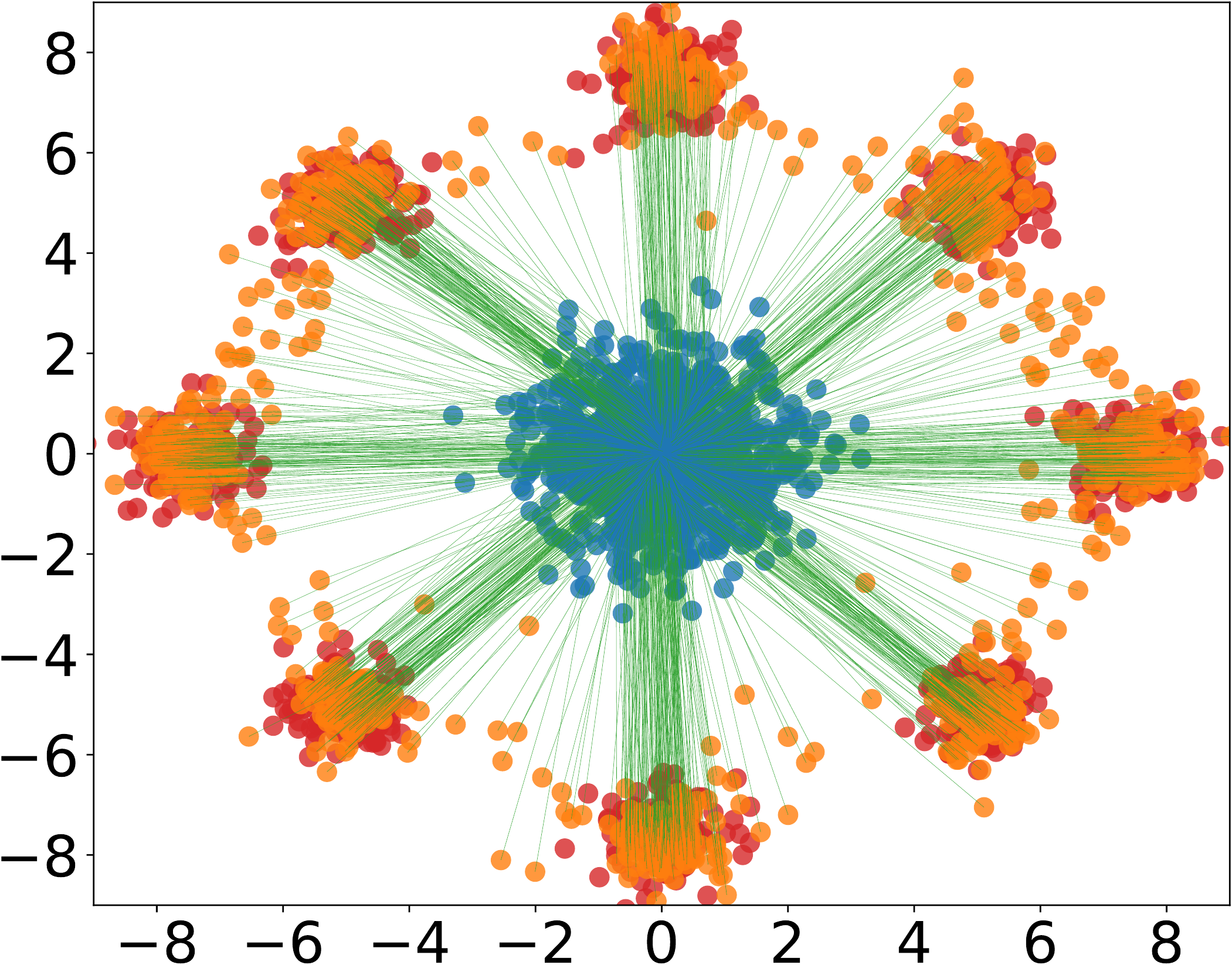} 
	\end{subfigure}%% 	
	
	\begin{subfigure}[b]{0.33\linewidth}
		\centering
		\hspace{-5pt}
		\includegraphics[width=0.6\linewidth]{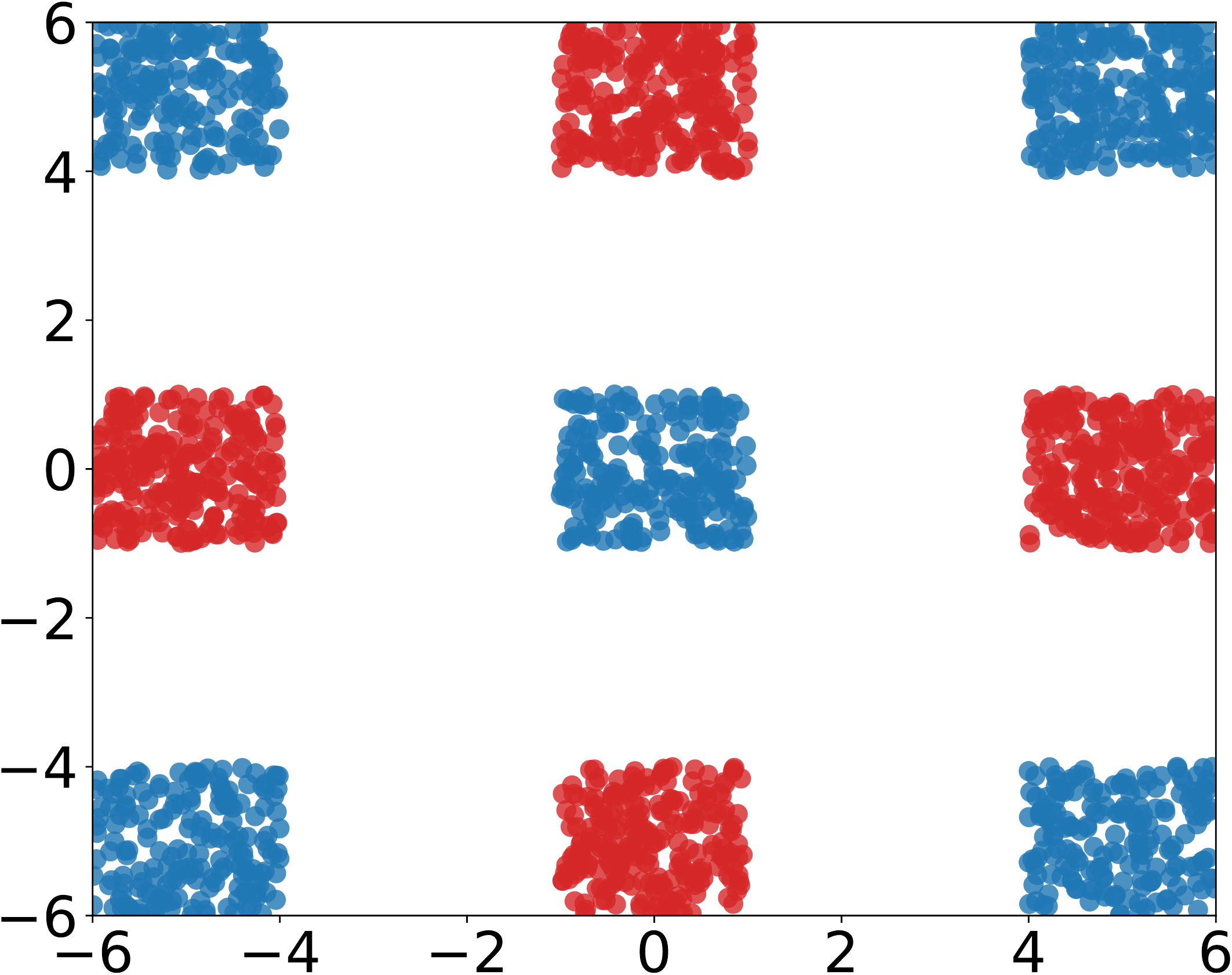} 
		\caption{Samples} 
	\end{subfigure}%% 
	\begin{subfigure}[b]{0.33\linewidth}
		\centering
		\hspace{-5pt}
		\includegraphics[width=0.6\linewidth]{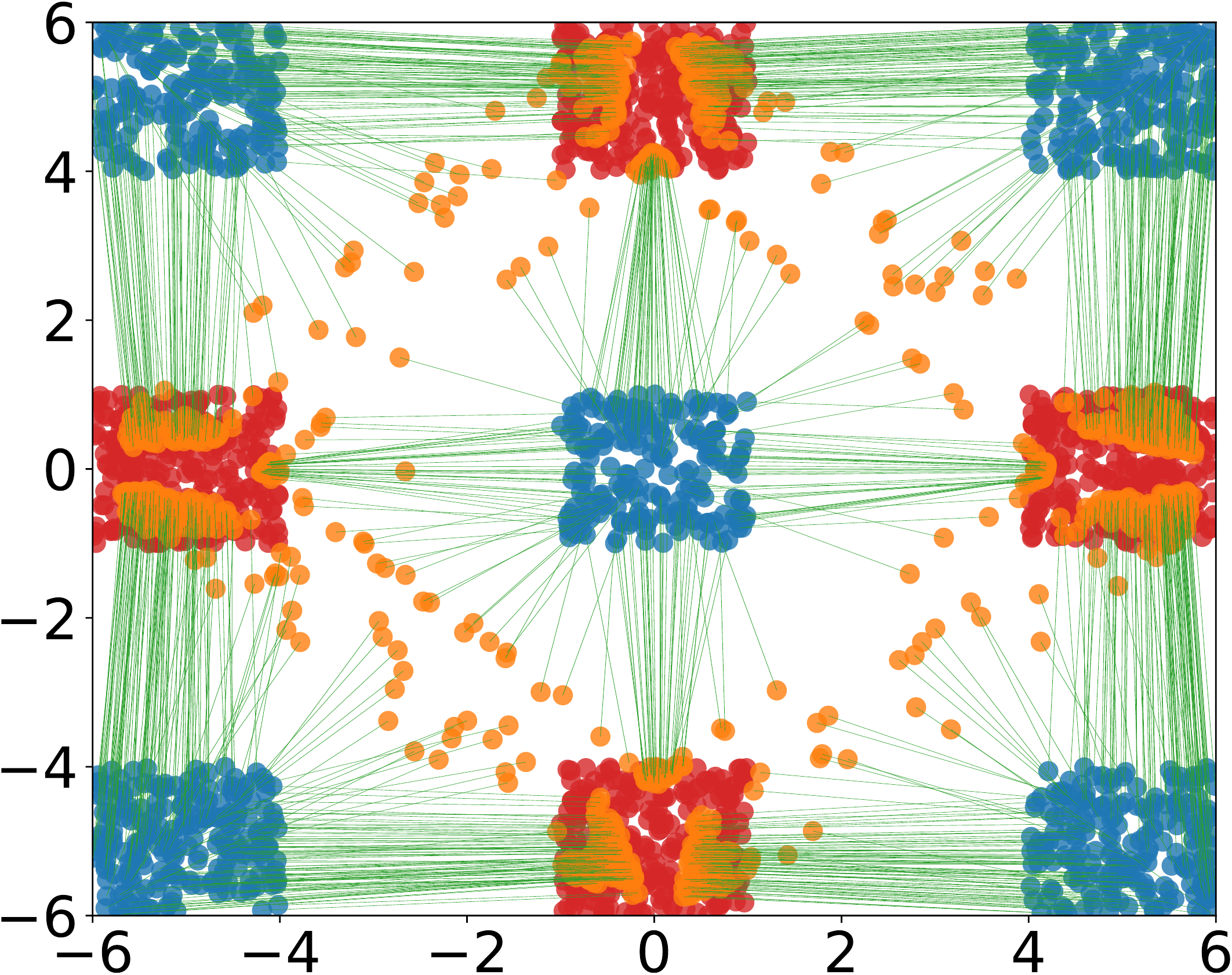} 
		\caption{BOT} 
	\end{subfigure}%% 
	\begin{subfigure}[b]{0.33\linewidth}
		\centering
		\hspace{-5pt}
		\includegraphics[width=0.6\linewidth]{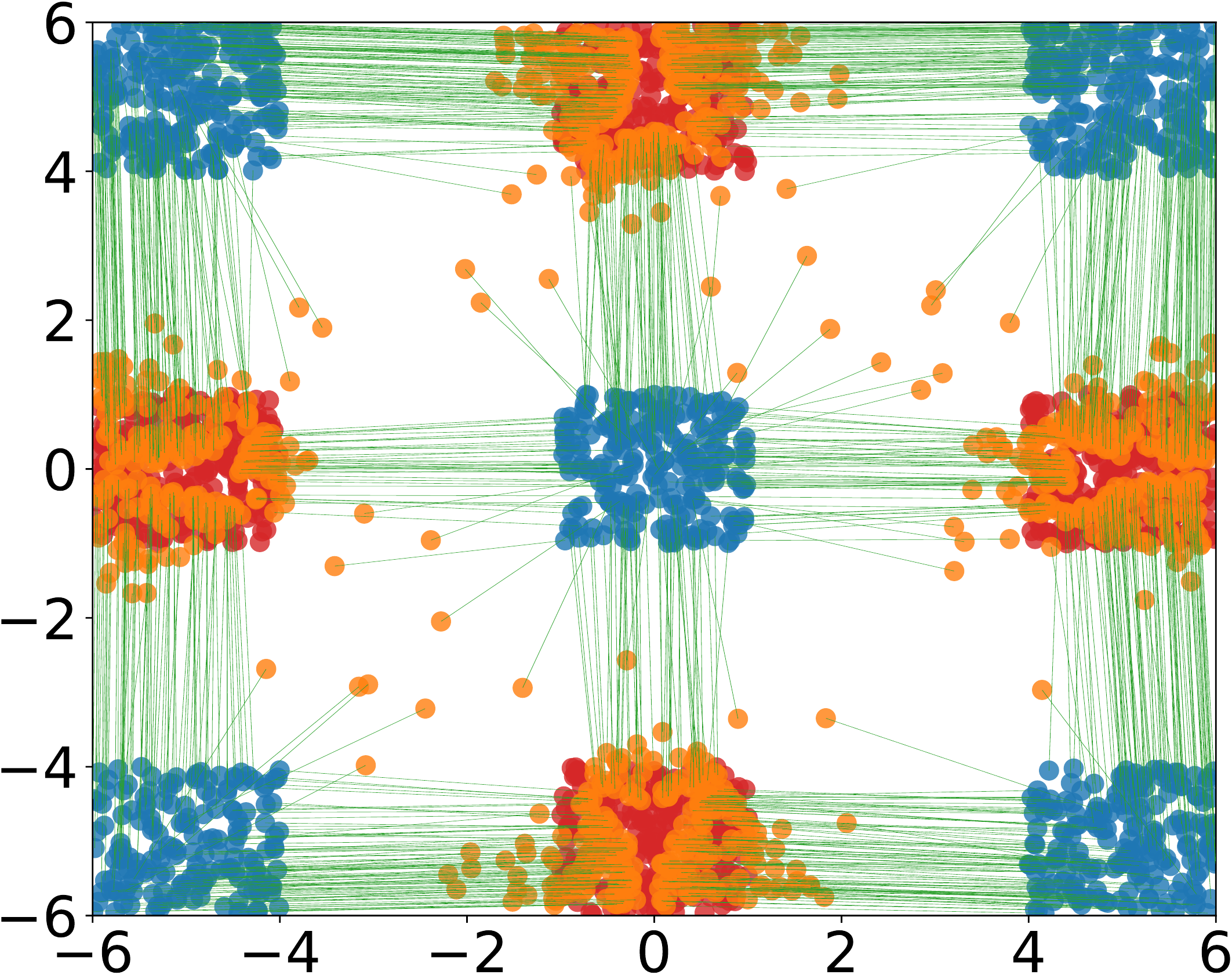} 
		\caption{K-solver} 
	\end{subfigure}%% 
	\vspace{-3pt} 
	\caption{Mappings learned by BOT and our Kantorovich solver on three 2D examples. Blue: source samples. Red: target samples. Orange: mapped samples. Green: the mapping. BOT exhibits collapse and out of distribution samples. Kantorovich solver achieves better performance in general.}
	\label{comparison_mapping}
	\vspace{-15pt}
\end{figure}

To study the effectiveness of stochastic neural network and adversarial training, we consider the OT problem from a discrete distribution to a continuous distribution. As shown in Fig. \ref{discrete2continuous_sample}, the discrete distribution is a uniform distribution supported on 4 discrete points: (-3, -3), (-3, 3), (3, -3) and (3, 3), while the continuous distribution is the standard Gaussian distribution $\mathcal{N}((0, 0)^T, I)$. We adopt the squared Euclidean distance $c(x, y) = \| x - y \| ^ 2$ as the cost function. 

\noindent 
Fig. \ref{discrete2continuous_mapping} shows the result of the Kantorovich solver. The mapped samples from the same source are marked with the same color. As we can see, with the stochastic neural network, each source sample is stochastically mapped to multiple samples. Besides, the target distribution is well recovered with adversarial training. Note that Monge and bijection mapping are not reasonable requests in this setting. 

To demonstrate the superior accuracy of our framework, we further compare it against BOT (Barycentric-OT) \cite{large-scale} on three 2D examples, including: (i) 4-Gaussian: both source and target are mixtures of 4 Gaussians, and the mixture centers of source are closer to each other than those of target; (ii) 8-Gaussian: source is the standard Gaussian and target is mixture of 8 Gaussians; (iii) Checkerboard: source and target are mixtures of uniform distributions over 2D squares of 5 and 4 chucks respectively and the mixture centers of source and target form an alternating checkerboard pattern. In this experiments, the cost function is also the squared Euclidean distance $c(x, y) = \| x - y \| ^ 2$. 

Fig. \ref{comparison_mapping} shows the results. The learned maps of BOT are noticeably collapsed in the case of 4-Gaussians and Checkerboard, and there are a large number of out of distribution mapped samples in the case of 8-Gaussian. This is because BOT learns the map by approximating the barycentric mapping. In contrast, the proposed solvers achieve better performance in general. We show the results of the Monge solver and the Bijection solver in the Appendix.

\subsection{Unsupervised Domain Adaptation}
\label{sec_exp_da}

\begin{table}[!t]
	\centering
	\vspace{-10pt}
    \caption{Results on domain adaptation among digit datasets.} 
	\vspace{5pt}
	\begin{tabular}{c|c|c|c|c}
		\hline 
		\multirow{2}{*}{Method} & ~~~~MNIST~~~~ & ~~~~USPS~~~~ & ~~~~SVHN~~~~ & ~~~~MNIST~~~~ \\ 
		                       & USPS & MNIST & MNIST & MNISTM \\
		\hline		
		
		\rule{0pt}{2.0ex}
		Source only & 81.5\% & 47.9\% & 80.8\% & 61.6\% \\
		\hline
		
 		\rule{0pt}{2.0ex}
 		CoGAN \cite{cogan} & 91.2\% & 89.1\% & - & - \\
 		\hline

 		\rule{0pt}{2.0ex}
 		ADDA \cite{adda} & 89.4\% & 90.1\% & 76.0\% & - \\
 		\hline

 		\rule{0pt}{2.0ex}
 		UNIT \cite{unit} & 96.0\% & 93.6\% & 90.5\% & - \\
 		\hline

 		\rule{0pt}{2.0ex}
 		CyCADA \cite{cycada} & 95.6\% & 96.5\% & 90.4\% & - \\
 		\hline

		\rule{0pt}{2.0ex}
		BOT \cite{large-scale} & 72.6\% & 60.5\% & 62.9\% & - \\
		\hline
		
		\rule{0pt}{2.0ex}
		StochJDOT \cite{stochjdot} & 93.6\% & 90.5\% & 67.6\% & 66.7\% \\
		\hline
		
		\rule{0pt}{2.0ex}
		DeepJDOT \cite{deepjdot} & 95.7\% & 96.4\% & \bf{96.7\%} & 92.4\% \\
		\hline
		
		\rule{0pt}{2.0ex}
		SPOT \cite{spot} & 97.5\% & 96.5\% & 96.2\% & 94.9\% \\
		\hline
		
		\rule{0pt}{2.0ex}
		K-solver & \bf{99.0\%} & \bf{97.1\%} & {95.7}\% & \bf{98.2\%} \\
		\hline

		\rule{0pt}{2.0ex}
		M-solver & {99.0\%} & {96.7\%} & {95.8}\% & {97.4\%} \\
		\hline

		\rule{0pt}{2.0ex}
		B-solver & {98.9\%} & {96.6\%} & \bf{96.7}\% & {97.5\%} \\
		\hline

		\rule{0pt}{2.0ex}
		Target only & 98.2\% & 99.0\% & 99.0\% & 96.1\% \\
		\hline
	\end{tabular}
	\label{domain_adaptation_table}
	\vspace{-10pt}
\end{table}

In domain adaptation, labeled data for a task are available in the source domain and there are only unlabeled data in the target domain. The objective of domain adaptation is to address the lack of labeled data problem and learn a well-performing model in the target domain based on these data. 

In this section, we explore OT for domain adaptation. To adapt the class labels from the source domain to the target domain, we learn an optimal mapping between the samples from the source domain and the samples from the target domain. Follow the common choice \cite{cycada,spot}, we define the cost function to be the cross-entropy $\mathcal{H}$ between the label of the source sample and the label prediction of the translated target sample: 
\begin{equation}
c(x, y) = \mathcal{H}(C_x(y), l(x)), 
\end{equation}
where $C_{x}$ is a pre-trained classifier on the source, $l(x)$ denotes the class label of $x$. For more training details, please refer to the Appendix. 

We perform domain adaptation between four digit image datasets: MNIST \cite{mnist}, USPS \cite{usps}, SVHN \cite{svhn}, and MNISTM \cite{mnistm}, and consider the following four adaptation directions: MNIST-to-USPS, USPS-to-MNIST, MNIST-to-MNISTM, and SVHN-to-MNIST. We compare our methods with various baselines, including BOT \cite{large-scale}, StochJDOT \cite{stochjdot}, DeepJDOT \cite{deepjdot}, SPOT \cite{spot}, CoGAN \cite{cogan}, ADDA \cite{adda}, UNIT \cite{unit} and CyCADA \cite{cycada}. 

The results are shown in Table \ref{domain_adaptation_table}. Here we also include the ``Source only'' and ``Target only'', which is the resulting accuracy of classifiers that trained with labeled source data and labeled target data respectively. They can be used as the empirical lower bound and upper bound. As we can see, domain adaptation based on our methods achieve large performance improvements over ``Source only'' on all tasks and approach the ``Target only'' results. And compared with other baseline methods, our ones generally achieve superior performances. 

For the task of domain adaptation, deterministic or bijection is not essential requirement, as a source is mapped to multiple targets with same label is acceptable. So, we think it is is understandable that our three solver share similar performance in this task. But we do can tell that the Monge solver and optimal bijection solver is slightly worse than the Kantorovich solver. We will see the similar in unsupervised image-to-image translation. We understand it as the cycle-consistency constraint when unnecessarily introduced will drive the objective a little bit towards unnecessary property thus degenerate its performance. 

\subsection{Unsupervised Image-to-Image Translation}
\label{sec_exp_image2image}

Image-to-image translation aims to establish a desired mapping between two image distributions so as to translate images in the source domain to images in the target domain. In a supervised case, such a desired mapping is defined by a large number of paired examples. 

OT can be used for unsupervised image-to-image translation via attaining an optimal mapping from the source image distribution to the target image distribution. For different tasks, different cost function $c(x,y)$ can be accordingly designed to reflect the desired property.
Compared with other unsupervised approaches, like CycleGAN \cite{cyclegan}, OT-based methods can better control the map towards being the desired. 

We perform image-to-image translation on the following two tasks: Edges-to-Handbags \cite{dataset_handbag} and Handbags-to-Shoes \cite{dataset_shoes}. For the first task, we expect the translated sample to be of similar sketch with the input sample and therefore design the cost function as the $L_2$ norm between feature maps extracted through different convolution kernels for edge detection. For the second, we expect the color of translated sample to be similar to the input sample and therefore define the cost function as the mean squared distance between the average color vectors. 

We compare our solvers against CycleGAN \cite{cyclegan}, BOT \cite{large-scale}, and OT-CycleGAN \cite{ot-cyclegan} with different reference coefficients $\lambda_{ref}$. Note that SPOT \cite{spot} does not learn a transport map and is thus not applicable for this task. We use the Kernel Inception Distance (KID) \cite{kid_score} and mismatching degree \cite{ot-cyclegan} to quantitatively evaluate different methods. KID computes the squared maximum mean discrepancy (MMD) between target distribution and distribution of the mapped images in the feature space, where the feature is extracted from the Inception network architecture \cite{inceptionnet}. Mismatching degree measures the average difference between source and corresponding mapped images. Both metrics are the lower the better. 

\begin{figure*}[!t]
	\begin{subfigure}[t]{0.5\textwidth}
		\centering
		\includegraphics[width=0.98\textwidth]{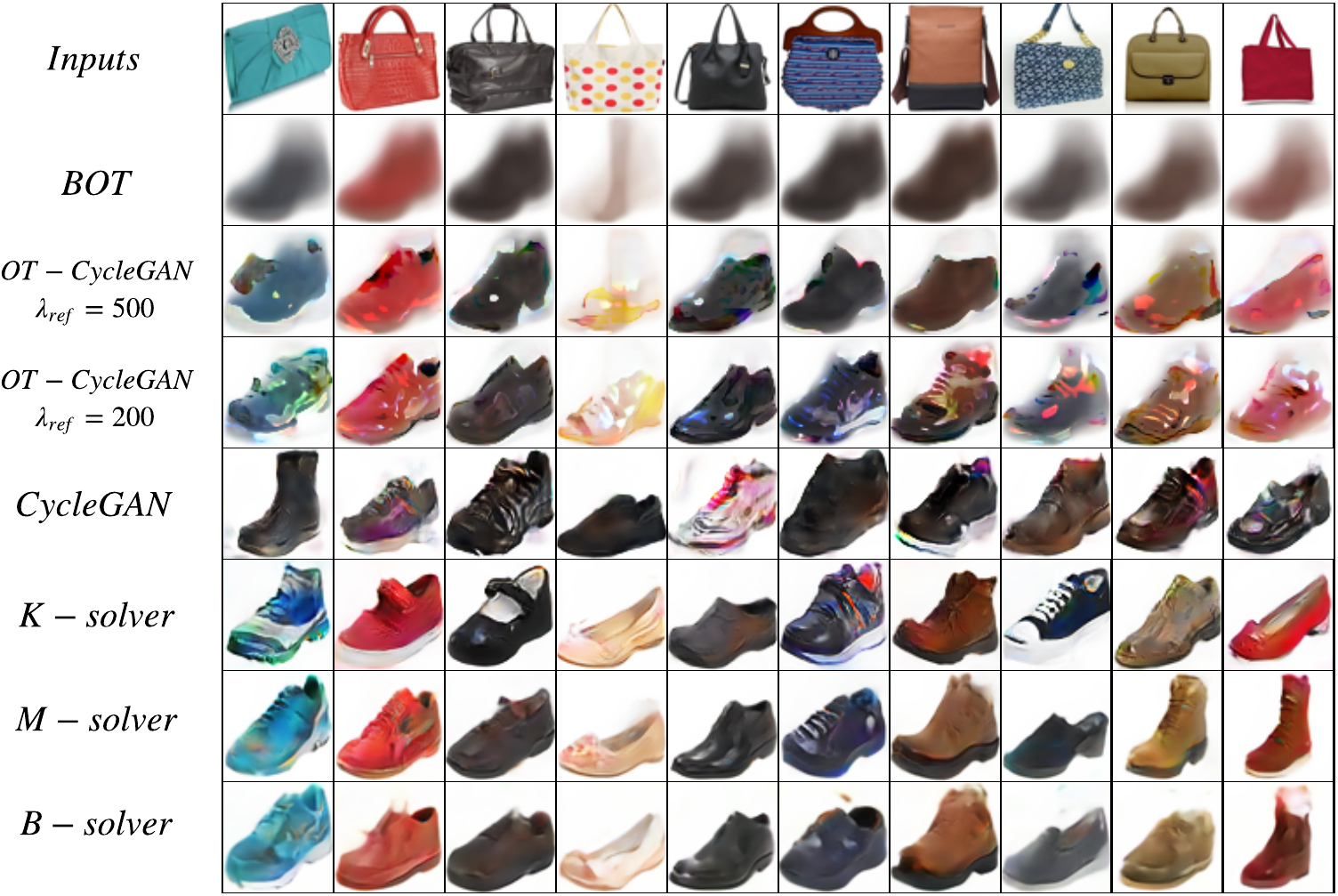}
		\caption{Handbags-to-Shoes}
	\end{subfigure}
	\begin{subfigure}[t]{0.5\textwidth}
		\centering
		\includegraphics[width=0.98\textwidth]{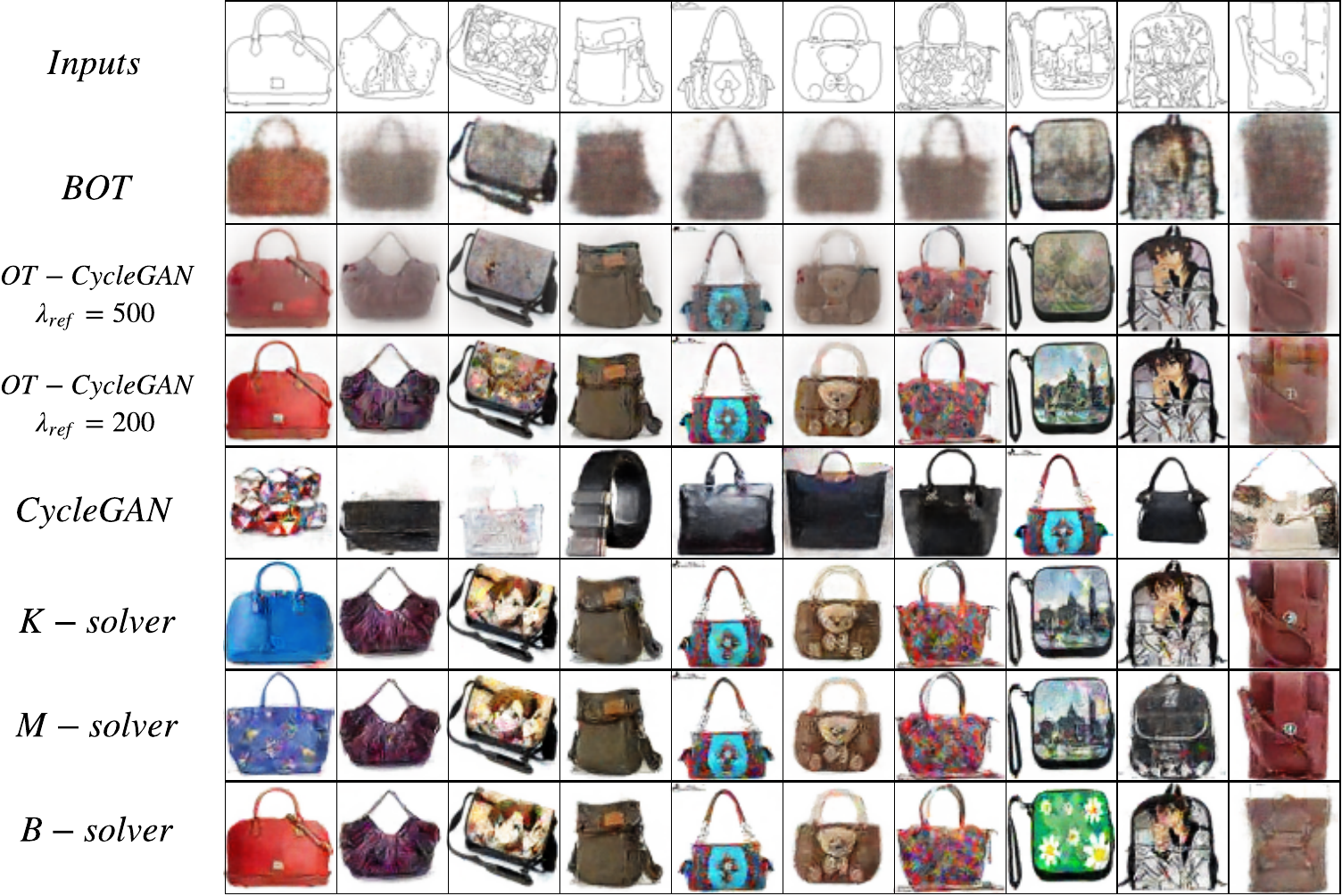}
		\caption{Edges-to-Handbags}
	\end{subfigure}
	\vspace{-5pt}
	\caption{Visual results on unsupervised image-to-image translation.}
	\label{image2image_translation}
	\vspace{-5pt}
\end{figure*}

\begin{table}[!t]
	\centering
	\vspace{-10pt}
	\caption{Quantitative comparison on image-to-image translation.}
	\vspace{3pt}
	\begin{tabular}{c|c|c|c|c}
	    \multicolumn{5}{c}{Handbags2shoes: h $\rightarrow$ s.  Edges2handbags: e $\rightarrow$ h.} \\
		\hline
		\multirow{2}{*}{Method} & \multicolumn{2}{c|}{~~~~~~~~~~~~~~KID~~~~~~~~~~~~~~} & \multicolumn{2}{c}{~mismatching degree~~} \\ \cline{2-5} 
		& ~~~h $\rightarrow$ s~~~ & e $\rightarrow$ h & ~~~h $\rightarrow$ s~~~    & e $\rightarrow$ h                 \\ \hline % \hline
		BOT \cite{large-scale}                    & 27.31$\pm$0.07  & 16.16$\pm$0.24 & ~~~8.7                & 245.00             \\ \hline
		OT-CycleGAN \cite{ot-cyclegan} ($\lambda_{ref}=500$)            & 12.25$\pm$0.12  & ~~1.95$\pm$0.10  & ~~13.2               & 290.03             \\ \hline
		OT-CycleGAN \cite{ot-cyclegan} ($\lambda_{ref}=200$)            & ~~6.86$\pm$0.10   & ~~1.85$\pm$0.08  & ~~16.3               & 360.82             \\ \hline
		CycleGAN \cite{cyclegan}               & ~~5.14$\pm$0.09   & ~~1.89$\pm$0.10  & 135.5              & 478.99             \\ \hline
		K-solver               & ~~2.27$\pm$0.04   & ~~1.54$\pm$0.09  & ~~8.9               & 329.87             \\ \hline
		M-solver               & ~~3.28$\pm$0.05   & ~~1.91$\pm$0.09  & ~~11.2                & 330.35             \\ \hline
		B-solver               & ~~4.73$\pm$0.06   & ~~2.04$\pm$0.11  & ~~12.0                & 329.86             \\ \hline		
	\end{tabular}
	\label{image_translation_table}
	\vspace{-15pt}
\end{table}

Table \ref{image_translation_table} shows the results in terms of KID and mismatching degree. As we can see, CycleGAN achieves low KIDs but has high mismatching degrees, which is reasonable because it has no explicit control on the property of the learned mapping. In contrast, BOT achieves low mismatching degrees but has high KIDs, since it uses the barycentric projection of an optimal transport plan, which changes the distribution of translated images and thus not match with the target distribution. OT-CycleGAN generates relatively better results, but it requires a good balance: large reference coefficient results in low mismatching degree but high KID, and vice versa. Fig. \ref{image2image_translation} shows the visual results of different methods. We can see the results of BOT and OT-CycleGAN with large reference weight are noticeably blurry. And the results of OT-CycleGAN even with a relatively small reference weight is also not clear enough. Compared with the baseline methods, the proposed Kantorovich solver can well control the mapping and at the same time generates realistic samples. 

Besides, we can see from Table \ref{image_translation_table}, the optimal bijection solver (B-solver), as an end-to-end version of OT-CycleGAN, has comparative and somewhat better results than OT-CycleGAN. Furthermore, if comparing our three solvers, we can see that the KID and mismatch degree generally decreases as the number of cycle-consistency decreases. Given that OT is already sufficient to establish a well-defined unsupervised mapping between two distributions, and deterministic or not is not critical for this task, we think the results are reasonable and echo with the results in domain adaption. Note that CycleGAN which has cycle-consistency also holds a relatively high KID. 

\subsection{Color Transfer}

\begin{figure}[!t]
    \begin{subfigure}[b]{0.5\linewidth}
    	\begin{subfigure}[b]{0.5\linewidth}
    		\centering
    		\includegraphics[width=0.9\linewidth]{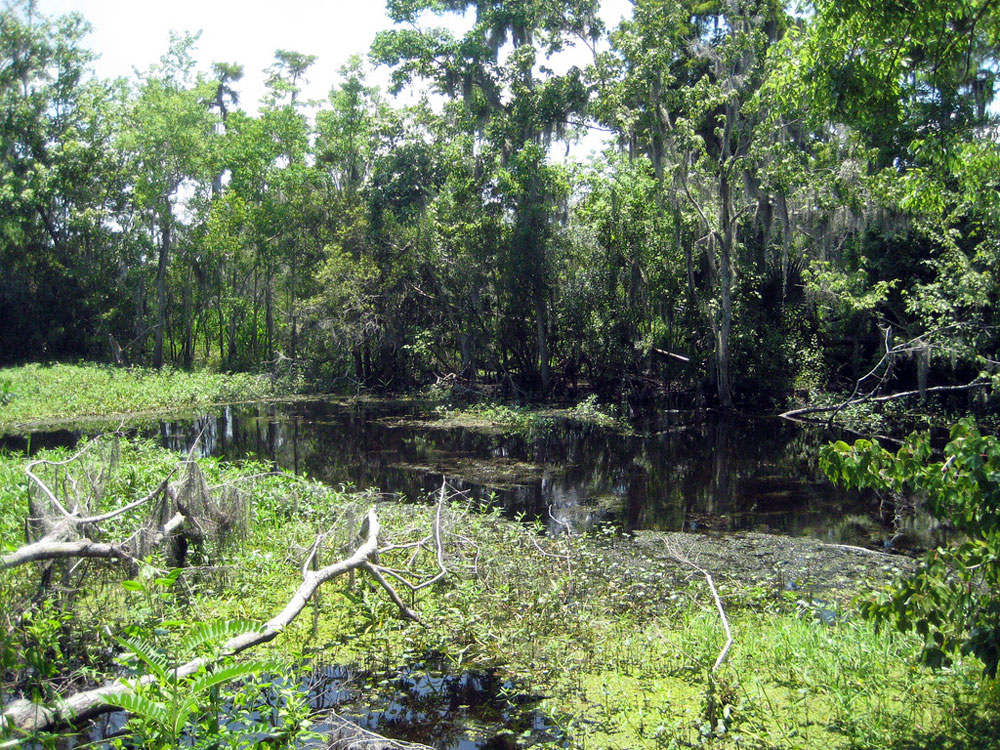} 
    	\end{subfigure}%% 
    	\begin{subfigure}[b]{0.5\linewidth}
    		\centering
    		\includegraphics[width=0.9\linewidth]{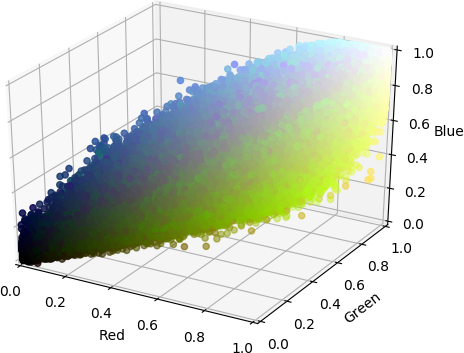} 
    	\end{subfigure}%% 
    	\caption{Source image} 
    \end{subfigure}
    \begin{subfigure}[b]{0.5\linewidth}
    	\begin{subfigure}[b]{0.5\linewidth}
    		\centering
    		\includegraphics[width=0.9\linewidth]{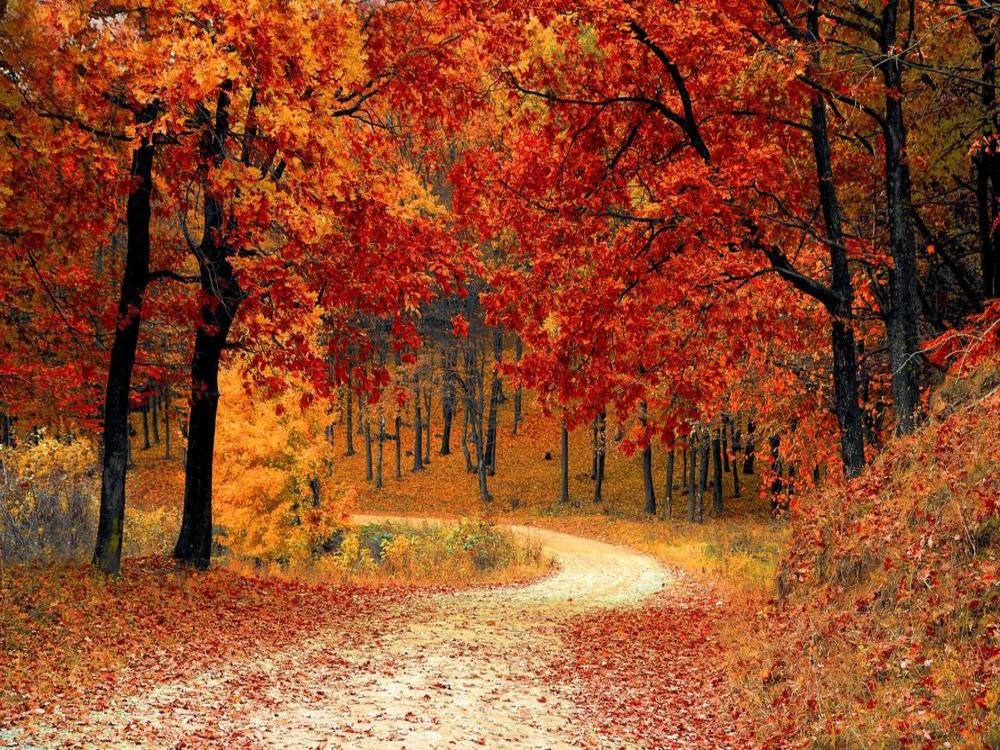} 
    	\end{subfigure}%% 
    	\begin{subfigure}[b]{0.5\linewidth}
    		\centering
    		\includegraphics[width=0.9\linewidth]{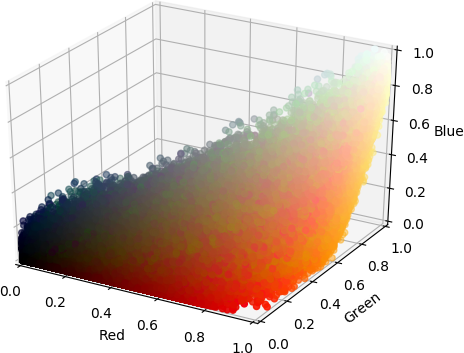} 
    	\end{subfigure}%% 
    	\caption{Target image} 
    \end{subfigure}
    % \vspace{-5pt}
	\caption{Source image, target image and corresponding 3D color distributions.}
	\label{color_transfer_sample}
\end{figure}

\begin{figure}[!t]
	\begin{subfigure}[b]{0.2\linewidth}
		\centering
		\includegraphics[width=0.99\linewidth]{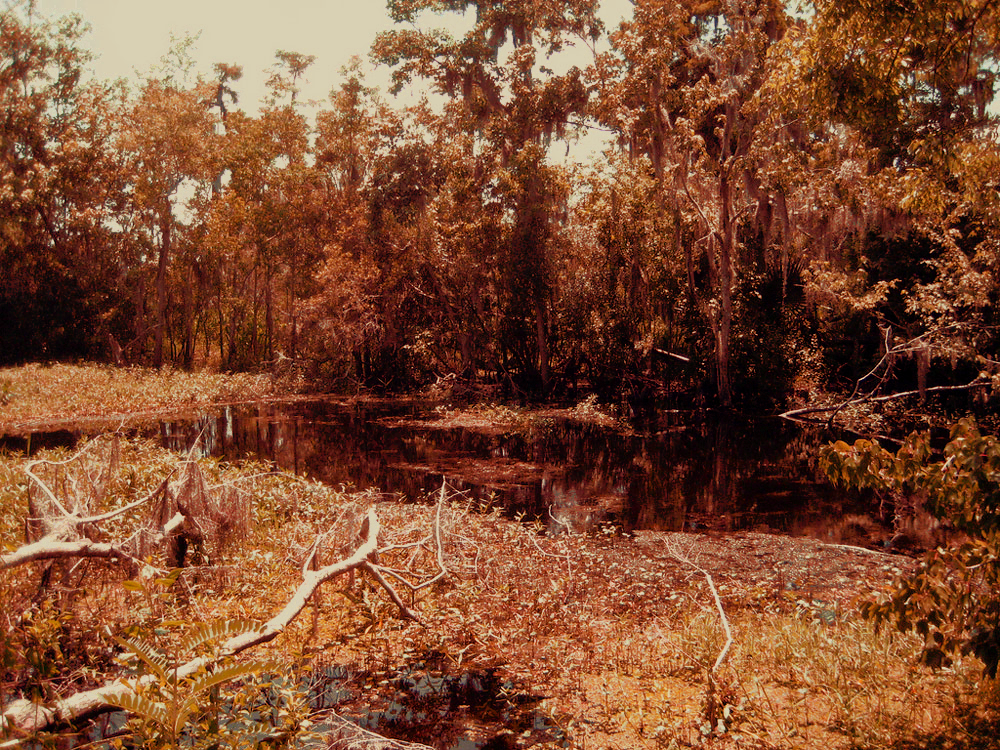} 
	\end{subfigure}%% 
	\begin{subfigure}[b]{0.2\linewidth}
		\centering
		\includegraphics[width=0.99\linewidth]{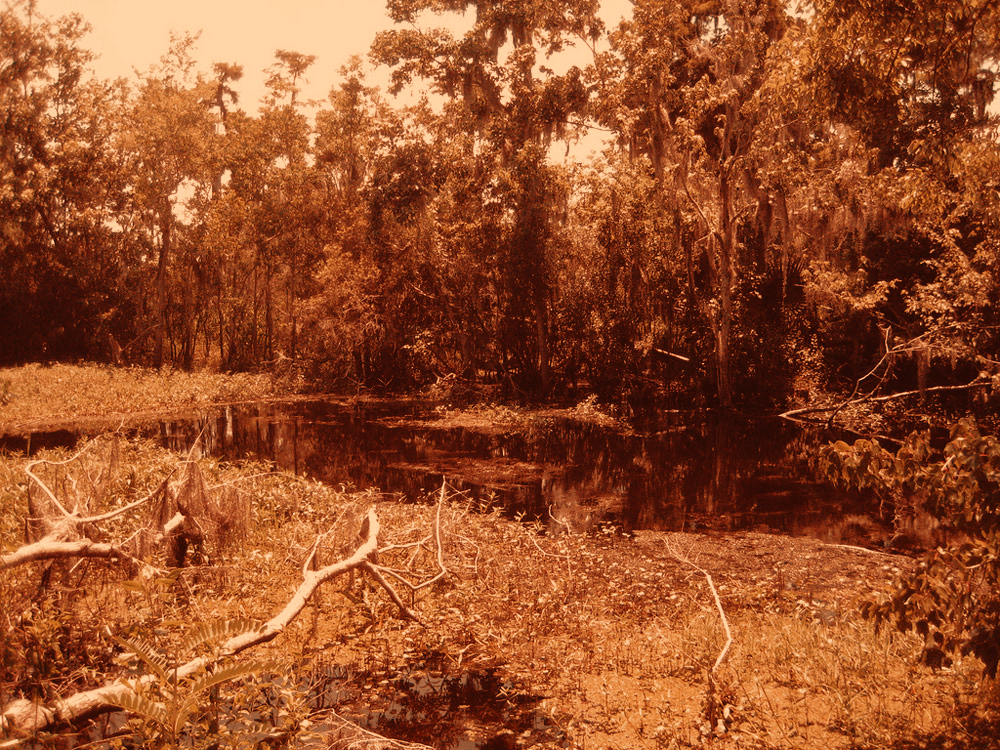} 
	\end{subfigure}%% 
	\begin{subfigure}[b]{0.2\linewidth}
		\centering
		\includegraphics[width=0.99\linewidth]{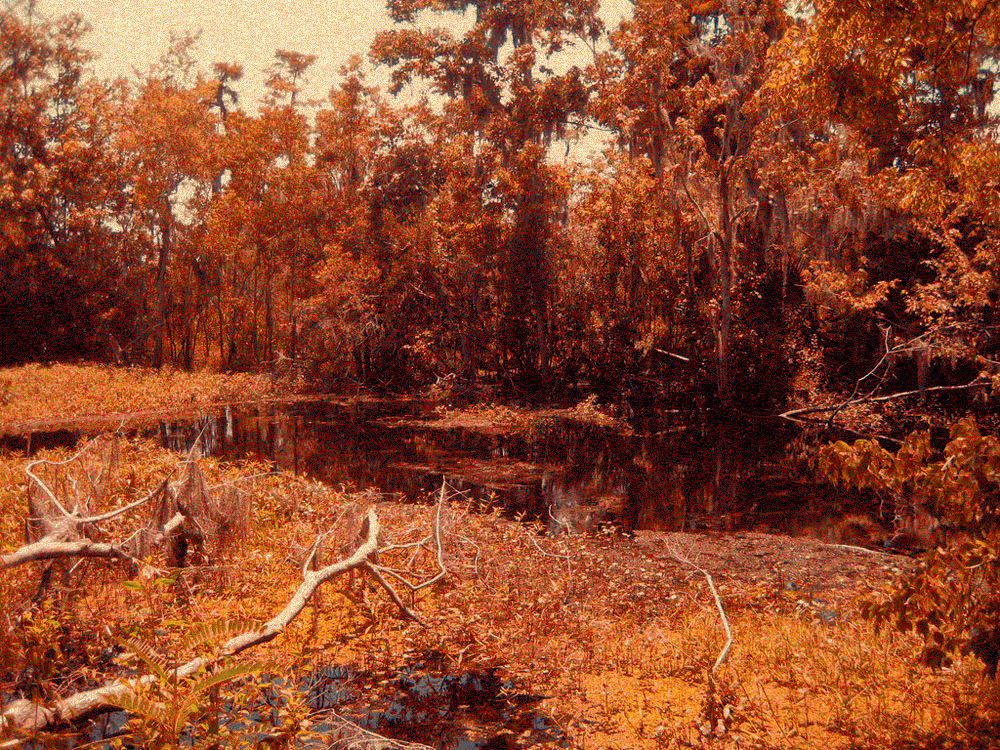} 
	\end{subfigure}%% 
	\begin{subfigure}[b]{0.2\linewidth}
		\centering
		\includegraphics[width=0.99\linewidth]{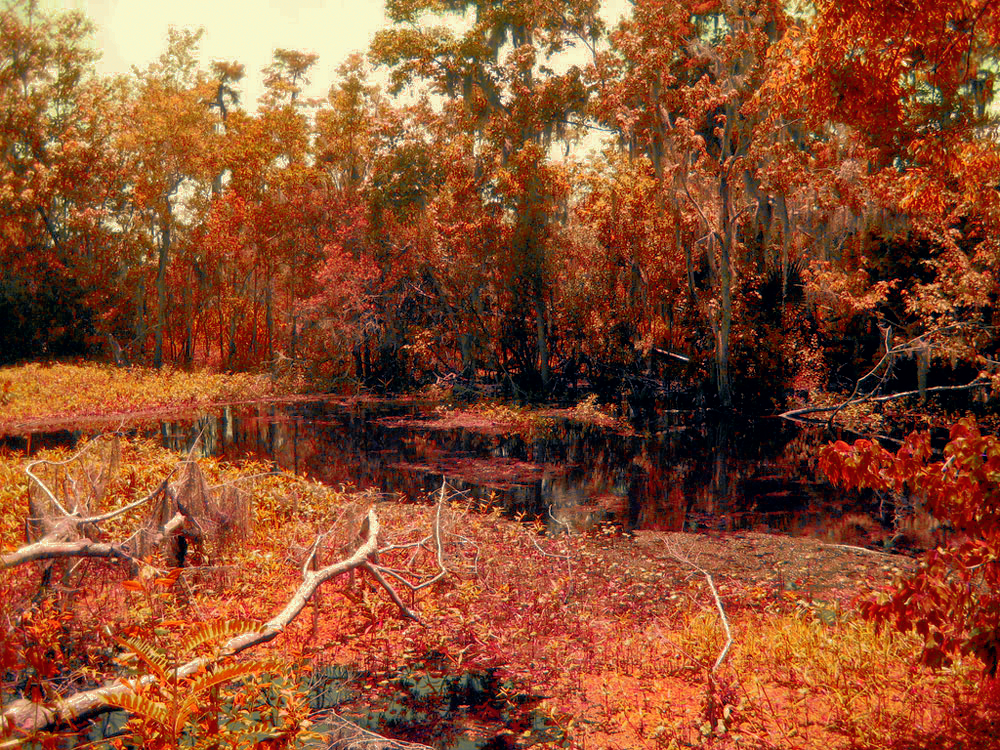}
	\end{subfigure}%% 
	\begin{subfigure}[b]{0.2\linewidth}
		\centering
		\includegraphics[width=0.99\linewidth]{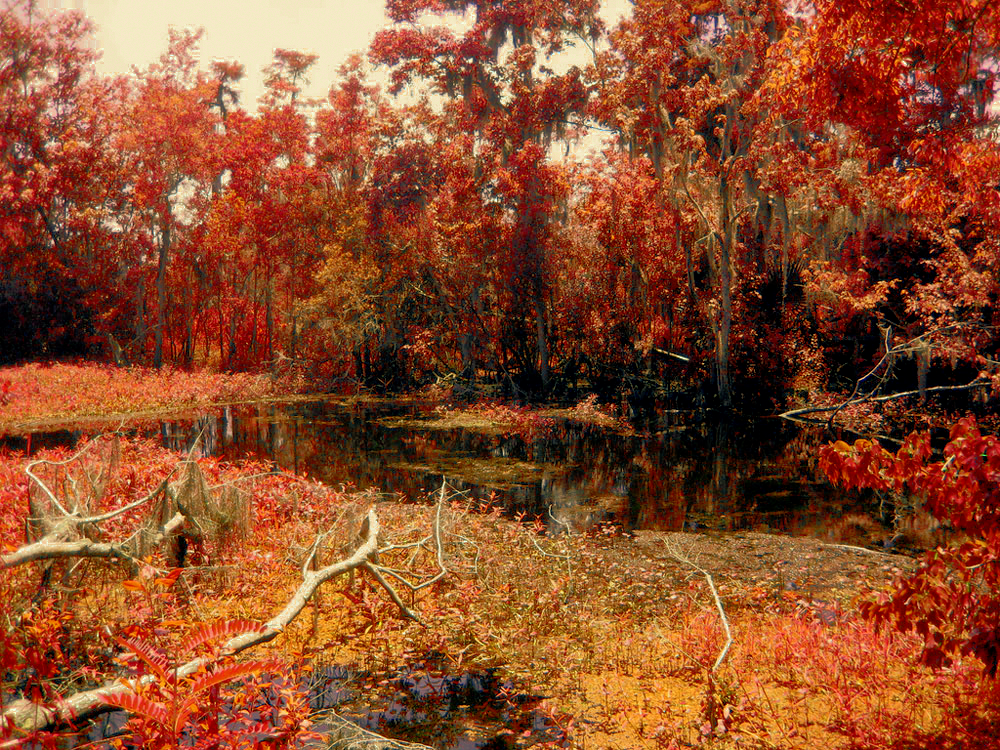}
	\end{subfigure}%% 
	
	\begin{subfigure}[b]{0.2\linewidth}
		\centering
		\includegraphics[width=0.99\linewidth]{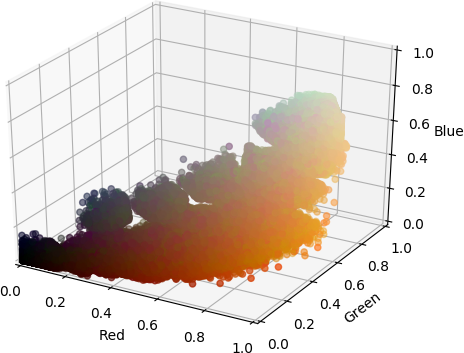} 
		\caption{ROT} 
	\end{subfigure}%% 
	\begin{subfigure}[b]{0.2\linewidth}
		\centering
		\includegraphics[width=0.99\linewidth]{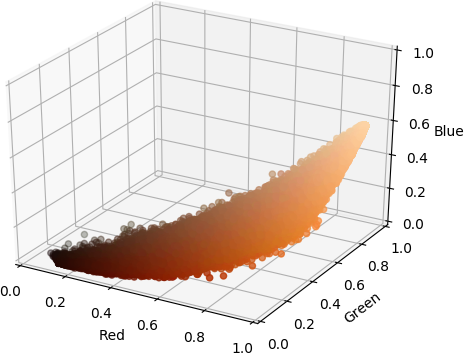} 
		\caption{BOT} 
	\end{subfigure}%% 
	\begin{subfigure}[b]{0.2\linewidth}
		\centering
		\includegraphics[width=0.99\linewidth]{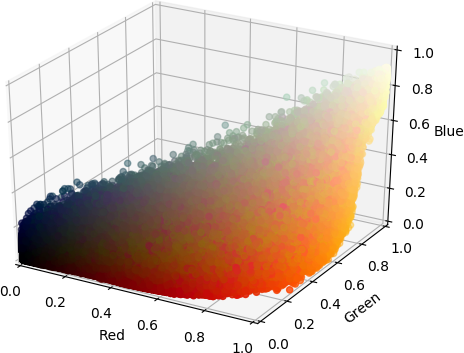} 
		\caption{Kantorovich} 
	\end{subfigure}%% 
	\begin{subfigure}[b]{0.2\linewidth}
		\centering
		\includegraphics[width=0.99\linewidth]{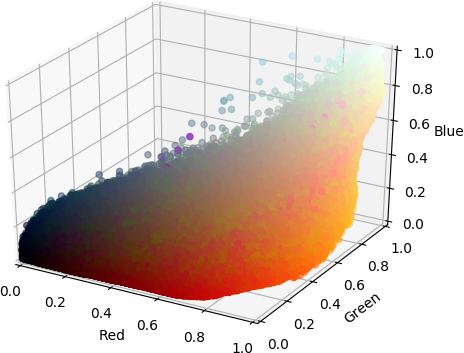} 
		\caption{Monge} 
	\end{subfigure}%% 
	\begin{subfigure}[b]{0.2\linewidth}
		\centering
		\includegraphics[width=0.99\linewidth]{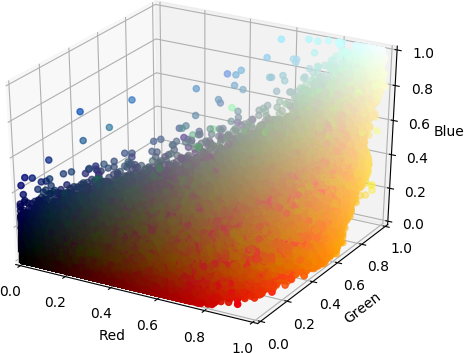} 
		\caption{Bijection} 
	\end{subfigure}%% 
	\vspace{-5pt}
	\caption{Transferred image and the corresponding 3D color distributions.}
	\label{comparison_color_transfer}
	\vspace{-15pt}
\end{figure}

Given source image $X$ and target image $Y$, color transfer aims to transfer the color style from image $Y$ to image $X$. OT-based methods seek to establish an optimal mapping from the color histogram of $X$ to the one of $Y$, and change the color of $X$ according to the mapping of the color histogram. 

In this experiment, we compare BOT \cite{large-scale} and the proposed Kantorovich, Monge and Bijection solver. We also include ROT \cite{ferradans2014regularized}, which is also an OT-based color transfer solution but is not large-scale. We use the two images shown in Fig. \ref{color_transfer_sample} as the source and target images as a demonstration for this task. 

Fig. \ref{comparison_color_transfer} shows the results by different methods. As we can see, the transferred histogram of ROT is collapsed and the visual result looks blurry. It is mainly because ROT is not large-scale so that sub-sampling and interpolation on color histograms is required. The transferred histogram of BOT is also collapsed, this is again because it bases on barycentric mapping which by natural has such a tendency. Results of our solvers are relatively more clear and the transferred color histogram is very close to the target's. 

Deterministic mapping in this task becomes kind of important, as the mapping of a same source might be used multiple times in a same image. With a close look at the result of Kantorovich solver, we can find that the sky appears noisy, which is a result of stochastic mapping of the Kantorovich solver. In contrast, the results of Monge solver and optimal bijection solver does not have such problem. Besides, we can also see that the results of Monge solver and optimal bijection solver look sharper than the Kantorovich solver. 

We show the learned color mappings of our solvers in the Appendix, from which we can tell that the mapping learned by the Kantorovich solver is stochastic, and on the contrary, the mapping learned by the Monge solver is deterministic. This result empirically verifies Proposition \ref{one-side cycle-consistency}. 

\section{Related Work}

Recently, lots of works has been devoted to accelerating the computation of OT to broaden its application scenario. Among them, \cite{sinkhorn} introduced entropy regularization into OT and solved the dual by Sinkhorn and Knopp's algorithm. However, it still has a complexity of $O(n^2)$ and cannot be used in continuous settings. \cite{largescale_semidual} proposed to optimize a semi-dual objective function with stochastic gradient algorithms and parameterized the dual variables as kernel expansions. It has a time complexity of $O(n)$ and hence scales moderately. 

Among the algorithms for large-scale OT, SPOT \cite{spot} learns a mapping from a latent variable $z$ to the OT plan by an implicit generative learning-based framework. However, the optimal mapping between the source distribution and the target distribution is not explicitly solved and hard to retrieve. BOT \cite{large-scale} gets the density of optimal transport plan via a stochastic dual approach and then fits a deep neural network to the barycentric projection of the solved transport plan as approximated Monge map. However, the pushforward distribution of such an estimated map may not well align with the target distribution, which results in out of distribution samples and poor performance in related applications. 

\cite{scalable_unbalanced_ot} formulated unbalanced OT as a problem of simultaneously learning of a transport map and a scaling factor. But, the Monge problem and optimal bijection transport are not considered. \cite{ot_inputconvex_nn,w2gan_origin} focused on the OT problem with Wasserstein-2 metric, which leads to limited application scenarios. On the contrary, our method allows to use any differentiable cost function.

\section{Conclusion}

In this paper, we proposed an end-to-end framework for large-scale optimal transport, which can be further extended towards learning a Monge map or an optimal bijection between two distributions. 

We built a soft links between the Kantorovich formulation and the Monge formulation with one-side cycle-consistency constraint. We extended the concept of OT and introduced the problem of optimal bijection transport, which can be efficiently solved with our framework's two-side cycle-consistency extension. 

In experiments, we found that though cycle-consistency and OT can both be used to achieve unsupervised pairing, OT seems to be sufficient and more effective than cycle-consistency, when a task-specific cost can be easy defined. But in some tasks, where deterministic mapping or bijection is preferred, cycle-consistency may benefit.

\bibliographystyle{splncs04}
\bibliography{egbib}

\begin{thebibliography}{10}
\providecommand{\url}[1]{\texttt{#1}}
\providecommand{\urlprefix}{URL }
\providecommand{\doi}[1]{https://doi.org/#1}

\bibitem{wgan}
Arjovsky, M., Chintala, S., Bottou, L.: Wasserstein gan. arXiv preprint
  arXiv:1701.07875  (2017)

\bibitem{kid_score}
Bi{\'n}kowski, M., Sutherland, D.J., Arbel, M., Gretton, A.: Demystifying mmd
  gans. arXiv preprint arXiv:1801.01401  (2018)

\bibitem{brenier1991polar}
Brenier, Y.: Polar factorization and monotone rearrangement of vector-valued
  functions. Communications on pure and applied mathematics  \textbf{44}(4),
  375--417 (1991)

\bibitem{stochjdot}
Courty, N., Flamary, R., Habrard, A., Rakotomamonjy, A.: Joint distribution
  optimal transportation for domain adaptation. In: NIPS. pp. 3730--3739 (2017)

\bibitem{courty2017optimal}
Courty, N., Flamary, R., Tuia, D., Rakotomamonjy, A.: Optimal transport for
  domain adaptation. PAMI  \textbf{39}(9),  1853--1865 (2017)

\bibitem{sinkhorn}
Cuturi, M.: Sinkhorn distances: Lightspeed computation of optimal transport.
  In: NIPS. pp. 2292--2300 (2013)

\bibitem{network_simplex}
Damian, K., Comm, B., Garret, M.: The minimum Cost Flow Problem and The Network
  Simplex Method. Ph.D. thesis, Dissertation de Mast{\`e}re, Universit{\'e}
  College Gublin, Irlande (1991)

\bibitem{deepjdot}
Damodaran, B.B., Kellenberger, B., Flamary, R., Tuia, D., Courty, N.: Deepjdot:
  Deep joint distribution optimal transport for unsupervised domain adaptation.
  In: ECCV. pp. 467--483. Springer (2018)

\bibitem{usps}
Denker, J.S., Gardner, W., Graf, H.P., Henderson, D., Howard, R.E., Hubbard,
  W., Jackel, L.D., Baird, H.S., Guyon, I.: Neural network recognizer for
  hand-written zip code digits. In: NIPS. pp. 323--331 (1989)

\bibitem{ferradans2014regularized}
Ferradans, S., Papadakis, N., Peyr{\'e}, G., Aujol, J.F.: Regularized discrete
  optimal transport. SIAM Journal on Imaging Sciences  \textbf{7}(3),
  1853--1882 (2014)

\bibitem{mnistm}
Ganin, Y., Ustinova, E., Ajakan, H., Germain, P., Larochelle, H., Laviolette,
  F., Marchand, M., Lempitsky, V.: Domain-adversarial training of neural
  networks. The Journal of Machine Learning Research  \textbf{17}(1),
  2096--2030 (2016)

\bibitem{largescale_semidual}
Genevay, A., Cuturi, M., Peyr{\'e}, G., Bach, F.: Stochastic optimization for
  large-scale optimal transport. In: NIPS. pp. 3440--3448 (2016)

\bibitem{gan}
Goodfellow, I., Pouget-Abadie, J., Mirza, M., Xu, B., Warde-Farley, D., Ozair,
  S., Courville, A., Bengio, Y.: Generative adversarial nets. In: NIPS (2014)

\bibitem{wgangp}
Gulrajani, I., Ahmed, F., Arjovsky, M., Dumoulin, V., Courville, A.: Improved
  training of wasserstein gans. arXiv preprint arXiv:1704.00028  (2017)

\bibitem{duallearning}
He, D., Xia, Y., Qin, T., Wang, L., Yu, N., Liu, T., Ma, W.Y.: Dual learning
  for machine translation. In: NIPS. pp. 820--828 (2016)

\bibitem{resnet}
He, K., Zhang, X., Ren, S., Sun, J.: Deep residual learning for image
  recognition. In: CVPR. pp. 770--778 (2016)

\bibitem{autoencoder}
Hinton, G.E., Salakhutdinov, R.R.: Reducing the dimensionality of data with
  neural networks. science  \textbf{313}(5786),  504--507 (2006)

\bibitem{cycada}
Hoffman, J., Tzeng, E., Park, T., Zhu, J.Y., Isola, P., Saenko, K., Efros,
  A.A., Darrell, T.: Cycada: Cycle-consistent adversarial domain adaptation.
  arXiv preprint arXiv:1711.03213  (2017)

\bibitem{kantorovich}
Kantorovich, L.V.: On the translocation of masses. In: Dokl. Akad. Nauk. USSR
  (NS). vol.~37, pp. 199--201 (1942)

\bibitem{disco}
Kim, T., Cha, M., Kim, H., Lee, J.K., Kim, J.: Learning to discover
  cross-domain relations with generative adversarial networks. arXiv preprint
  arXiv:1703.05192  (2017)

\bibitem{mnist}
LeCun, Y., Bottou, L., Bengio, Y., Haffner, P., et~al.: Gradient-based learning
  applied to document recognition. Proceedings of the IEEE  \textbf{86}(11),
  2278--2324 (1998)

\bibitem{w2gan_origin}
Leygonie, J., She, J., Almahairi, A., Rajeswar, S., Courville, A.: Adversarial
  computation of optimal transport maps. arXiv preprint arXiv:1906.09691
  (2019)

\bibitem{unit}
Liu, M.Y., Breuel, T., Kautz, J.: Unsupervised image-to-image translation
  networks. In: NIPS. pp. 700--708 (2017)

\bibitem{cogan}
Liu, M.Y., Tuzel, O.: Coupled generative adversarial networks. In: NIPS. pp.
  469--477 (2016)

\bibitem{ot-cyclegan}
Lu, G., Zhou, Z., Song, Y., Ren, K., Yu, Y.: Guiding the one-to-one mapping in
  cyclegan via optimal transport. In: AAAI. vol.~33, pp. 4432--4439 (2019)

\bibitem{ot_inputconvex_nn}
Makkuva, A.V., Taghvaei, A., Oh, S., Lee, J.D.: Optimal transport mapping via
  input convex neural networks. arXiv preprint arXiv:1908.10962  (2019)

\bibitem{monge}
Monge, G.: M{\'e}moire sur la th{\'e}orie des d{\'e}blais et des remblais.
  Histoire de l'Acad{\'e}mie Royale des Sciences de Paris  (1781)

\bibitem{svhn}
Netzer, Y., Wang, T., Coates, A., Bissacco, A., Wu, B., Ng, A.Y.: Reading
  digits in natural images with unsupervised feature learning  (2011)

\bibitem{rubner2000earth}
Rubner, Y., Tomasi, C., Guibas, L.J.: The earth mover's distance as a metric
  for image retrieval. IJCV  \textbf{40}(2),  99--121 (2000)

\bibitem{ot_applymath}
Santambrogio, F.: Optimal transport for applied mathematicians. Birk{\"a}user,
  NY  \textbf{55},  58--63 (2015)

\bibitem{large-scale}
Seguy, V., Damodaran, B.B., Flamary, R., Courty, N., Rolet, A., Blondel, M.:
  Large-scale optimal transport and mapping estimation. arXiv preprint
  arXiv:1711.02283  (2017)

\bibitem{inceptionnet}
Szegedy, C., Vanhoucke, V., Ioffe, S., Shlens, J., Wojna, Z.: Rethinking the
  inception architecture for computer vision. In: CVPR. pp. 2818--2826 (2016)

\bibitem{adda}
Tzeng, E., Hoffman, J., Saenko, K., Darrell, T.: Adversarial discriminative
  domain adaptation. In: CVPR. pp. 7167--7176 (2017)

\bibitem{spot}
Xie, Y., Chen, M., Jiang, H., Zhao, T., Zha, H.: On scalable and efficient
  computation of large scale optimal transport. arXiv preprint arXiv:1905.00158
   (2019)

\bibitem{scalable_unbalanced_ot}
Yang, K.D., Uhler, C.: Scalable unbalanced optimal transport using generative
  adversarial networks. arXiv preprint arXiv:1810.11447  (2018)

\bibitem{dualgan}
Yi, Z., Zhang, H.R., Tan, P., Gong, M.: Dualgan: Unsupervised dual learning for
  image-to-image translation. In: ICCV. pp. 2868--2876 (2017)

\bibitem{dataset_shoes}
Yu, A., Grauman, K.: Fine-grained visual comparisons with local learning. In:
  CVPR. pp. 192--199 (2014)

\bibitem{wideresnet}
Zagoruyko, S., Komodakis, N.: Wide residual networks. arXiv preprint
  arXiv:1605.07146  (2016)

\bibitem{dataset_handbag}
Zhu, J.Y., Kr{\"a}henb{\"u}hl, P., Shechtman, E., Efros, A.A.: Generative
  visual manipulation on the natural image manifold. In: ECCV. pp. 597--613.
  Springer (2016)

\bibitem{cyclegan}
Zhu, J.Y., Park, T., Isola, P., Efros, A.A.: Unpaired image-to-image
  translation using cycle-consistent adversarial networks. arXiv preprint
  (2017)

\end{thebibliography}

\appendix
\setcounter{prop}{0}

\section{Proof of Proposition \ref{one-side cycle-consistency}}

\begin{prop}
	\label{one-side cycle-consistency_appendix}
	Given two distributions $\mu$ and $\nu$ defined in domain $X$ and $Y$ respectively and two stochastic mappings $G_{xy}: X \rightarrow Y $ and $G_{yx}: Y \rightarrow X$. If $G_{yx}\#\nu = \mu$ and $L_{cycle}(\nu)=0$, then 
	\begin{enumerate}
		\item $G_{xy}$ becomes a deterministic mapping;
		\item $\forall ~ y_1$, $y_2$, if $y_1 \neq y_2$, then $p(G_{yx}(y_1)=G_{yx}(y_2))=0$.
	\end{enumerate}
\end{prop}

\begin{proof}  $ $ \\
	\vspace{-10pt}
	\begin{enumerate}
		
		\item Suppose $G_{xy}$ is not deterministic.

		$G_{xy}$ is not deterministic means $\exists x_0$, $y_1$, $y_2$, $y_1 \neq y_2$ such that $p(G_{xy}(x_0) = y_1) > 0$ and $p(G_{xy}(x_0) = y_2) > 0$. 		
		
		Since $G_{yx}\#\nu = \mu$, we have $\exists y_0$ such that $p(G_{yx}(y_0)=x_0) > 0$.

		As $y_1 \neq y_2$, then at least one of $y_1 \neq y_0$ and $y_2 \neq y_0$ holds. 
		
		Then $L_{cycle}(\nu) \geq \mathbb{E}_{x \sim G_{yx}(y_0)}\mathbb{E}_{\hat{y} \sim G_{xy}(x)} [\|\hat{y} - y_0\|_2] > 0$, which conflicts with the condition $L_{cycle}(\nu) = 0$. 		
		So the hypothesis ``$G_{xy}$ is not deterministic'' does not hold.
		
		Therefore, $G_{xy}$ is deterministic. 
		
		\item Suppose $\exists y_1$, $y_2$, $y_1 \neq y_2$ and $p(G_{yx}(y_1)=G_{yx}(y_2)) > 0$. 
		
		Since $G_{yx}\#\nu = \mu$, we have $\exists x_0$ such that $p(G_{yx}(y_1)=x_0) > 0$ and $p(G_{yx}(y_2)=x_0) > 0$. 	
		
		We already have $G_{xy}$ is deterministic, so at least one of $G_{xy}(x_0) \neq y_1$ and $G_{xy}(x_0) \neq y_2$ holds. 	
		
		Then $L_{cycle}(\nu) \geq (\mathbb{E}_{x \sim G_{yx}(y_1)}\mathbb{E}_{\hat{y} \sim G_{xy}(x)} [\|\hat{y} - y_1\|_2] + \mathbb{E}_{x \sim G_{yx}(y_2)}\mathbb{E}_{\hat{y} \sim G_{xy}(x)} [\|\hat{y} - y_2\|_2]) > 0$, which is conflict with the condition $L_{cycle}(\nu) = 0$. 			
		So the hypothesis does not hold. 
		
		Therefore, $\forall ~ y_1$, $y_2$, if $y_1 \neq y_2$, then $p(G_{yx}(y_1)=G_{yx}(y_2))=0$. 
	\end{enumerate}
\end{proof}

\section{Proof of Proposition \ref{two-side cycle-consistency}}

Before proving Proposition \ref{two-side cycle-consistency}, we first present the following lemma: 
\begin{lemma}
	\label{lemma}
	If stochastic mapping $G_{xy}$ satisfies the following conditions:
	\begin{enumerate}
		\item $G_{xy}\#\mu = \nu$;
		\item $G_{xy}$ is deterministic;
		\item $\forall ~ x_1$, $x_2$, if $x_1 \neq x_2$, then $p(G_{xy}(x_1)=G_{xy}(x_2))=0$,
	\end{enumerate}
	then $G_{xy}$ is a bijection from $\mu$ to $\nu$.
\end{lemma}

The proof is straight forward: Given $G_{xy}$ is deterministic, then condition 1 means $G_{xy}$ is a surjection; condition 3 means $G_{xy}$ is a injection. So $G_{xy}$ is a bijection. 

\begin{prop}
	\label{two-side cycle-consistency_appendix}
	Given two distributions $\mu$ and $\nu$ defined in domain $X$ and $Y$ respectively and two stochastic mappings $G_{xy}: X \rightarrow Y$ and $G_{yx}: Y \rightarrow X$. If $G_{xy}\#\mu = \nu$, $G_{yx}\#\nu = \mu$, $L_{cycle}(\mu)=0$ and $L_{cycle}(\nu)=0$, then $G_{xy}$, $G_{yx}$ becomes bijections.
\end{prop}

\begin{proof} $ $ \\
	\indent
	Since $G_{yx}\#\nu = \mu$ and $L_{cycle}(\nu)=0$, according to Proposition \ref{one-side cycle-consistency}, we have $G_{xy}$ is deterministic. Since $G_{xy}\#\mu = \nu$ and $L_{cycle}(\mu)=0$, according to Proposition \ref{one-side cycle-consistency}, we have $\forall ~ x_1$, $x_2$, if $x_1 \neq x_2$, $p(G_{xy}(x_1)=G_{xy}(x_2))=0$. Besides, $G_{xy}\#\mu = \nu$, then according to Lemma \ref{lemma}, $G_{xy}$ is a bijection from $\mu$ to $\nu$. 
	
	%		\noindent
	The same reason, $G_{yx}$ is a bijection from $\nu$ to $\mu$. 
\end{proof}

\begin{figure}[!b]
	\begin{subfigure}[b]{0.198\linewidth}
		\centering
		\includegraphics[width=0.99\linewidth]{./samples_4gau.pdf} 
	\end{subfigure}%% 
	\begin{subfigure}[b]{0.198\linewidth}
		\centering
		\includegraphics[width=0.99\linewidth]{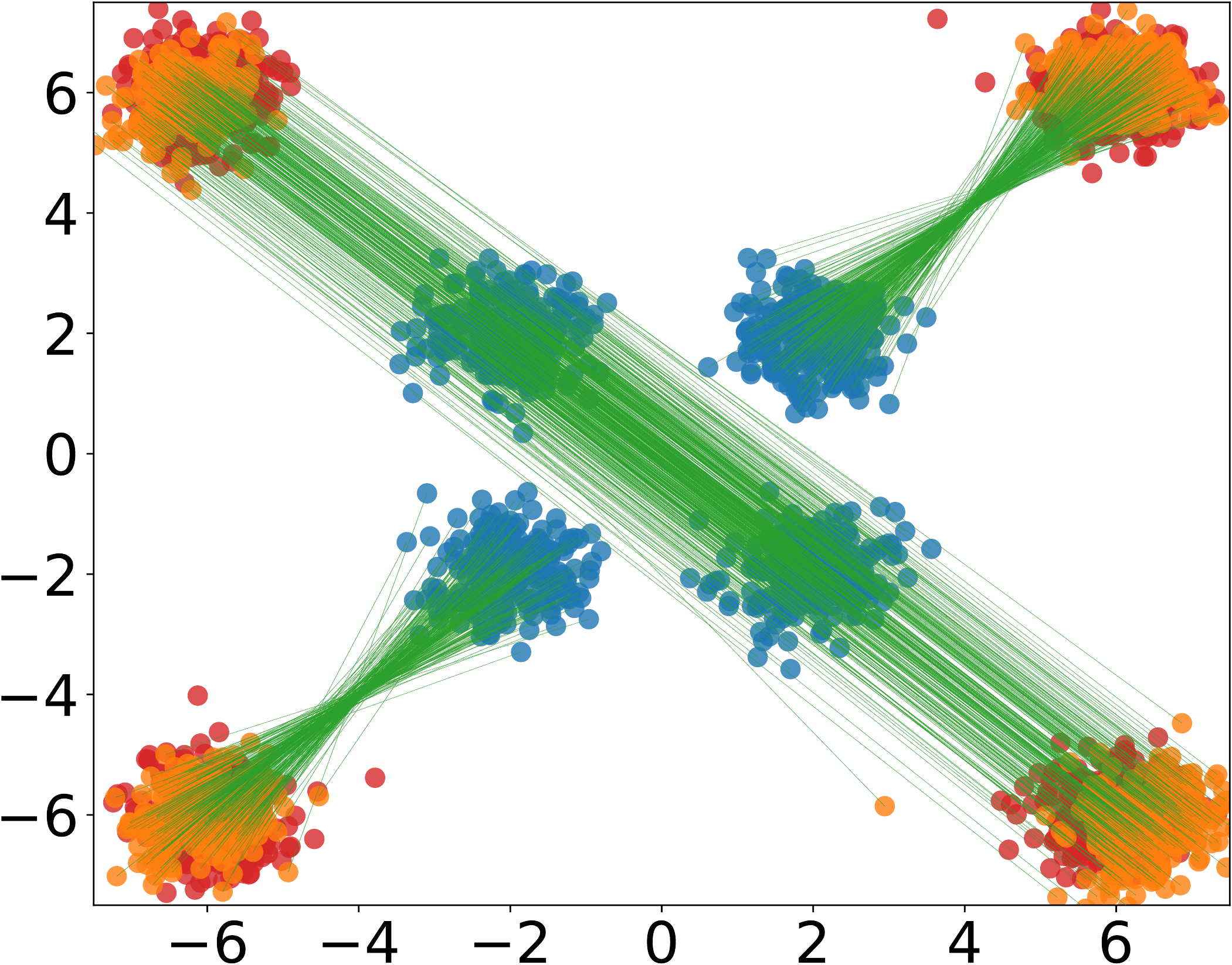} 
	\end{subfigure}%% 
	\begin{subfigure}[b]{0.198\linewidth}
		\centering
		\includegraphics[width=0.99\linewidth]{./mapping_kantorovich_4gau.pdf} 
	\end{subfigure}%% 
	\begin{subfigure}[b]{0.198\linewidth}
		\centering
		\includegraphics[width=0.99\linewidth]{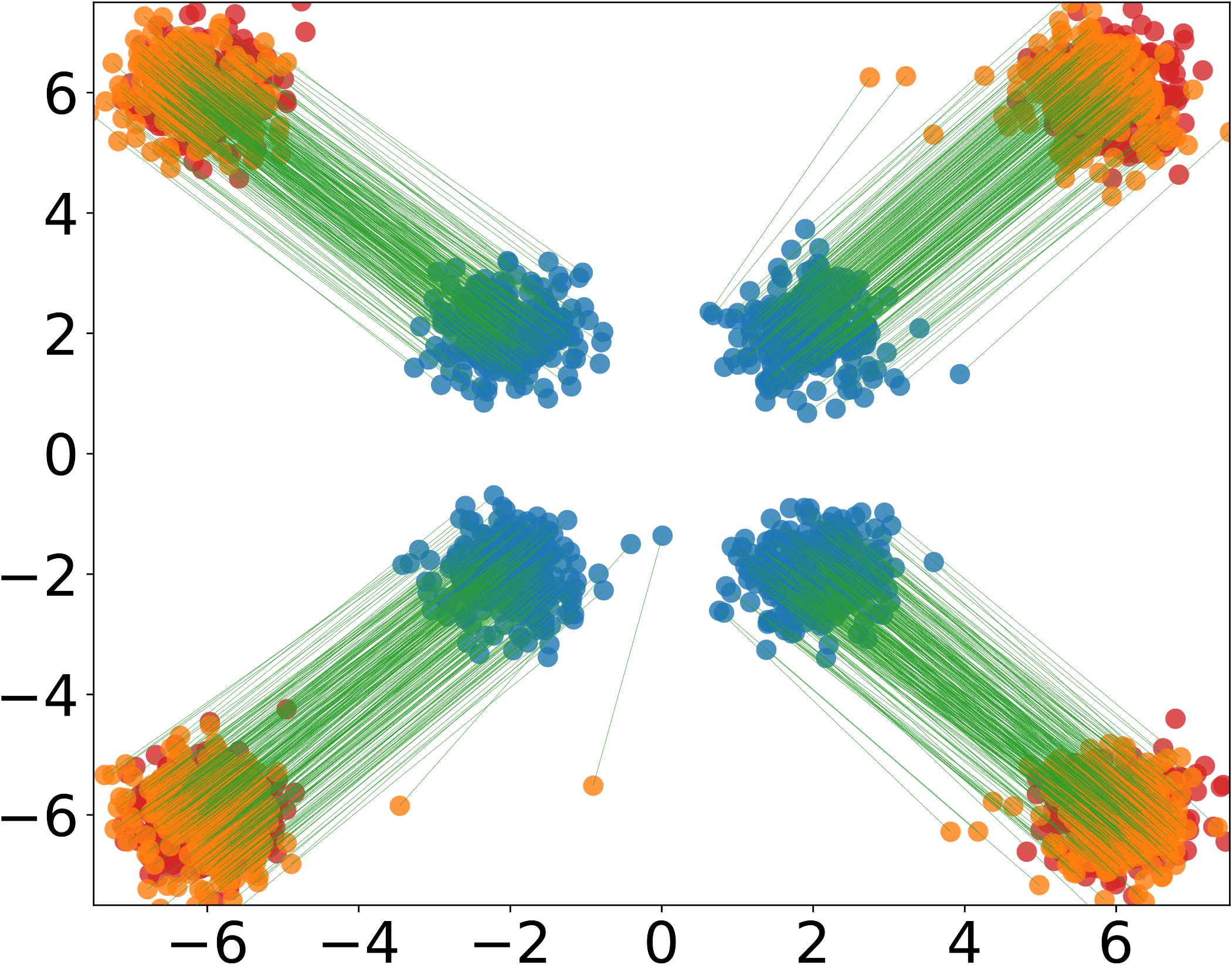} 
	\end{subfigure}%% 	
	\begin{subfigure}[b]{0.198\linewidth}
		\centering
		\includegraphics[width=0.99\linewidth]{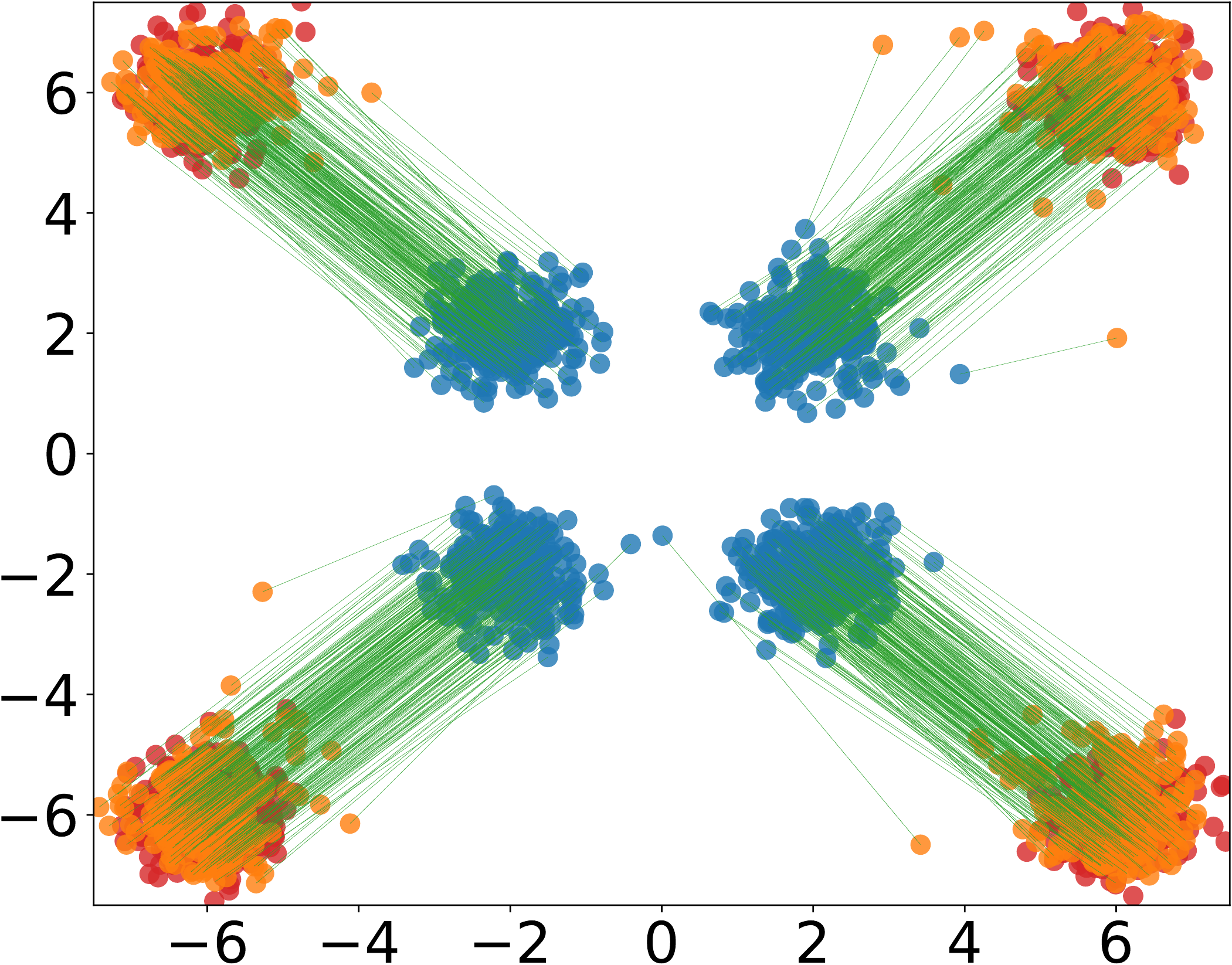} 
	\end{subfigure}%% 	
	
	\begin{subfigure}[b]{0.198\linewidth}
		\centering
		\includegraphics[width=0.99\linewidth]{./samples_gau_8gau.pdf} 
	\end{subfigure}%% 
	\begin{subfigure}[b]{0.198\linewidth}
		\centering
		\includegraphics[width=0.99\linewidth]{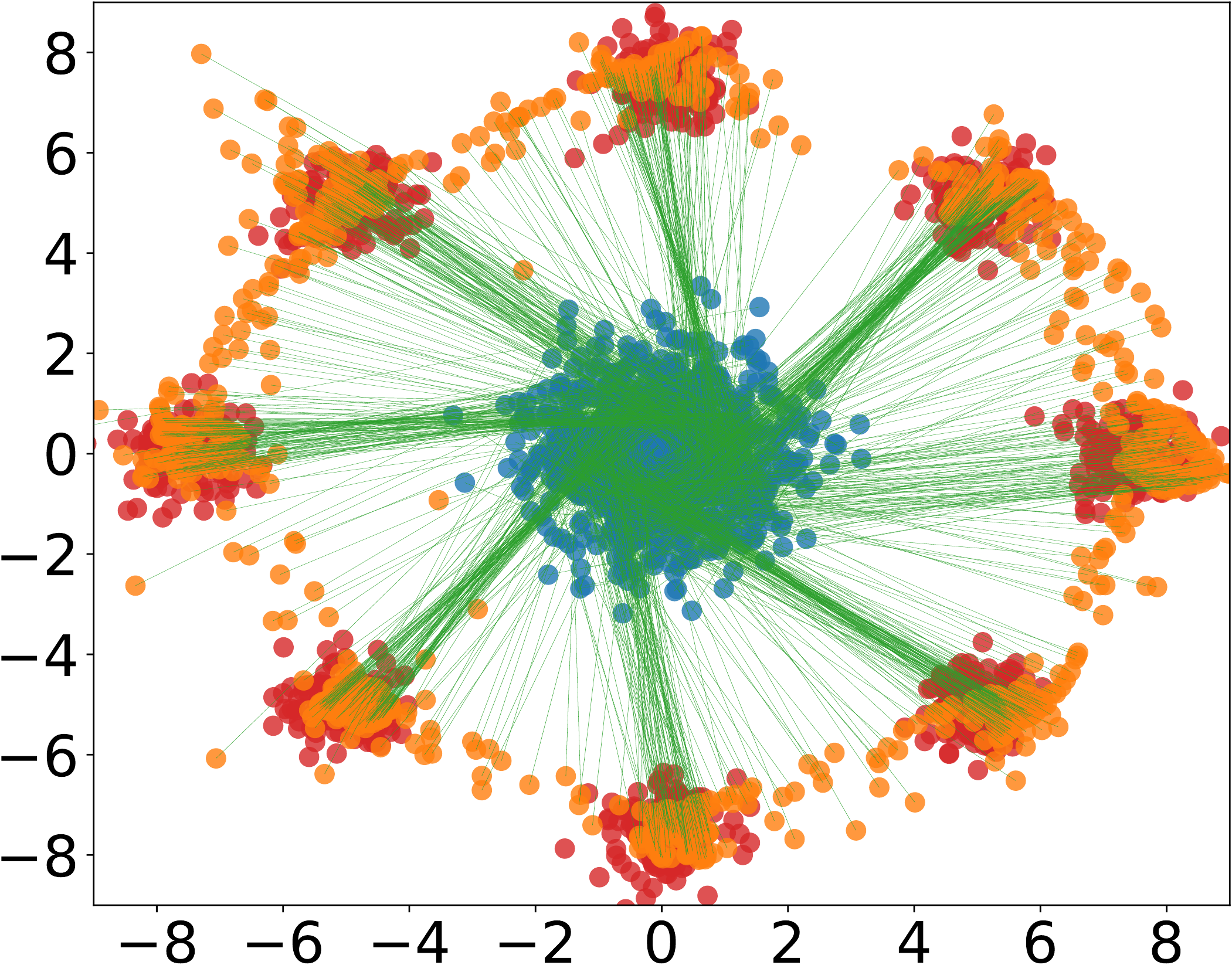} 
	\end{subfigure}%% 
	\begin{subfigure}[b]{0.198\linewidth}
		\centering
		\includegraphics[width=0.99\linewidth]{./mapping_kantorovich_gau_8gau.pdf} 
	\end{subfigure}%% 
	\begin{subfigure}[b]{0.198\linewidth}
		\centering
		\includegraphics[width=0.99\linewidth]{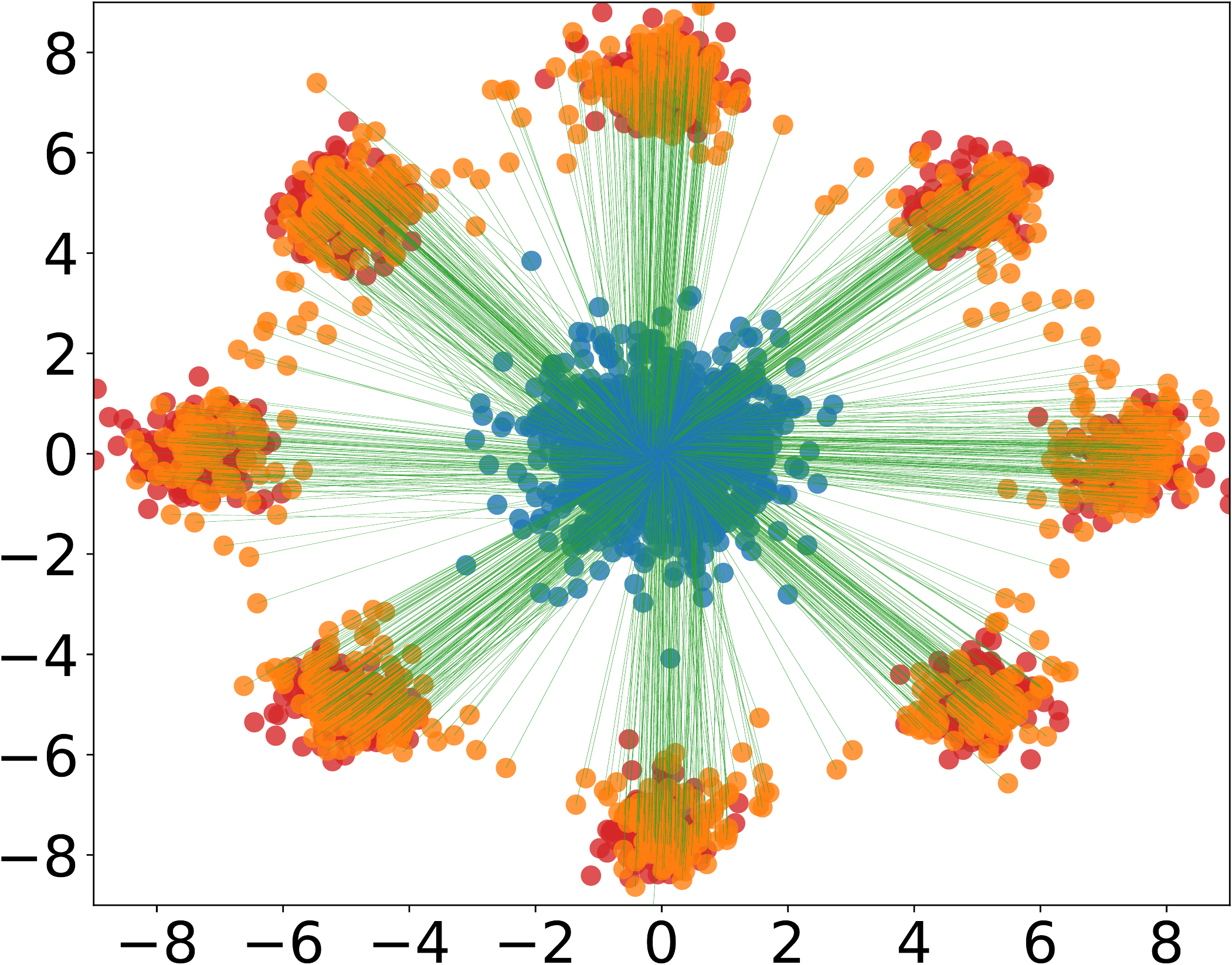} 
	\end{subfigure}%% 	
	\begin{subfigure}[b]{0.198\linewidth}
		\centering
		\includegraphics[width=0.99\linewidth]{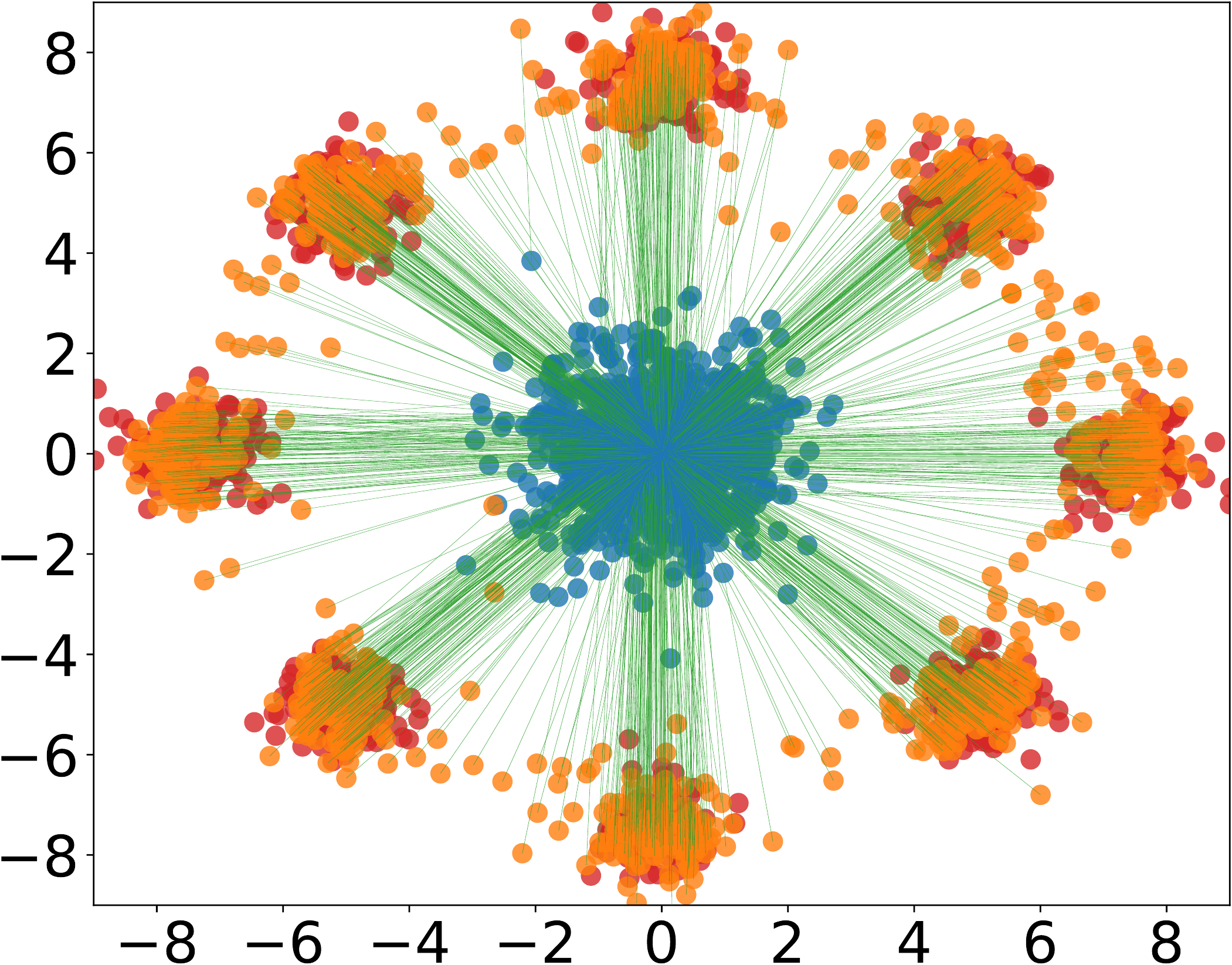} 
	\end{subfigure}%% 	
	
	\begin{subfigure}[b]{0.198\linewidth}
		\centering
		\hspace{-3pt}
		\includegraphics[width=0.99\linewidth]{./samples_checkerboard.pdf} 
		\caption{Samples} 
	\end{subfigure}%% 
	\begin{subfigure}[b]{0.198\linewidth}
		\centering
		\hspace{-3pt}
		\includegraphics[width=0.99\linewidth]{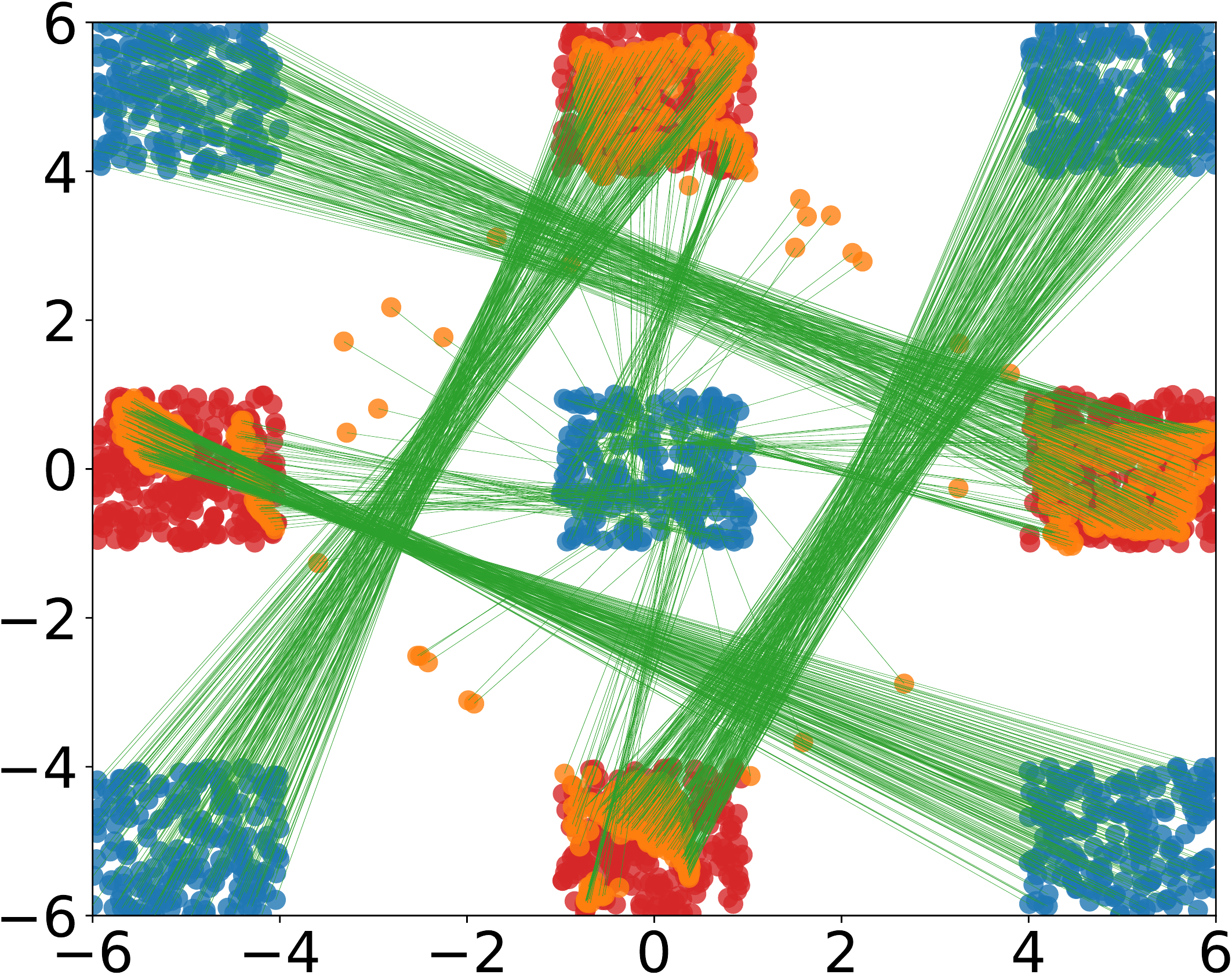} 
		\caption{WGAN-GP} 
	\end{subfigure}%% 
	\begin{subfigure}[b]{0.198\linewidth}
		\centering
		\hspace{-3pt}
		\includegraphics[width=0.99\linewidth]{./mapping_kantorovich_checkerboard.pdf} 
		\caption{K-solver} 
		\label{toy2d_comparison_kantorovich} 
	\end{subfigure}%% 
	\begin{subfigure}[b]{0.198\linewidth}
		\centering
		\hspace{-3pt}
		\includegraphics[width=0.99\linewidth]{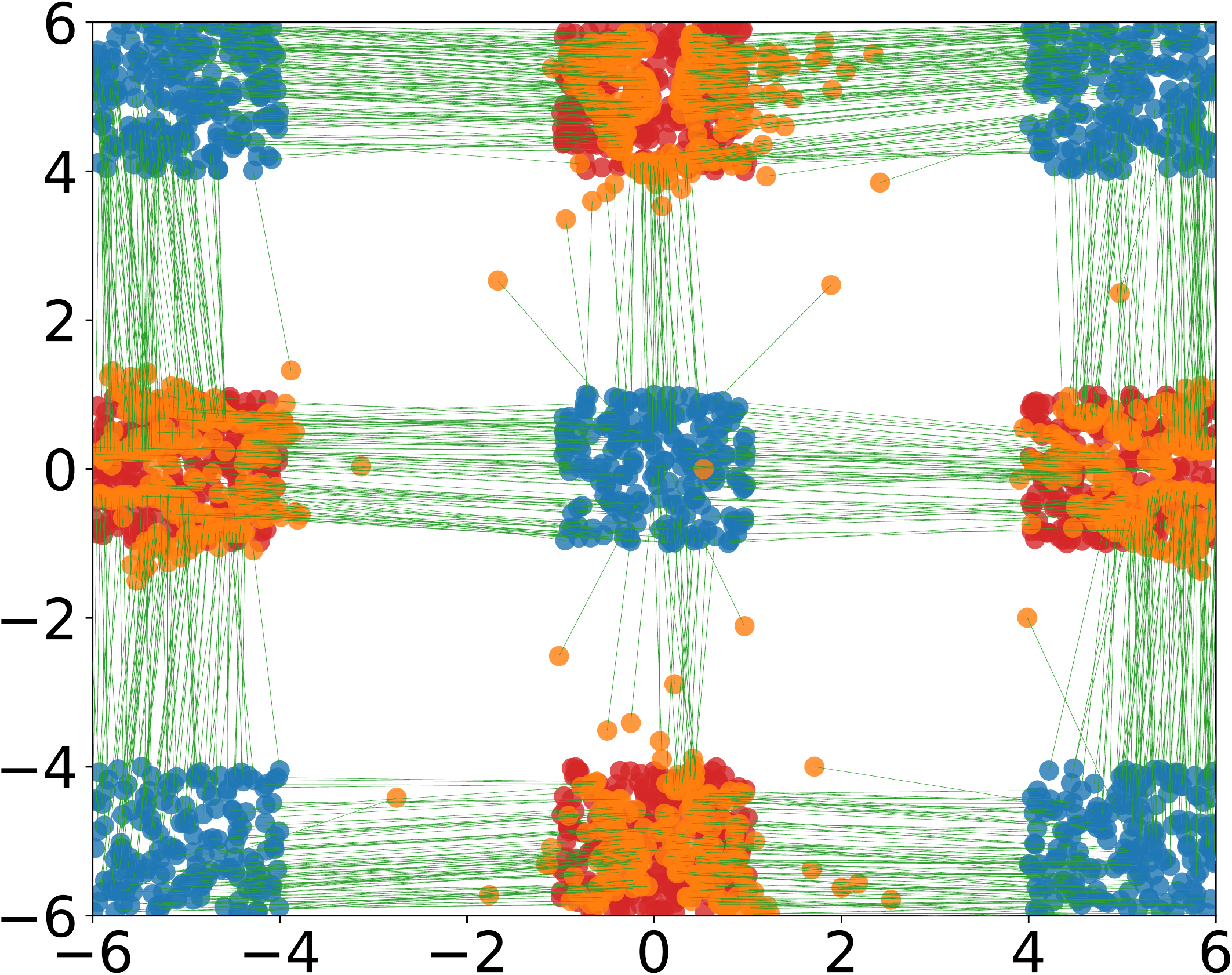} 
		\caption{M-solver} 
	\end{subfigure}
	\begin{subfigure}[b]{0.198\linewidth}
		\centering
		\hspace{-10pt}
		\includegraphics[width=0.99\linewidth]{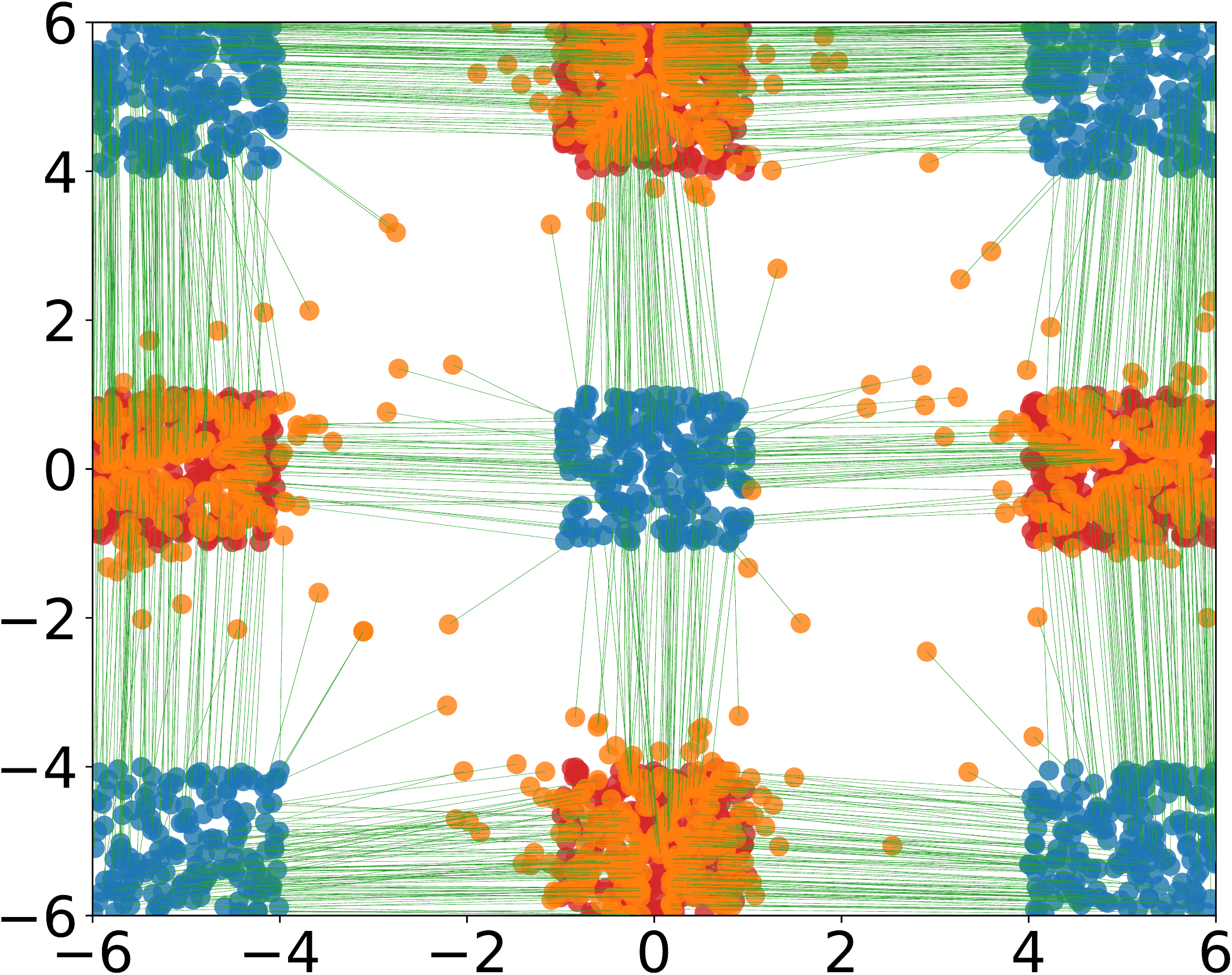} 
		\caption{B-solver} 
	\end{subfigure}
	\vspace{-3pt} 
	\caption{Mappings learned by WGAN-GP generator, our Kantorovich solver, Monge solver and Bijection solver on three 2D examples. Blue: source samples. Red: target samples. Orange: mapped samples. Green: the mapping. Number of samples are 1000.}
	\label{comparison_mapping2d}
	%	\vspace{-15pt}
\end{figure}

\section{Additional Experiment Results}

\subsection{Mappings Learned by WGAN-GP, Monge / Bijection Solver}

Fig. \ref{comparison_mapping2d} shows the mappings learned by the WGAN-GP generator, our Kantorovich solver, Monge solver and Bijection solver on three 2D examples. Results of BOT are provided in the main body. As we can see, WGAN-GP generator cannot learn the optimal transport map in general as there is no constraint on the learned map except that the push-forward of source distribution should be target distribution. BOT exhibits collapse and out of distribution samples. Kantorovich solver, Monge solver and Bijiection solver achieve better performance in general.

\subsection{Visual Results for Domain Adaptation}

Fig. \ref{da_showcase} shows some source samples and the corresponding mapped samples in different domain adaptation tasks. We can see that, our model learns the desired mapping, which maps samples from the source domain to samples from the target domain with the same class label.

\begin{figure}[h]
	\centering
	\includegraphics[width=1.\columnwidth]{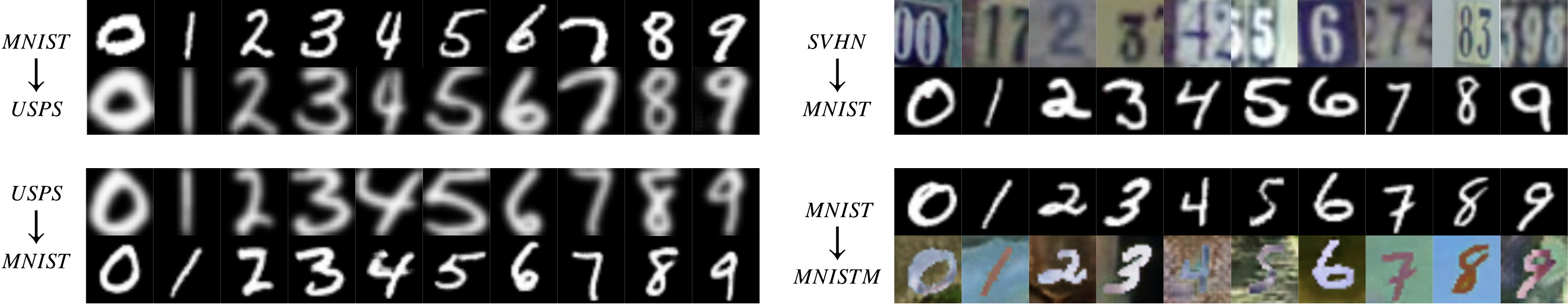}
	\caption{Source and mapped samples of different domain adaptations tasks.}
	\label{da_showcase}
	%		\vspace{-15pt}
\end{figure}

\subsection{Mappings Learned in Color Transfer}

Fig. \ref{color_transfer_fix_mapping} shows the learned mappings by different solvers in color transfer. As we can see, the mapping learned by Kantorovich solver is a stochastic mapping, while the ones learned by Monge solver and Bijection solver are deterministic mappings.

\begin{figure}[h]
	\vspace{-5pt}
	\centering
	\begin{subfigure}[b]{0.3\linewidth}
		\centering
		\includegraphics[width=0.9\linewidth]{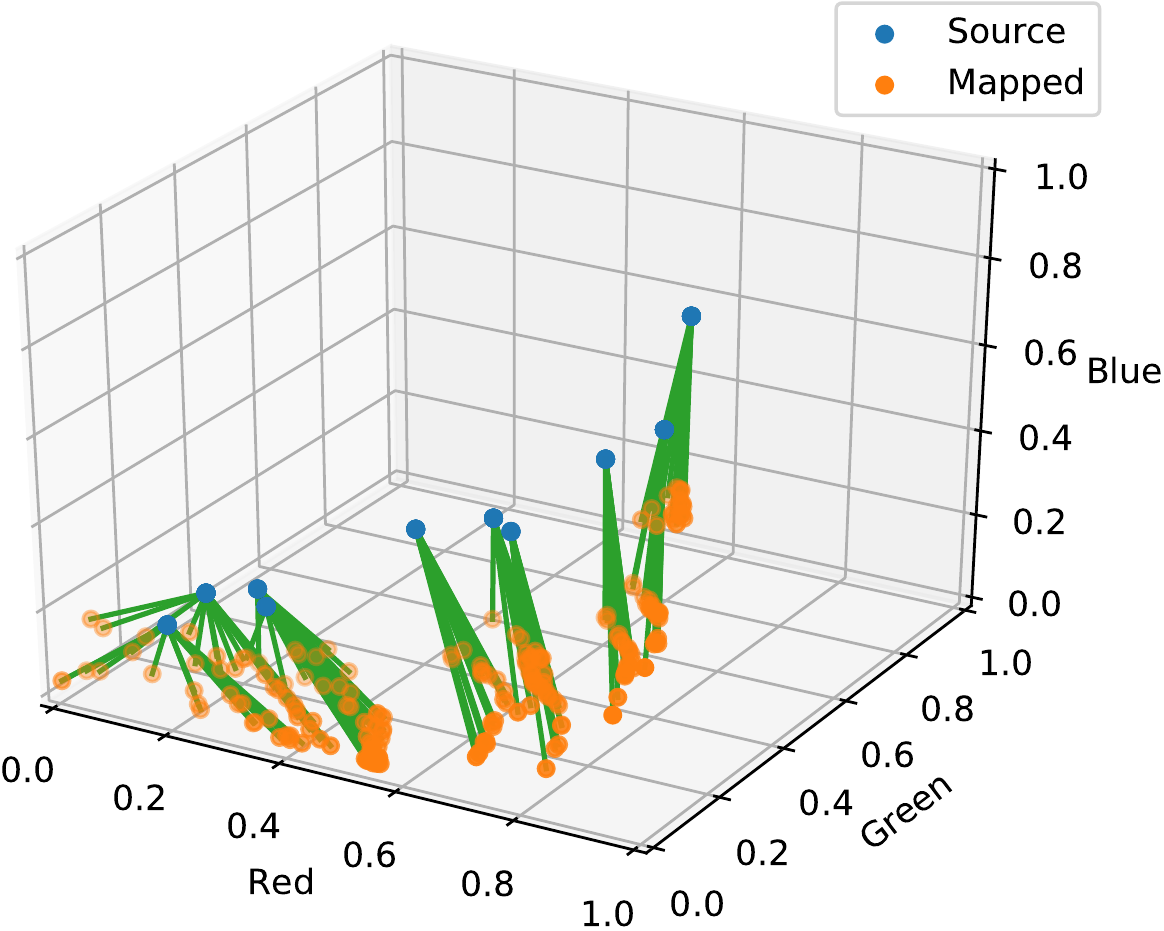} 
		\caption{Kantorovich solver} 
	\end{subfigure}%% 
	\begin{subfigure}[b]{0.3\linewidth}
		\centering
		\includegraphics[width=0.9\linewidth]{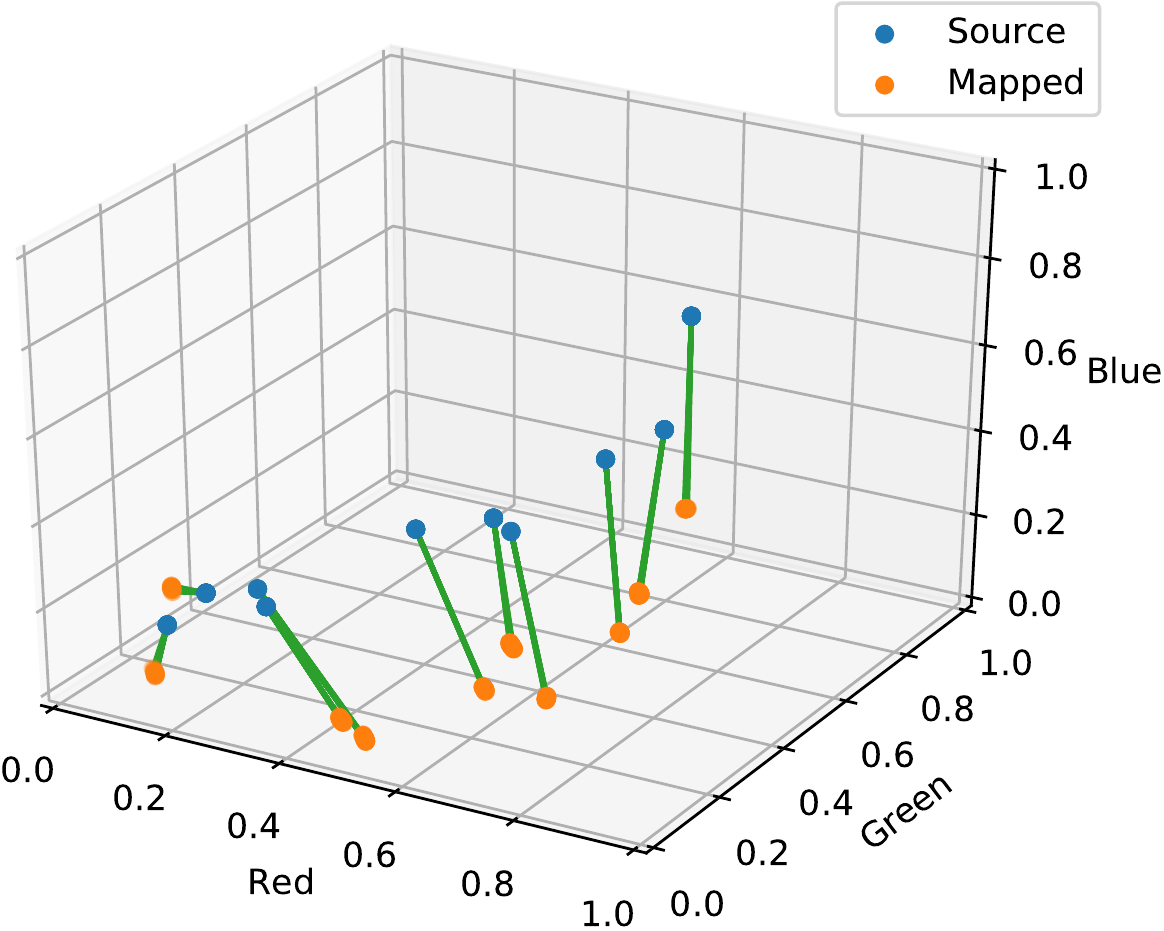}
		\caption{Monge solver}
	\end{subfigure}
	\begin{subfigure}[b]{0.3\linewidth}
		\centering
		\includegraphics[width=0.9\linewidth]{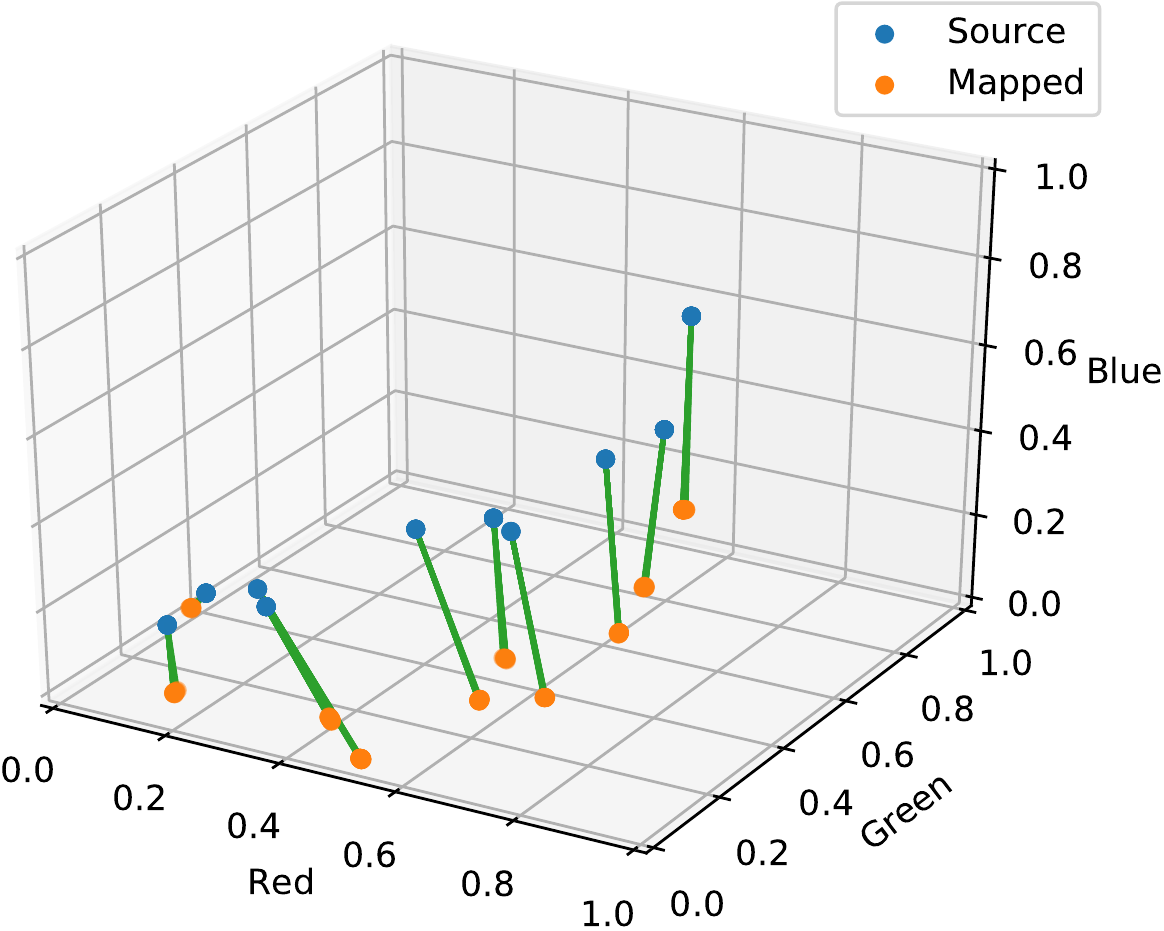}
		\caption{Bijection solver}
	\end{subfigure}
	\caption{Visualization of the mapping with various noise $z$ for color transfer.}
	\label{color_transfer_fix_mapping}
\end{figure}

\subsection{Kantorovich Solver with Different Noise $z$}

As our Kantorovich solver learns a stochastic mapping, in this part, we check the stochasticness of the mapping learned by the Kantorovich solver. 

\subsubsection{Toy Experiments}

Fig. \ref{k_discrete2continuous_different_z} and Fig. \ref{k_toy_different_z} shows results of Kantorovich solver on 2D toy examples. As we can see, different noise $z$ results in slightly different results, which indicates that our Kantorovich solver learns a stochastic mapping.

\begin{figure}[h]
	\vspace{-10pt}
	\begin{subfigure}[b]{1.\linewidth}
		\begin{subfigure}[b]{0.33\linewidth}
			\centering
			\includegraphics[width=0.9\linewidth]{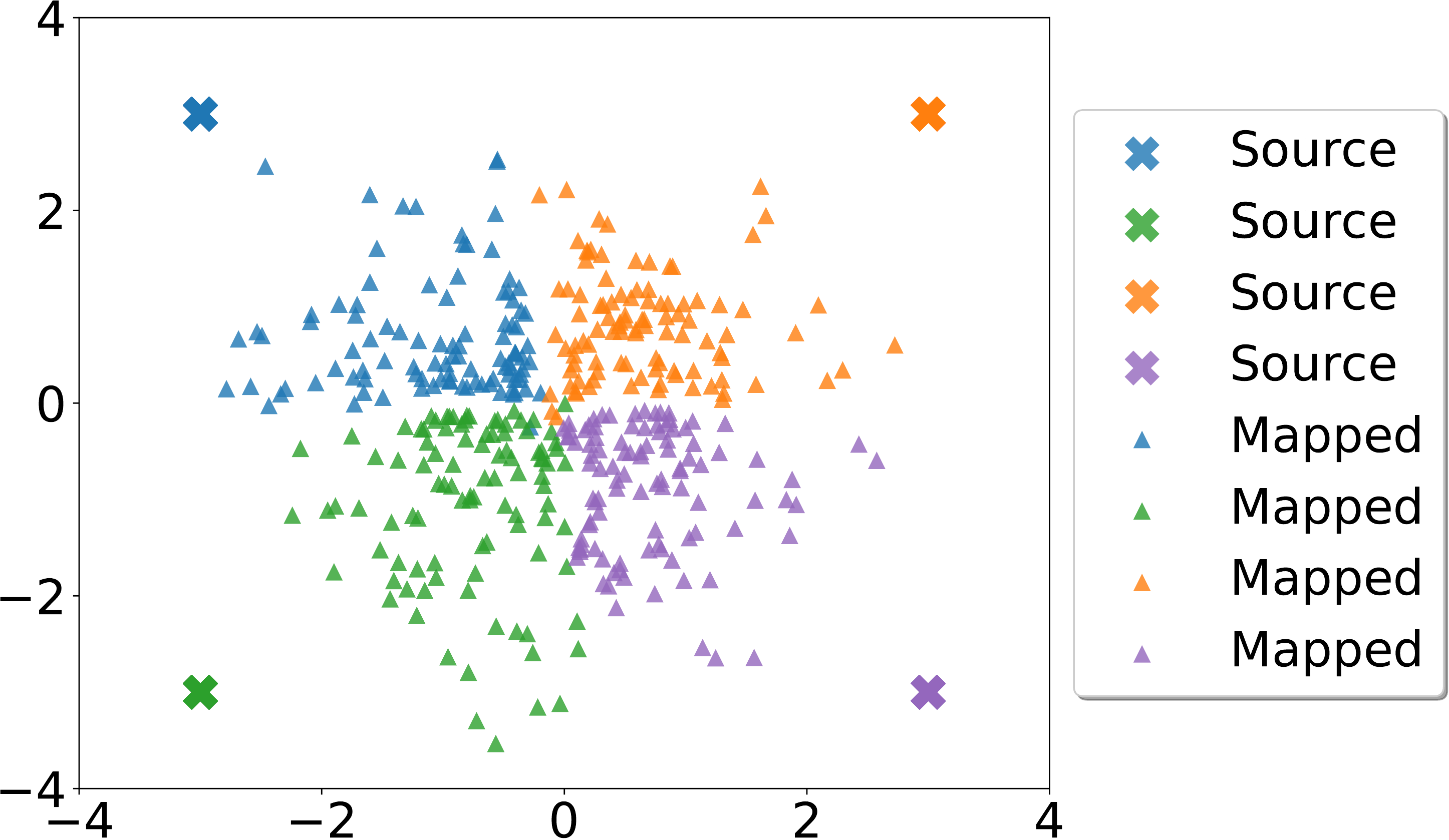} 
		\end{subfigure}%% 
		\begin{subfigure}[b]{0.33\linewidth}
			\centering
			\includegraphics[width=0.9\linewidth]{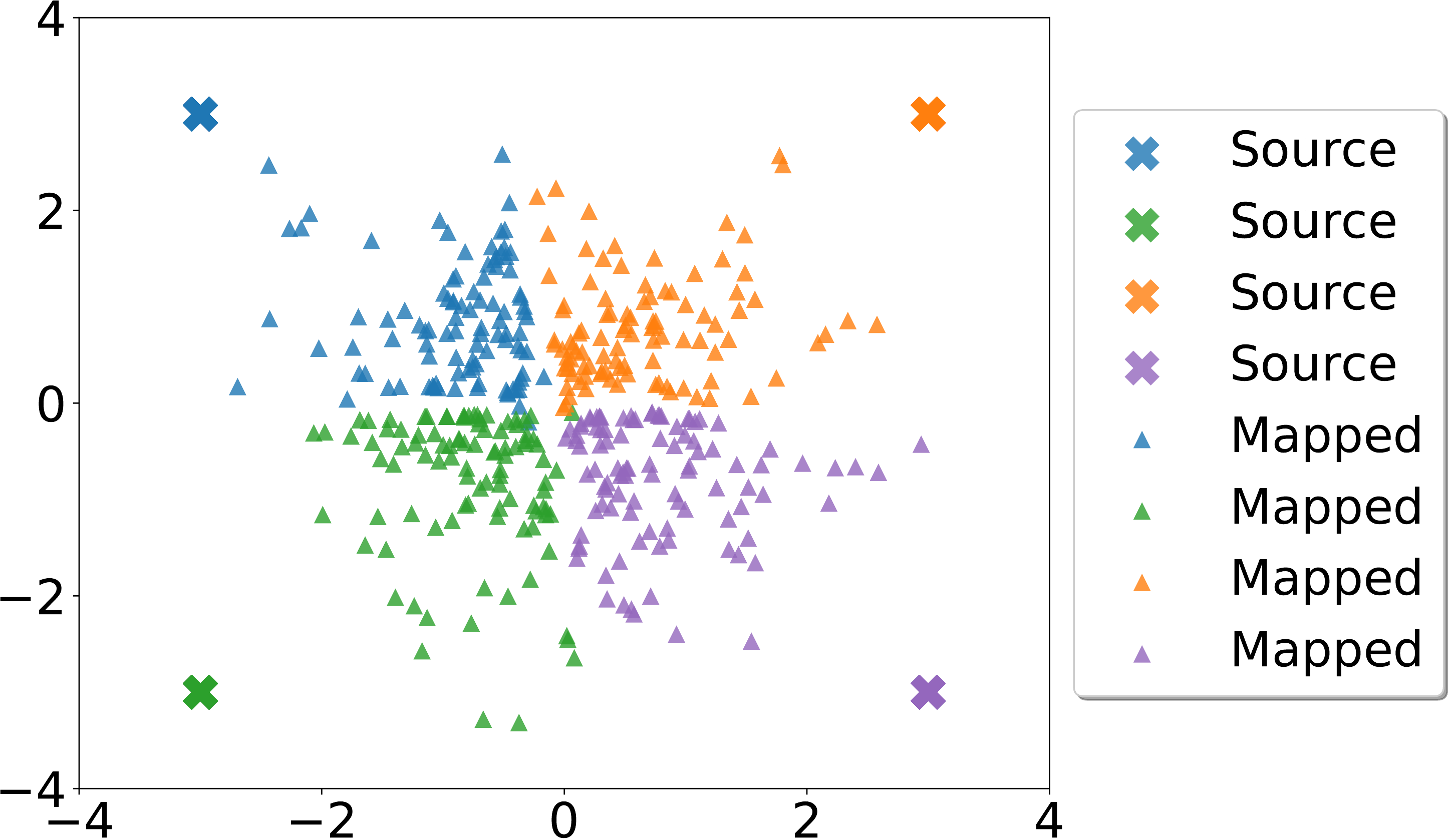} 
		\end{subfigure}%% 
		\begin{subfigure}[b]{0.33\linewidth}
			\centering
			\includegraphics[width=0.9\linewidth]{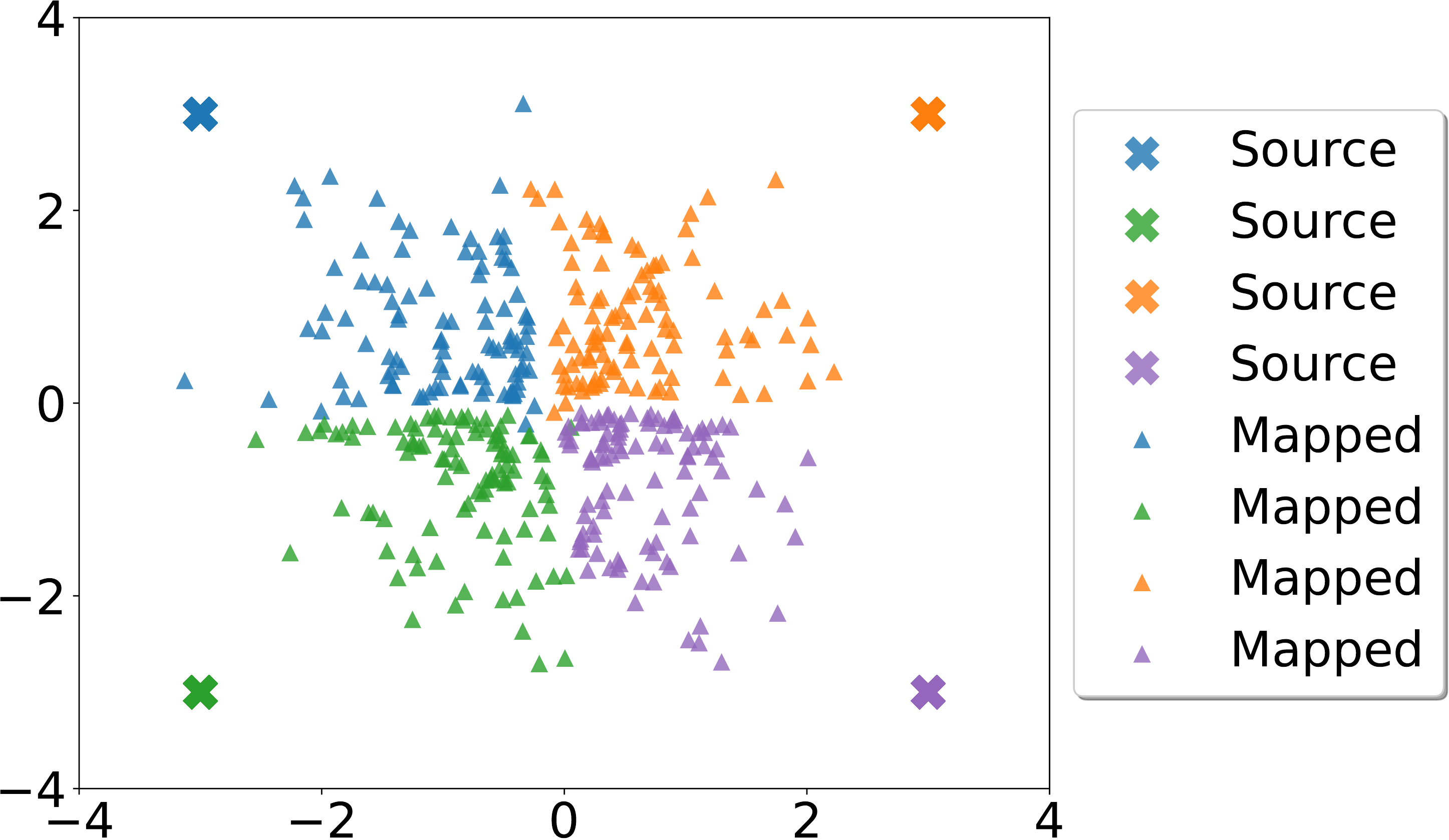} 
		\end{subfigure}%% 	
	\end{subfigure}
	\caption{Results with different noise $z$ on 2D discrete-to-continuous example.}
	\label{k_discrete2continuous_different_z}
%	\vspace{-25pt}
\end{figure}

\begin{figure}[!h]
	\begin{subfigure}[b]{1.\linewidth}
		\begin{subfigure}[b]{0.33\linewidth}
			\centering
			\includegraphics[width=0.7\linewidth]{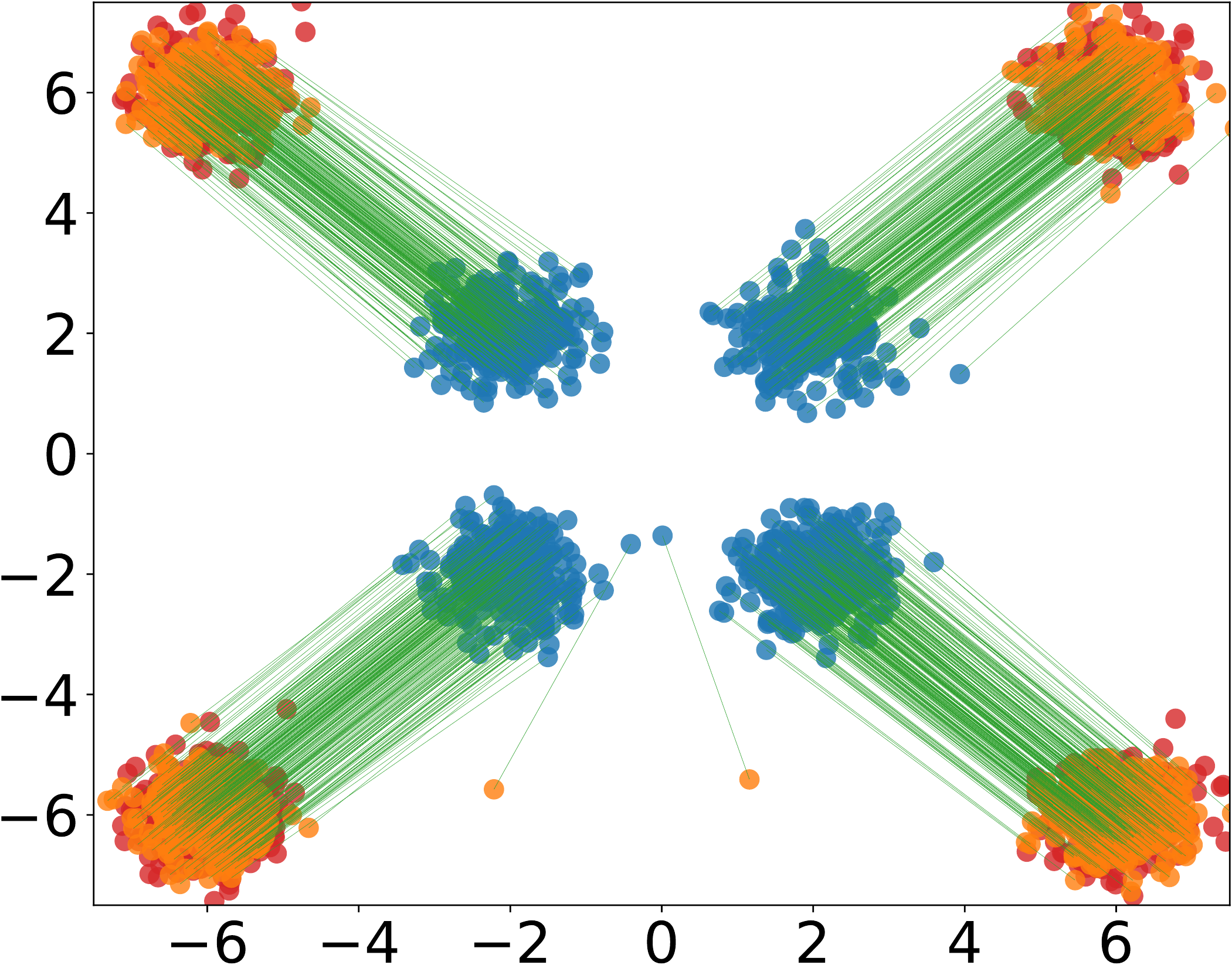} 
		\end{subfigure}%% 
		\begin{subfigure}[b]{0.33\linewidth}
			\centering
			\includegraphics[width=0.7\linewidth]{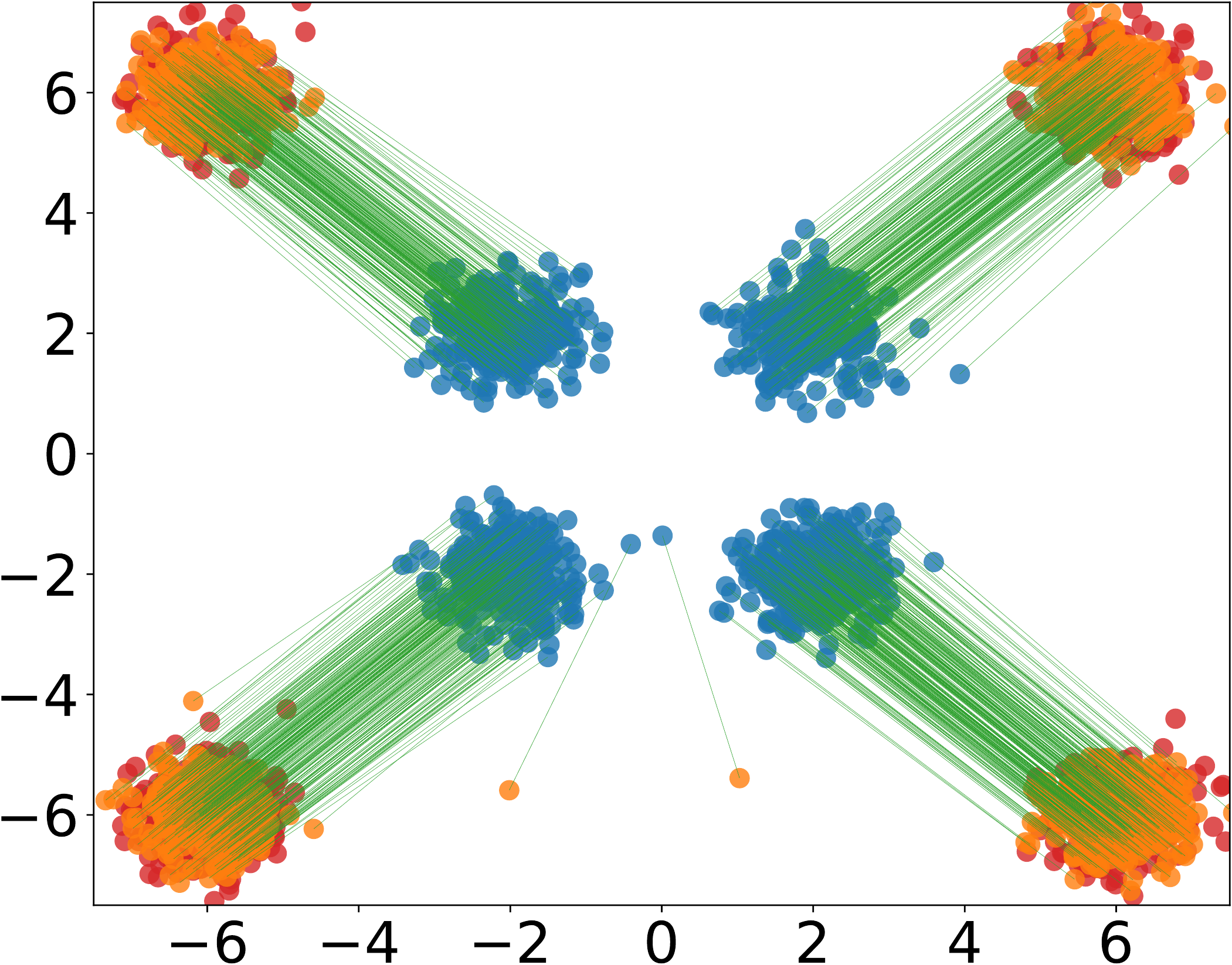} 
		\end{subfigure}%% 
		\begin{subfigure}[b]{0.33\linewidth}
			\centering
			\includegraphics[width=0.7\linewidth]{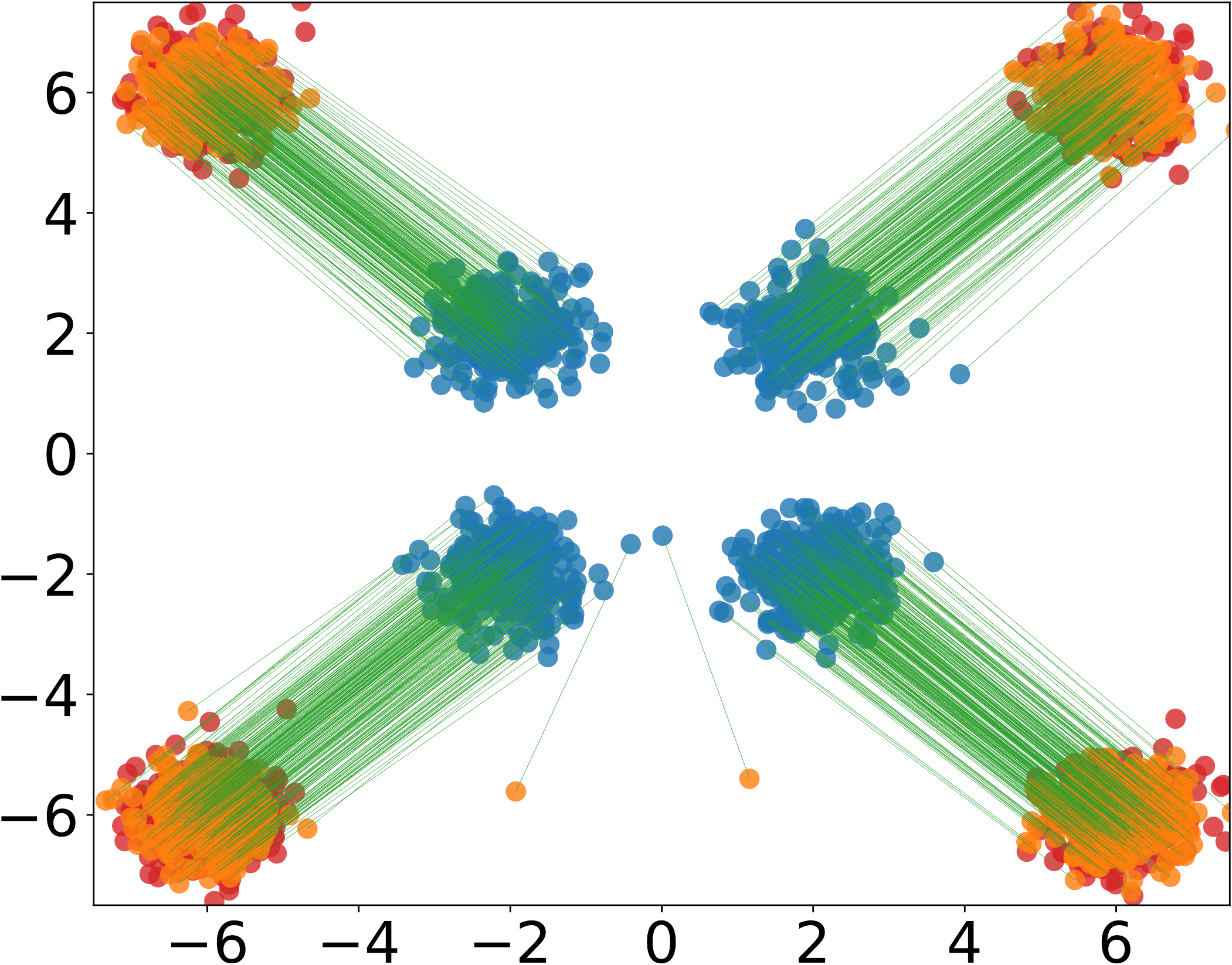} 
		\end{subfigure}%% 	
		\caption{Results with different noise $z$ on 4-Gaussian example.}
		\vspace{5pt}
	\end{subfigure}
	
	\begin{subfigure}[b]{1.\linewidth}
		\begin{subfigure}[b]{0.33\linewidth}
			\centering
			\hspace{-5pt}
			\includegraphics[width=0.7\linewidth]{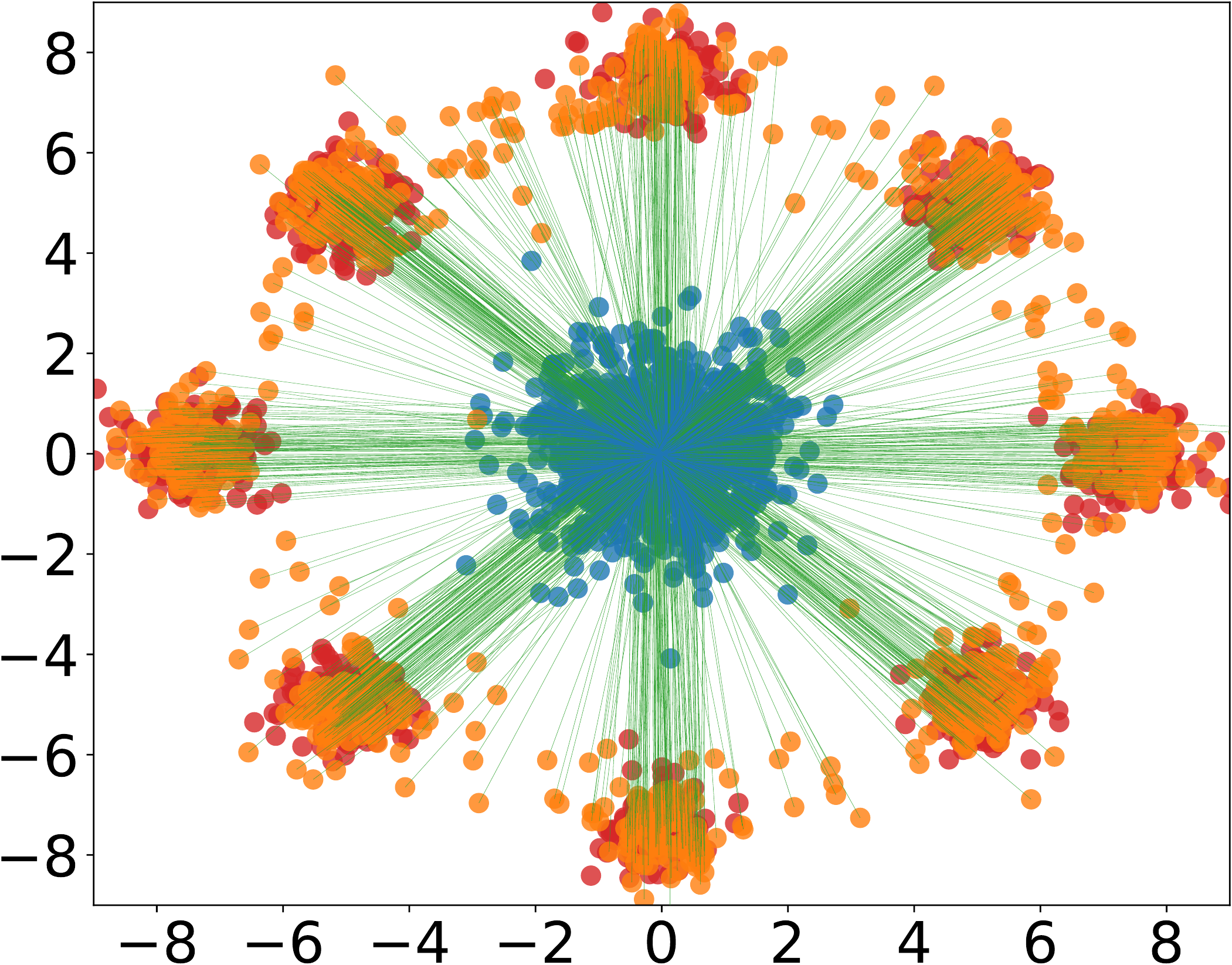} 
		\end{subfigure}%% 
		\begin{subfigure}[b]{0.33\linewidth}
			\centering
			\hspace{-5pt}
			\includegraphics[width=0.7\linewidth]{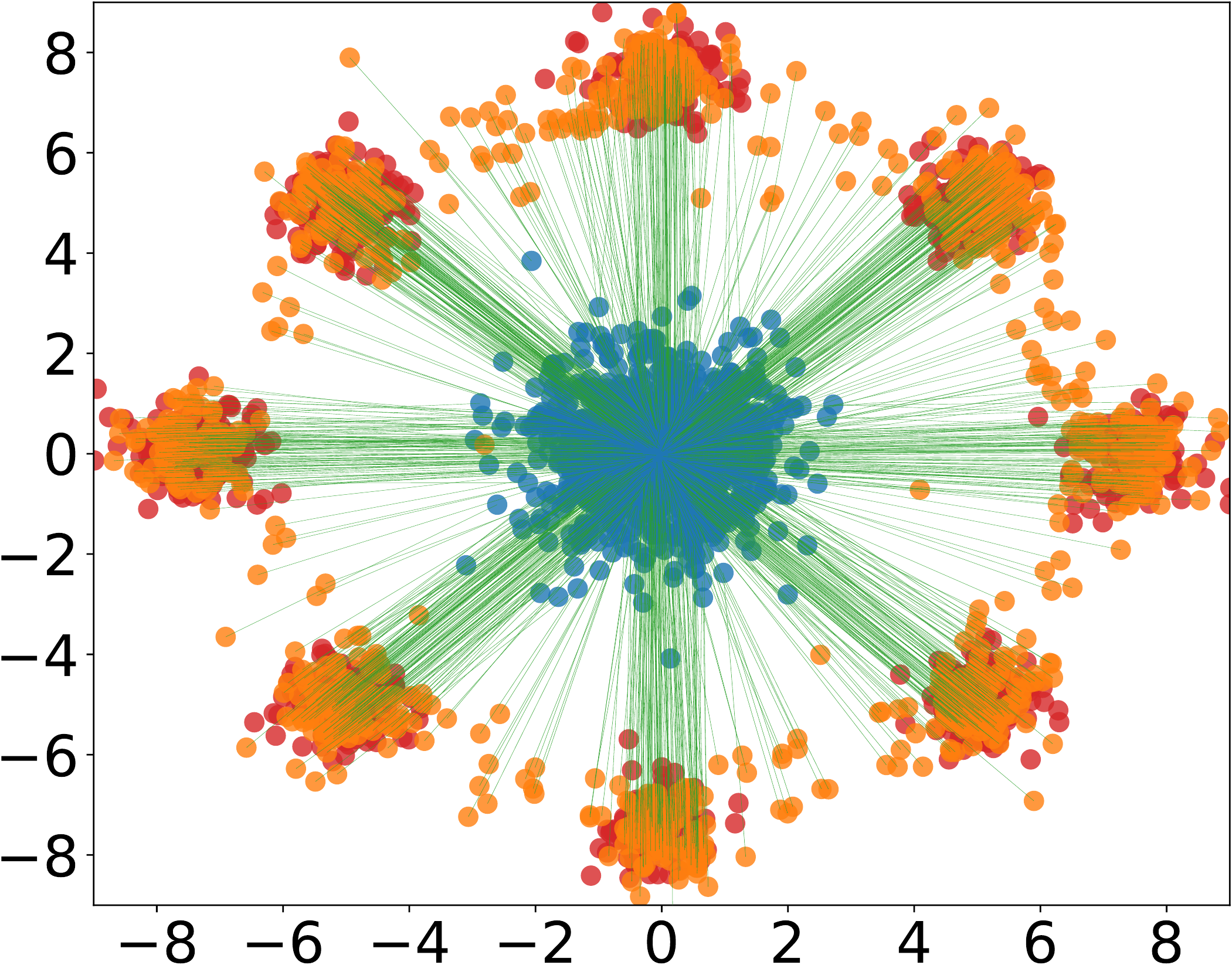} 
		\end{subfigure}%% 
		\begin{subfigure}[b]{0.33\linewidth}
			\centering
			\hspace{-5pt}
			\includegraphics[width=0.7\linewidth]{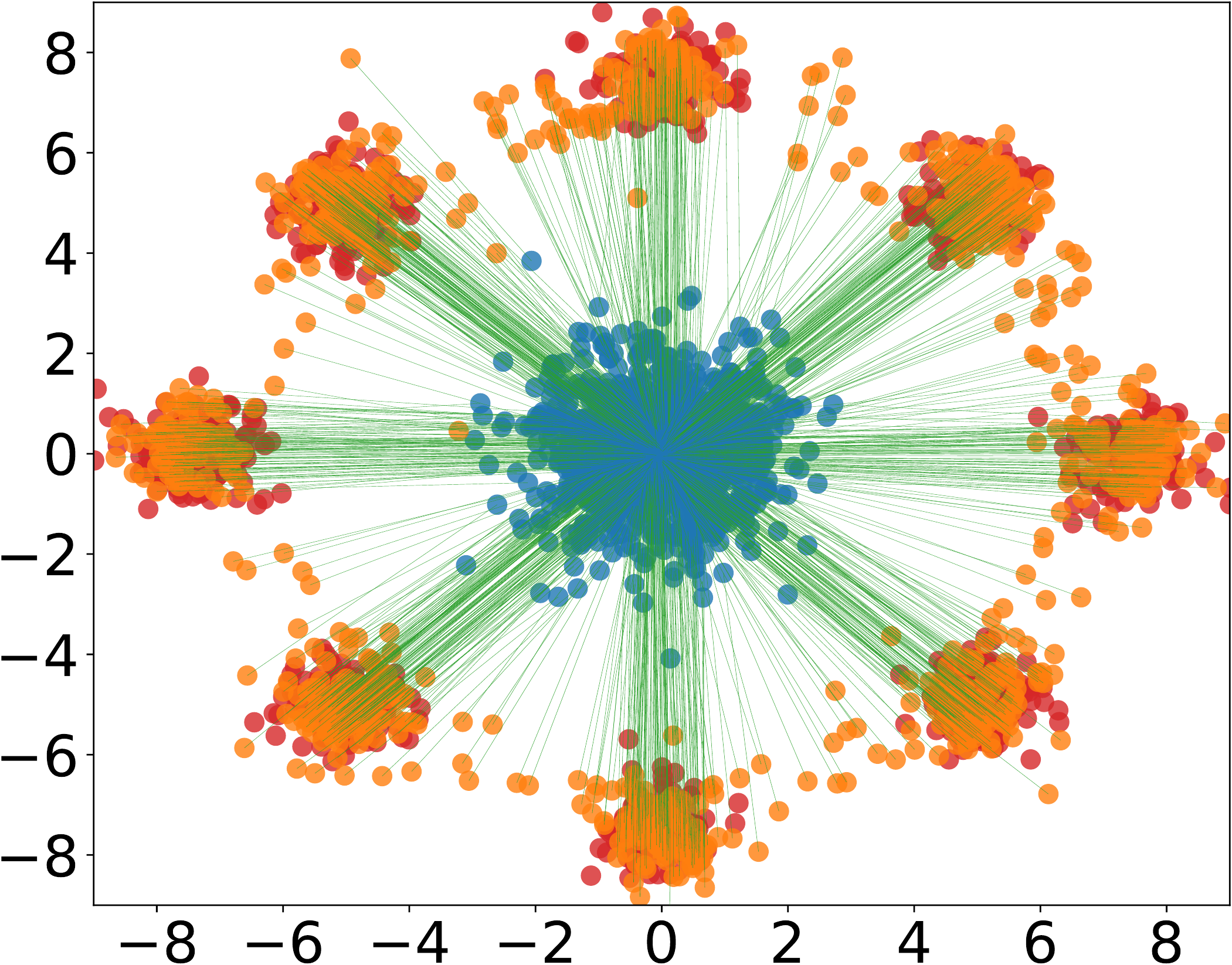} 
		\end{subfigure}%% 
		\caption{Results with different noise $z$ on 8-Gaussian example.}
		\vspace{5pt}
	\end{subfigure}
	
	\begin{subfigure}[b]{1.\linewidth}
		\begin{subfigure}[b]{0.33\linewidth}
			\centering
			\hspace{-5pt}
			\includegraphics[width=0.7\linewidth]{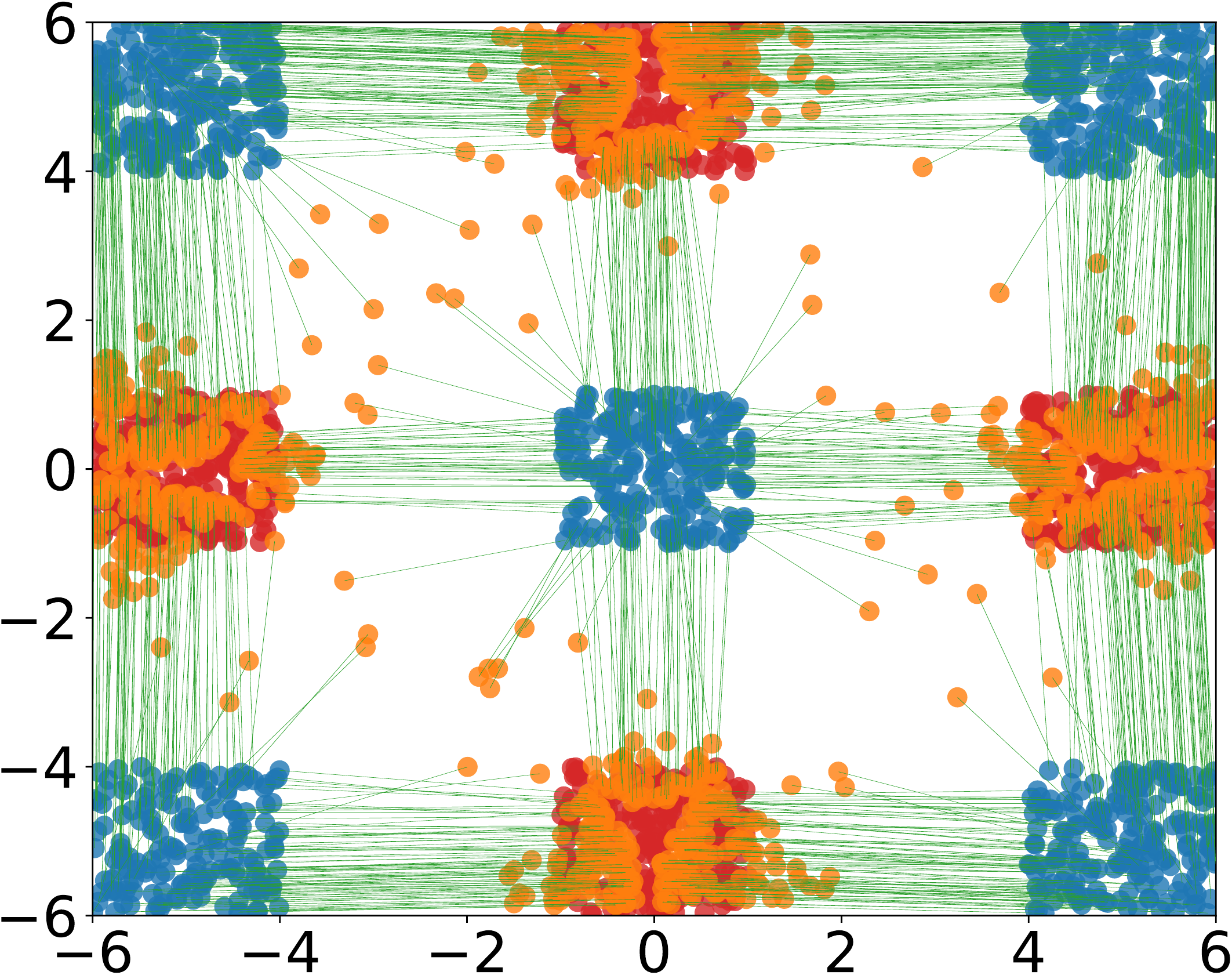} 
		\end{subfigure}%% 
		\begin{subfigure}[b]{0.33\linewidth}
			\centering
			\hspace{-5pt}
			\includegraphics[width=0.7\linewidth]{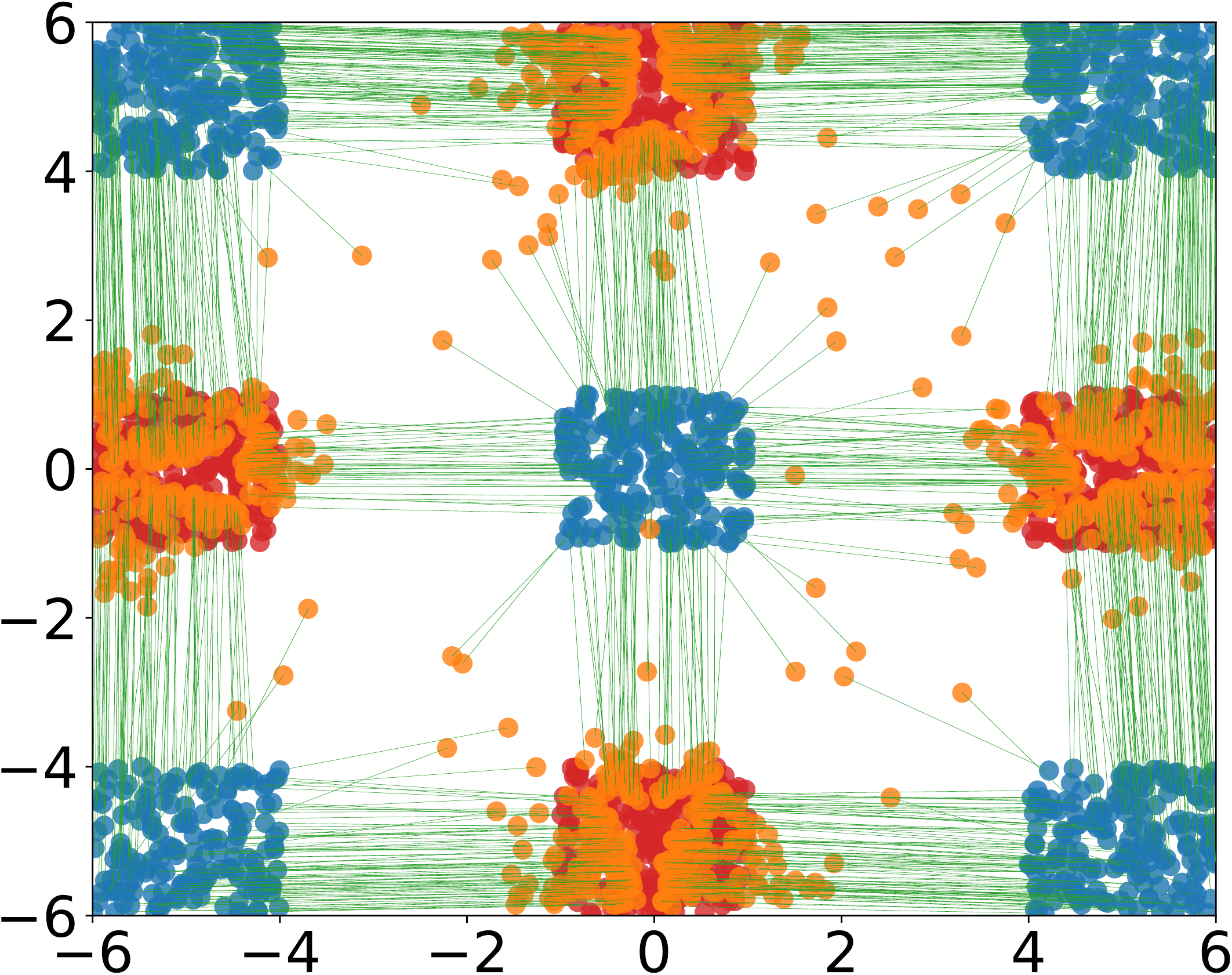} 
		\end{subfigure}%% 
		\begin{subfigure}[b]{0.33\linewidth}
			\centering
			\hspace{-5pt}
			\includegraphics[width=0.7\linewidth]{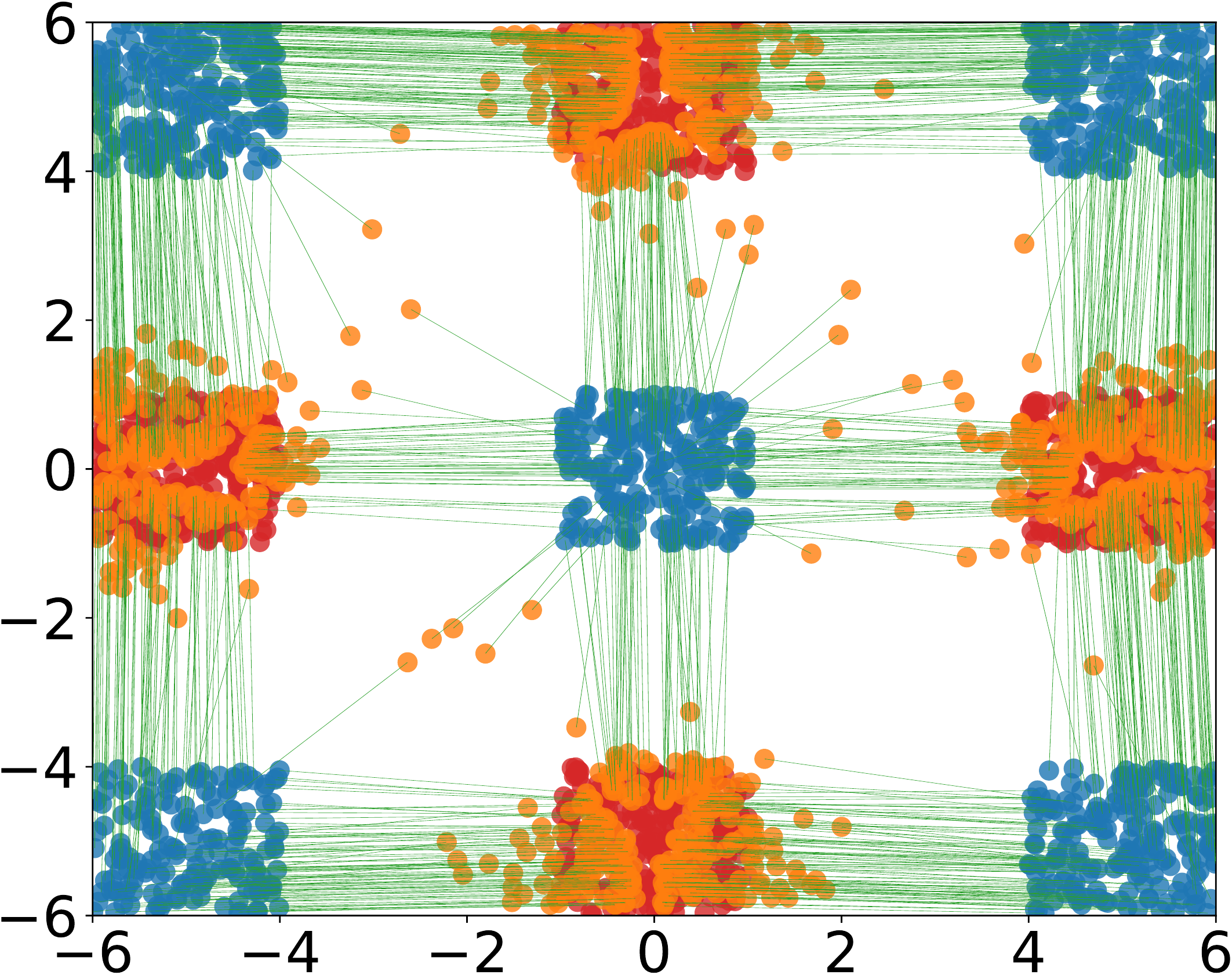} 
		\end{subfigure}%% 
		\caption{Results with different noise $z$ on Checkerboard example.}
		\vspace{5pt}
	\end{subfigure}
	\caption{Results of our Kantorovich solver with different noise $z$ on 2D examples. Blue: source samples. Red: target samples. Orange: mapped samples. Green: the mapping. Number of samples are 1000.}
	\label{k_toy_different_z}
	\vspace{-15pt}
\end{figure}

\subsubsection{Unsupervised Image-to-Image Translation}

Table \ref{image_translation_table_appendix} and Fig. \ref{image2image_translation_appendix} shows results of Kantorovich solver on unsupervised image-to-image translation over different noise $z$. According to Table \ref{image_translation_table_appendix}, we can see that Kantorovich solver learns a stochastic mapping as different noise $z$ results in slightly different scores. According to Fig. \ref{image2image_translation_appendix}, different noise $z$ yield visually similar results.

\begin{figure*}[!h]
	\vspace{-10pt}
	\begin{subfigure}[t]{0.5\textwidth}
		\centering
		\includegraphics[width=0.99\textwidth]{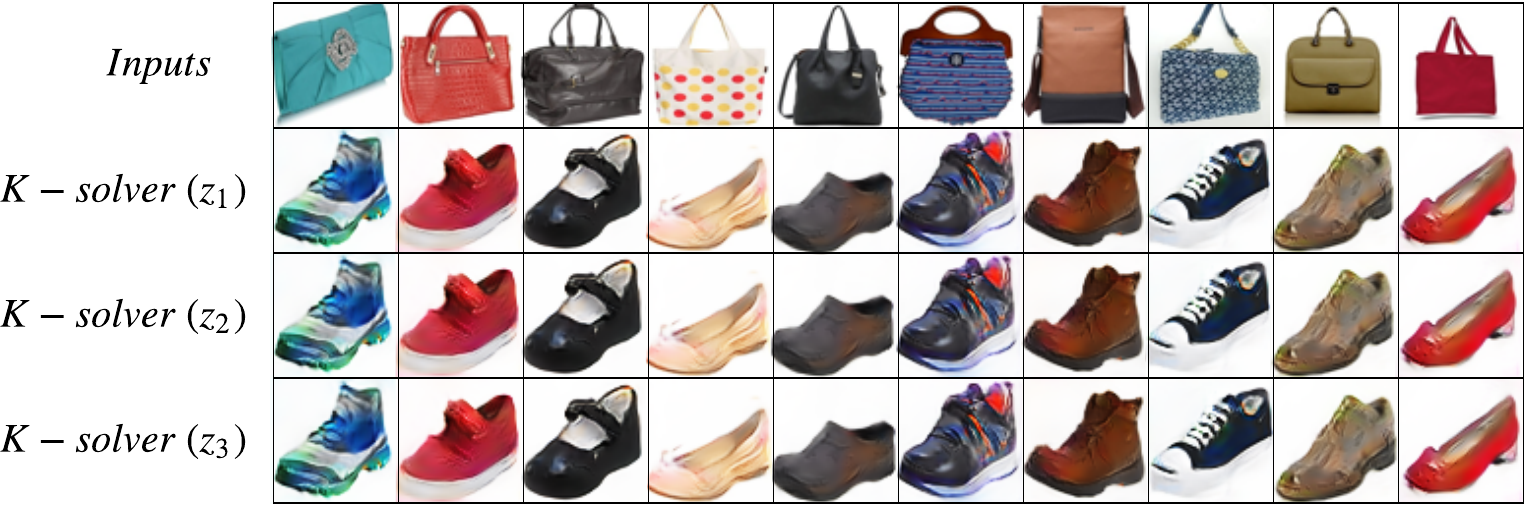}
		\caption{Handbags-to-Shoes}
	\end{subfigure}
	\begin{subfigure}[t]{0.5\textwidth}
		\centering
		\includegraphics[width=0.99\textwidth]{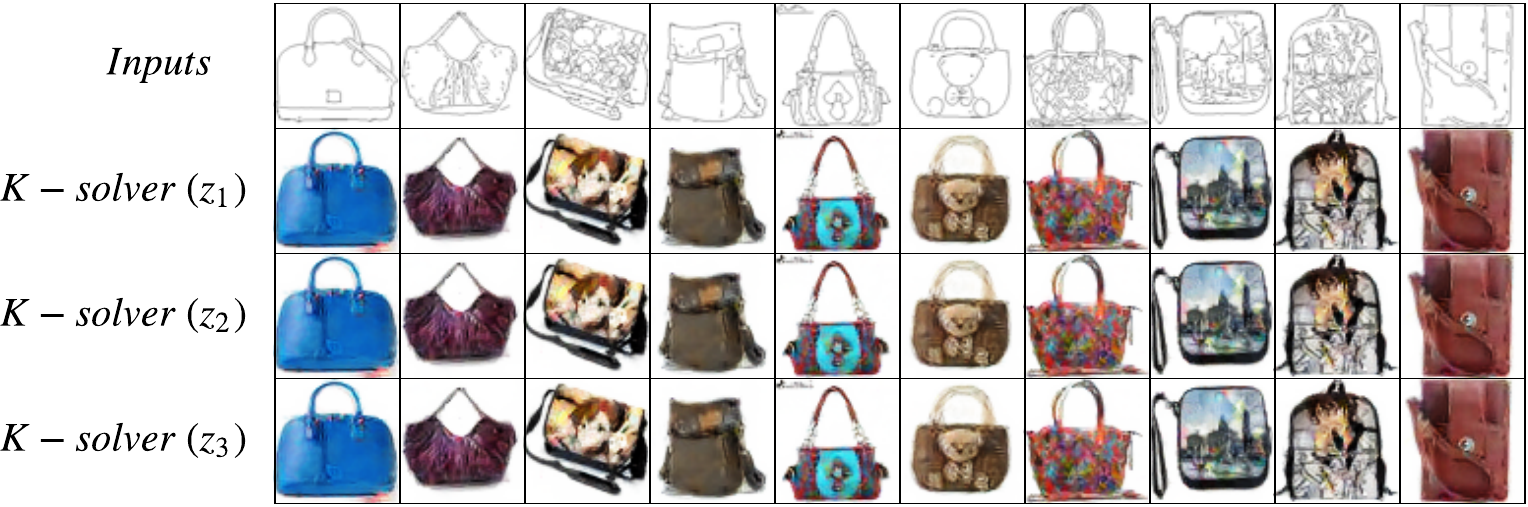}
		\caption{Edges-to-Handbags}
	\end{subfigure}
	\vspace{-5pt}
	\caption{Visual results on unsupervised image-to-image translation.}
	\label{image2image_translation_appendix}
%	\vspace{-35pt}
\end{figure*}

\begin{table}[!h]
	\centering
	\vspace{-20pt}
	\caption{Quantitative results on image-to-image translation.}
	\vspace{3pt}
	\begin{tabular}{c|c|c|c|c}
		\multicolumn{5}{c}{Handbags2shoes: h $\rightarrow$ s.  Edges2handbags: e $\rightarrow$ h.} \\
		\hline
		\multirow{2}{*}{Method} & \multicolumn{2}{c|}{~~~~~~~~~~~~~~KID~~~~~~~~~~~~~~} & \multicolumn{2}{c}{~mismatching degree~~} \\ \cline{2-5} 
		& ~~~h $\rightarrow$ s~~~ & e $\rightarrow$ h & ~~~h $\rightarrow$ s~~~    & e $\rightarrow$ h                 \\ \hline % \hline
		K-solver ($z_1$)               & ~~2.35$\pm$0.05   & ~~1.48$\pm$0.10  & ~~8.9               & 330.02             \\ \hline
		K-solver ($z_2$)               & ~~2.36$\pm$0.05   & ~~1.66$\pm$0.09  & ~~8.9               & 329.83             \\ \hline
		K-solver ($z_3$)              & ~~2.21$\pm$0.05   & ~~1.59$\pm$0.09  & ~~8.9               & 329.99             \\ \hline
	\end{tabular}
	\label{image_translation_table_appendix}
	\vspace{-15pt}
\end{table}

\subsubsection{Color Transfer}

Fig. \ref{color_transfer_diff_z} shows results of Kantorovich solver on color transfer with different noise $z$. 
According to the 3D color distributions, we can see that Kantorovich solver learns a stochastic mapping. According to the transferred images, different noise $z$ yield visually similar results.

\begin{figure}[h]
	\begin{subfigure}[b]{0.33\linewidth}
		\centering
		\includegraphics[width=0.8\linewidth]{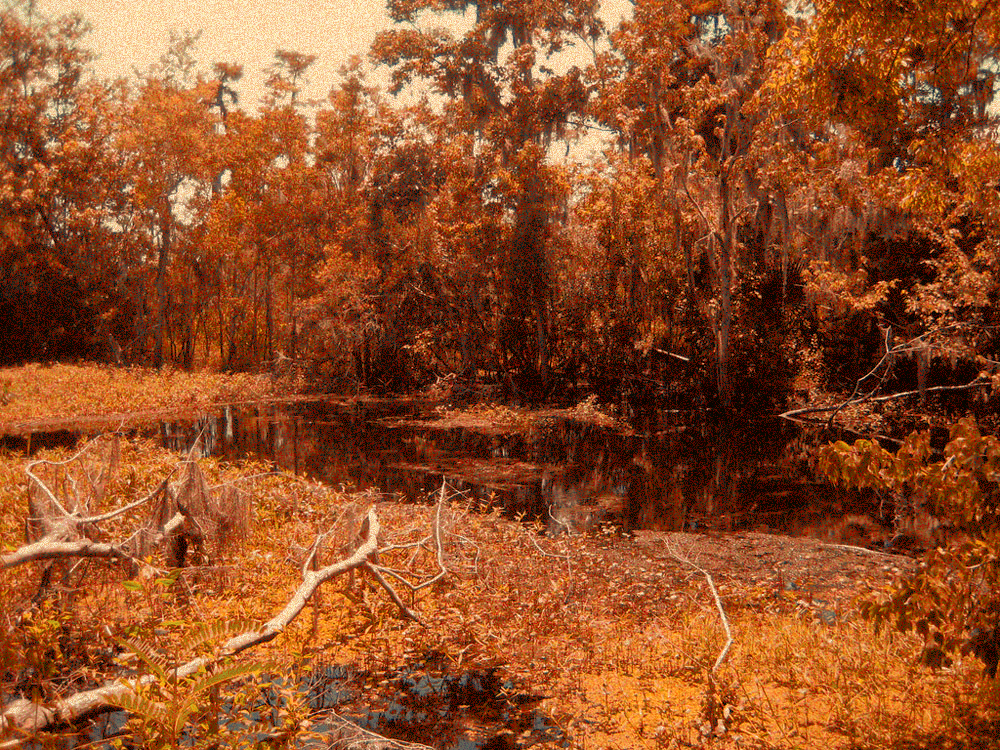} 
	\end{subfigure}%% 
	\begin{subfigure}[b]{0.33\linewidth}
		\centering
		\includegraphics[width=0.8\linewidth]{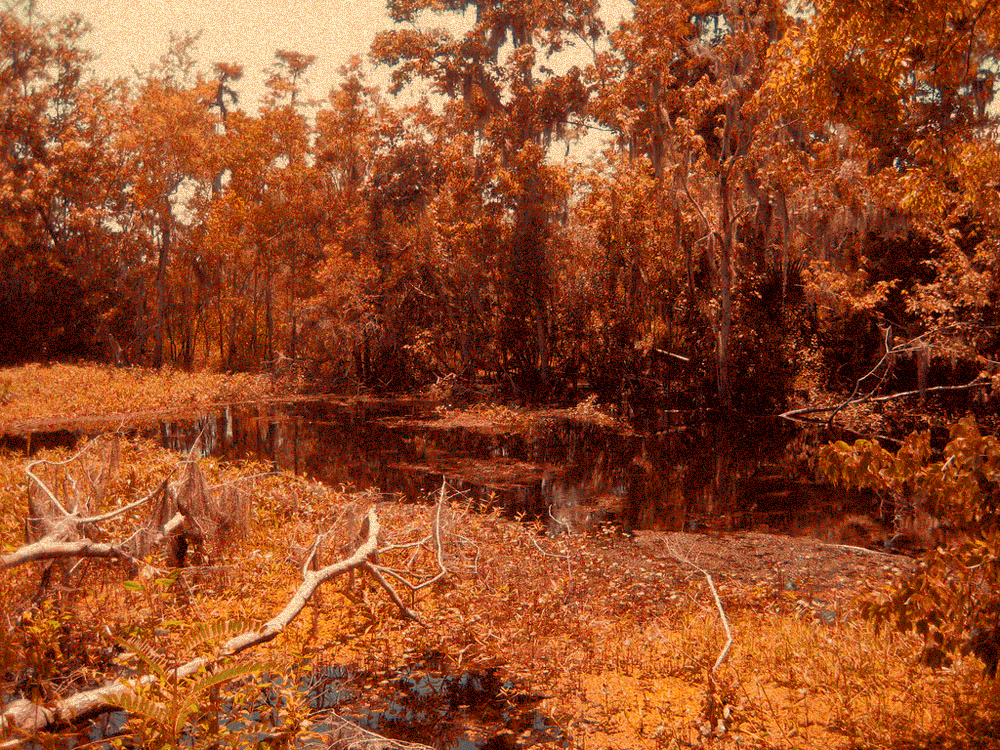} 
	\end{subfigure}%% 
	\begin{subfigure}[b]{0.33\linewidth}
		\centering
		\includegraphics[width=0.8\linewidth]{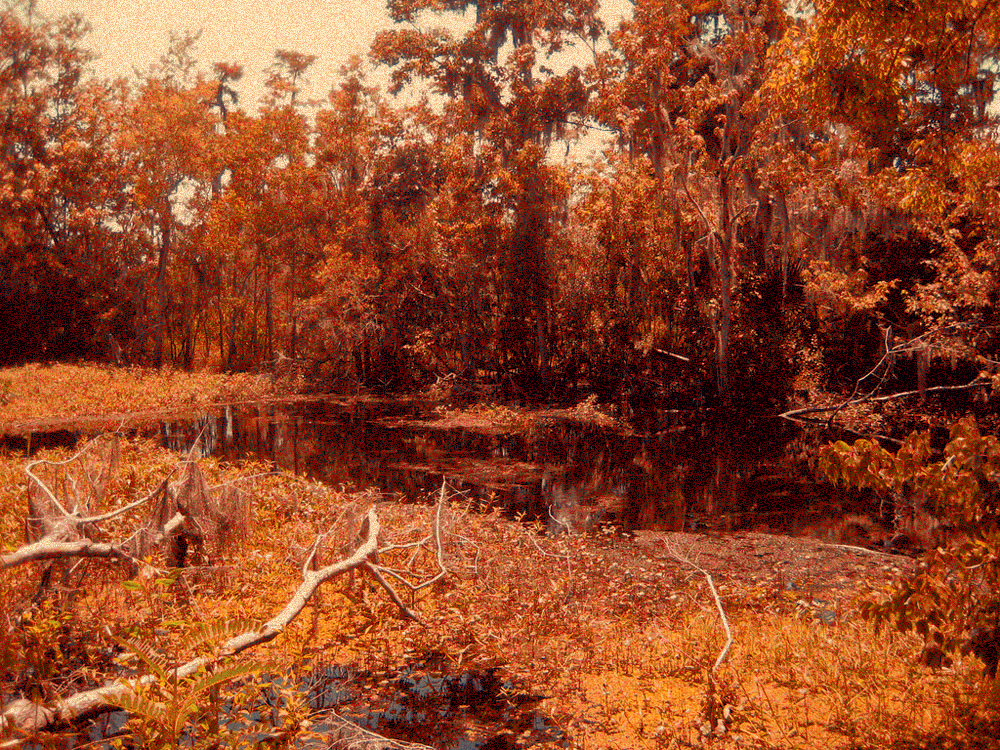} 
	\end{subfigure}%% 	
	\vspace{10pt}
	
	\begin{subfigure}[b]{0.33\linewidth}
		\centering
		\includegraphics[width=0.8\linewidth]{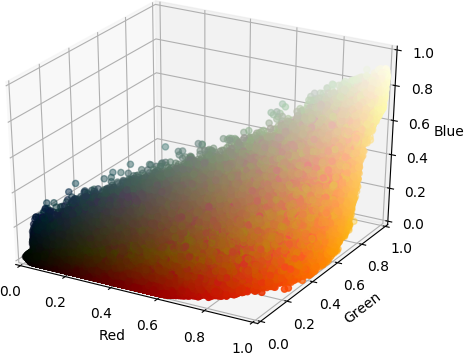} 
		\caption{Kantorovich ($z_1$)}
	\end{subfigure}%% 
	\begin{subfigure}[b]{0.33\linewidth}
		\centering
		\includegraphics[width=0.8\linewidth]{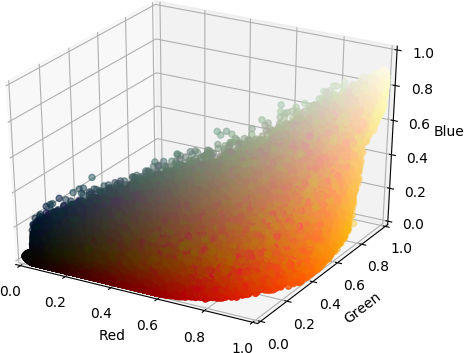} 
		\caption{Kantorovich ($z_2$)}
	\end{subfigure}%% 
	\begin{subfigure}[b]{0.33\linewidth}
		\centering
		\includegraphics[width=0.8\linewidth]{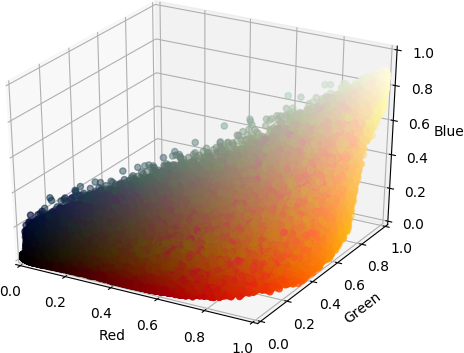} 
		\caption{Kantorovich ($z_3$)}
	\end{subfigure}%% 	
	\vspace{-3pt} 
	\caption{Transferred results and corresponding 3D color distributions with different noise $z$ for the color transfer example in the main body.}
	\label{color_transfer_diff_z}
	\vspace{-15pt}
\end{figure}

\clearpage
\subsection{More Results for Color Transfer}

\begin{figure}[h!]
	\begin{subfigure}[b]{0.5\linewidth}
		\begin{subfigure}[b]{0.5\linewidth}
			\centering
			\includegraphics[width=0.9\linewidth]{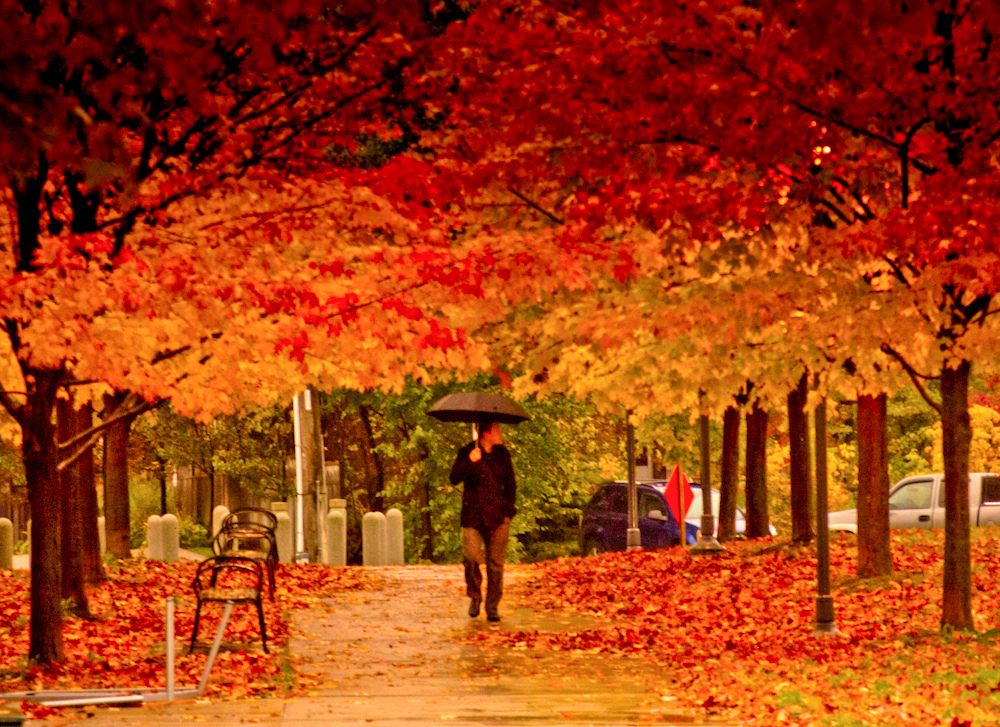} 
		\end{subfigure}%% 
		\begin{subfigure}[b]{0.5\linewidth}
			\centering
			\includegraphics[width=0.9\linewidth]{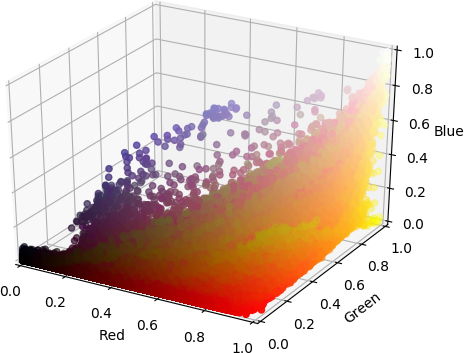} 
		\end{subfigure}%% 
		\caption{Source image} 
	\end{subfigure}
	\begin{subfigure}[b]{0.5\linewidth}
		\begin{subfigure}[b]{0.5\linewidth}
			\centering
			\includegraphics[width=0.9\linewidth]{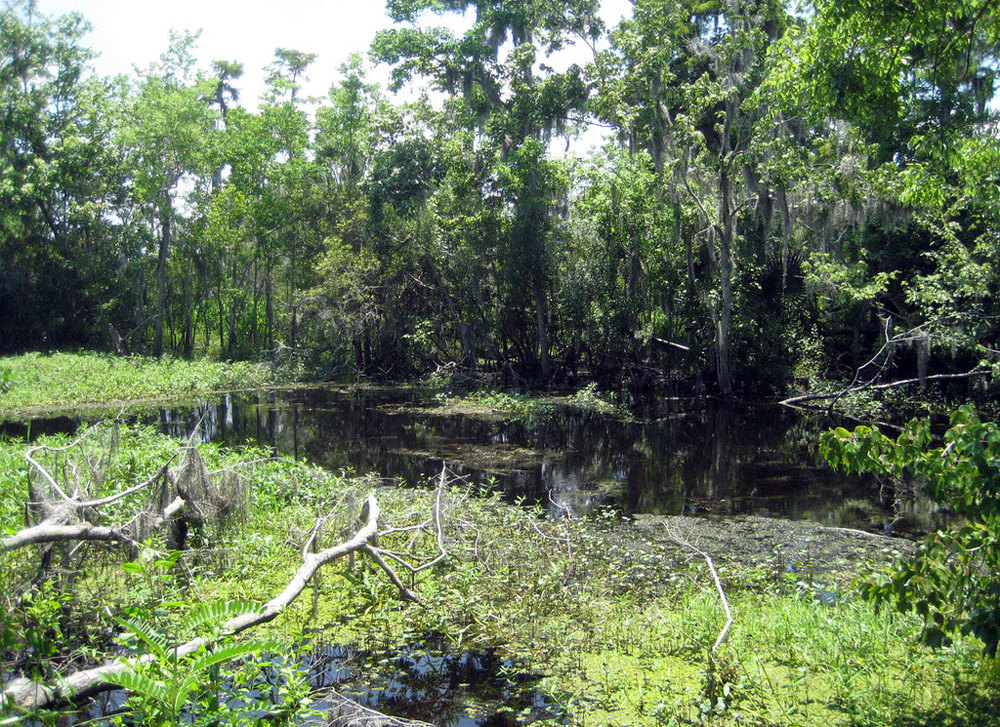} 
		\end{subfigure}%% 
		\begin{subfigure}[b]{0.5\linewidth}
			\centering
			\includegraphics[width=0.9\linewidth]{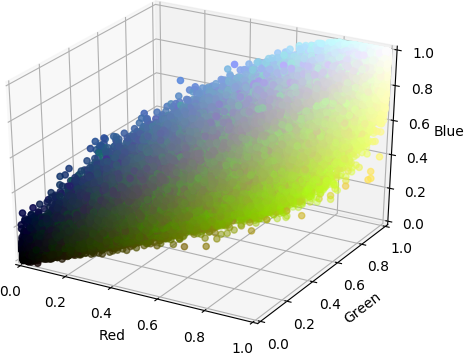} 
		\end{subfigure}%% 
		\caption{Target image} 
	\end{subfigure}
	
	\begin{subfigure}[b]{0.2\linewidth}
		\centering
		\includegraphics[width=0.99\linewidth]{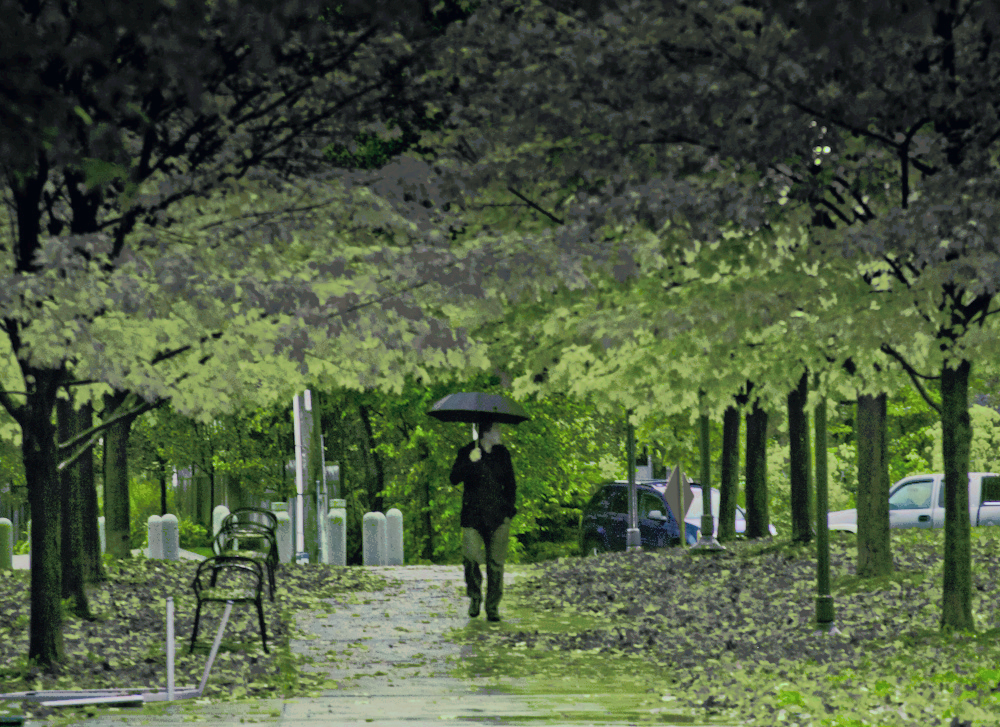} 
	\end{subfigure}%% 
	\begin{subfigure}[b]{0.2\linewidth}
		\centering
		\includegraphics[width=0.99\linewidth]{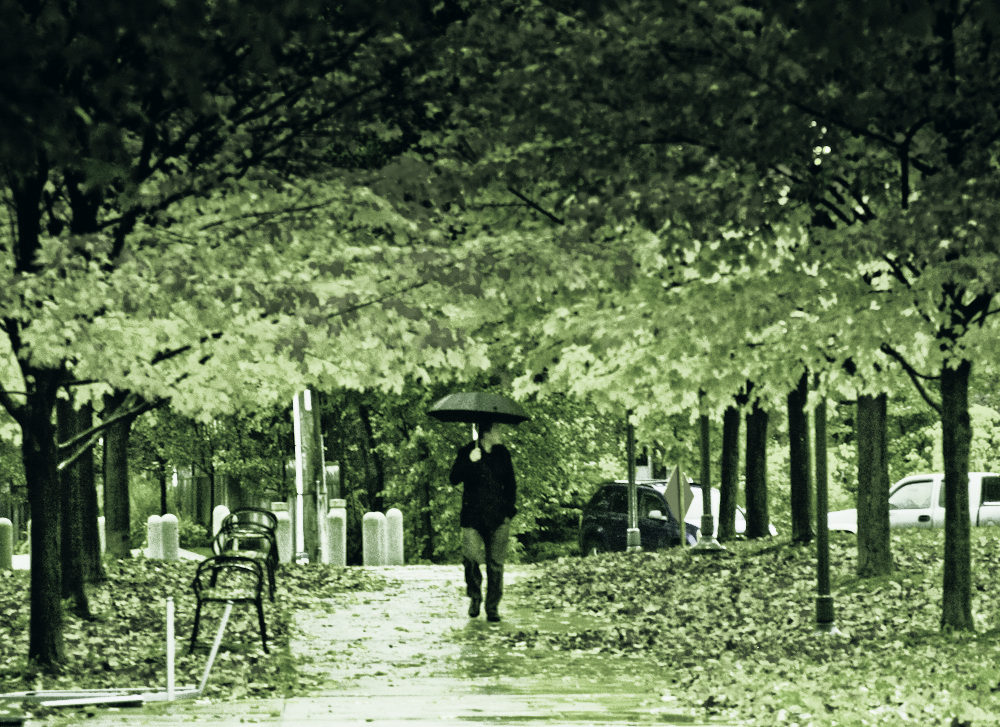} 
	\end{subfigure}%% 
	\begin{subfigure}[b]{0.2\linewidth}
		\centering
		\includegraphics[width=0.99\linewidth]{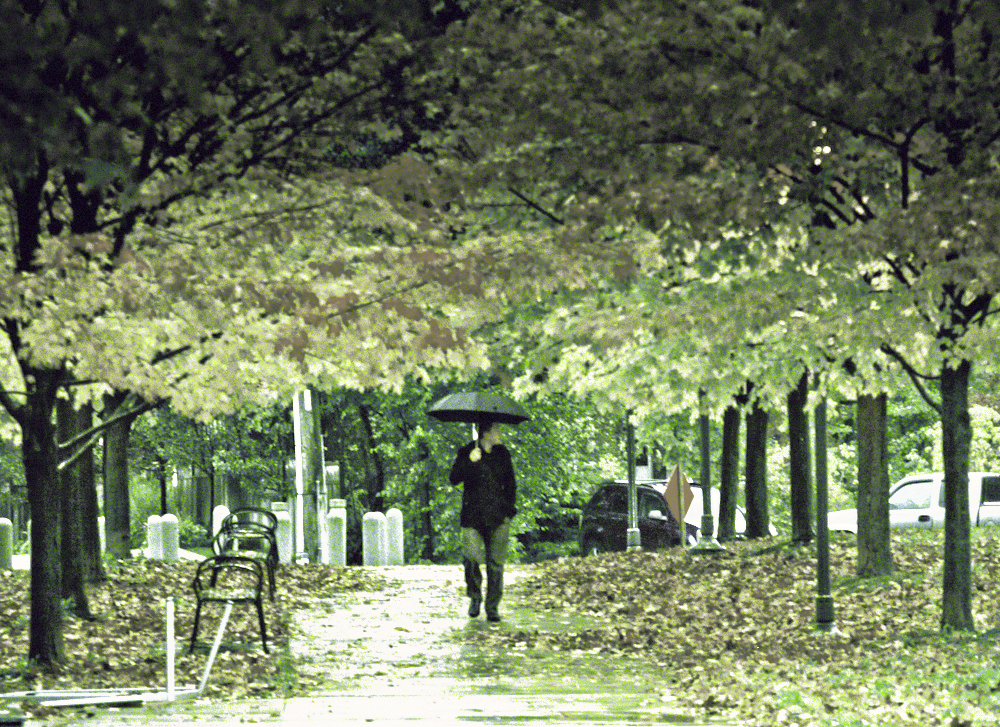} 
	\end{subfigure}%% 
	\begin{subfigure}[b]{0.2\linewidth}
		\centering
		\includegraphics[width=0.99\linewidth]{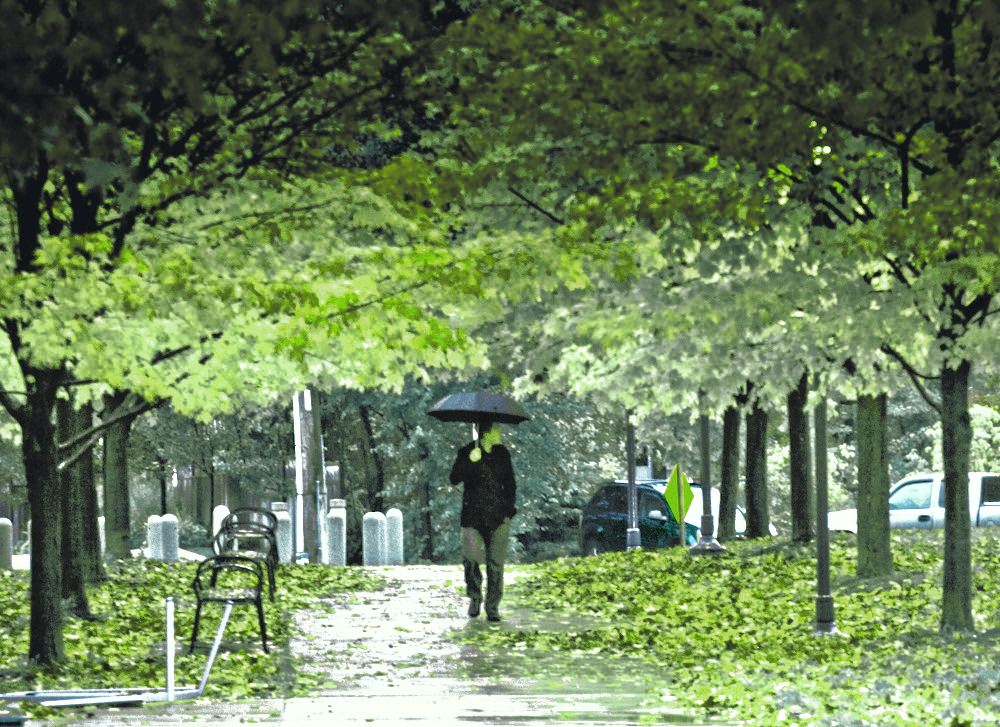} 
	\end{subfigure}%% 
	\begin{subfigure}[b]{0.2\linewidth}
		\centering
		\includegraphics[width=0.99\linewidth]{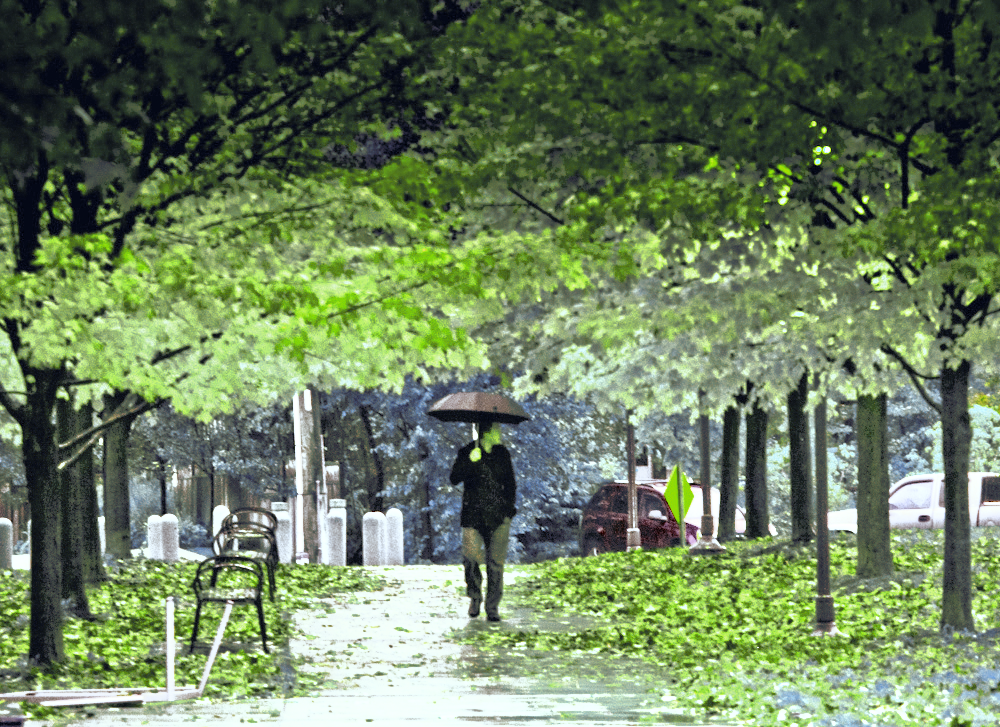} 
	\end{subfigure}%% 
	
	\begin{subfigure}[b]{0.2\linewidth}
		\centering
		\includegraphics[width=0.99\linewidth]{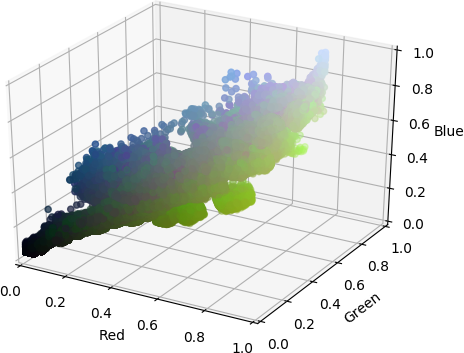} 
		\caption{ROT} 
	\end{subfigure}%% 
	\begin{subfigure}[b]{0.2\linewidth}
		\centering
		\includegraphics[width=0.99\linewidth]{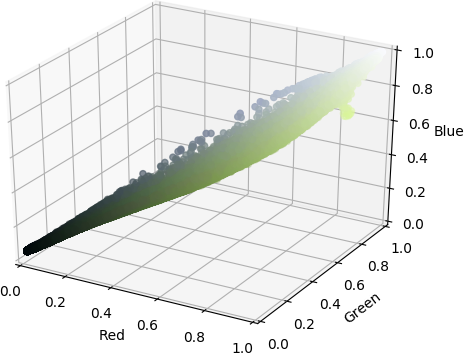} 
		\caption{BOT} 
	\end{subfigure}%% 
	\begin{subfigure}[b]{0.2\linewidth}
		\centering
		\includegraphics[width=0.99\linewidth]{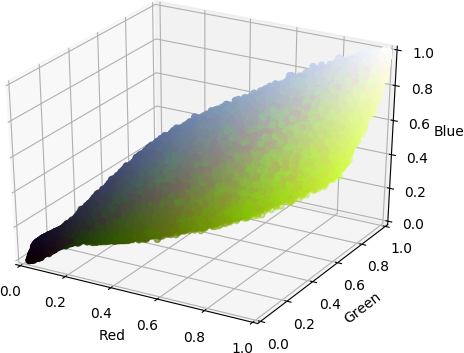} 
		\caption{Kantorovich} 
	\end{subfigure}%% 
	\begin{subfigure}[b]{0.2\linewidth}
		\centering
		\includegraphics[width=0.99\linewidth]{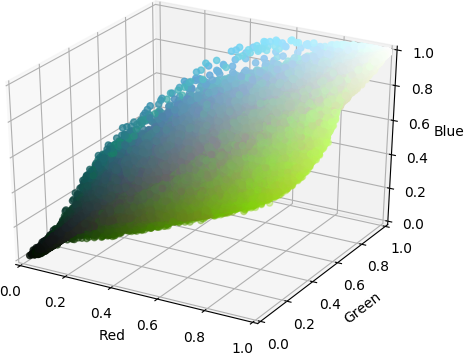} 
		\caption{Monge} 
	\end{subfigure}%% 
	\begin{subfigure}[b]{0.2\linewidth}
		\centering
		\includegraphics[width=0.99\linewidth]{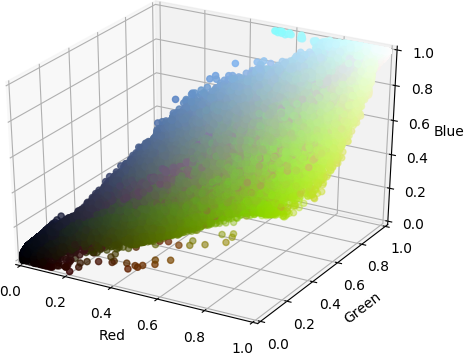} 
		\caption{Bijection} 
	\end{subfigure}%% 
	
	\begin{subfigure}[b]{0.33\linewidth}
		\centering
		\includegraphics[width=0.7\linewidth]{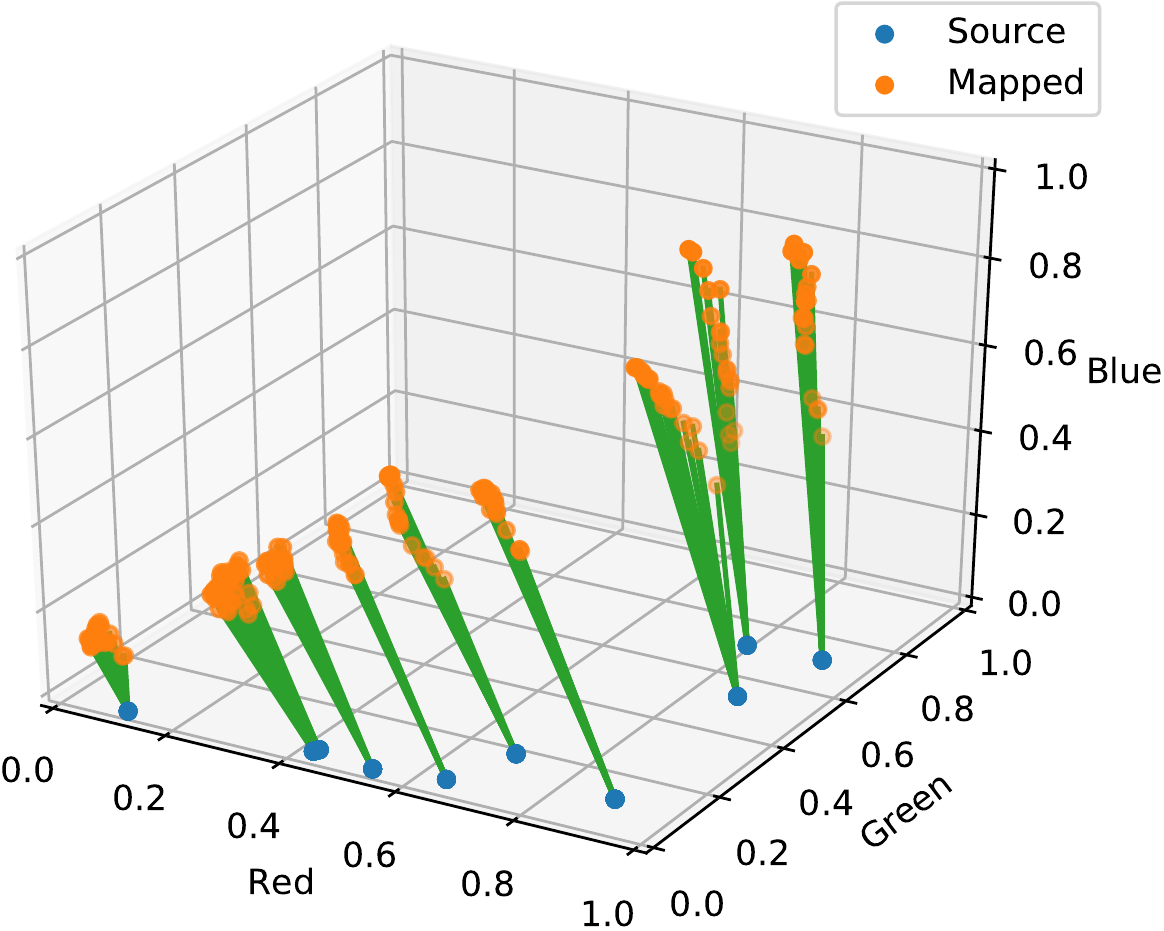} 
		\caption{Kantorovich solver} 
	\end{subfigure}%% 
	\begin{subfigure}[b]{0.33\linewidth}
		\centering
		\includegraphics[width=0.7\linewidth]{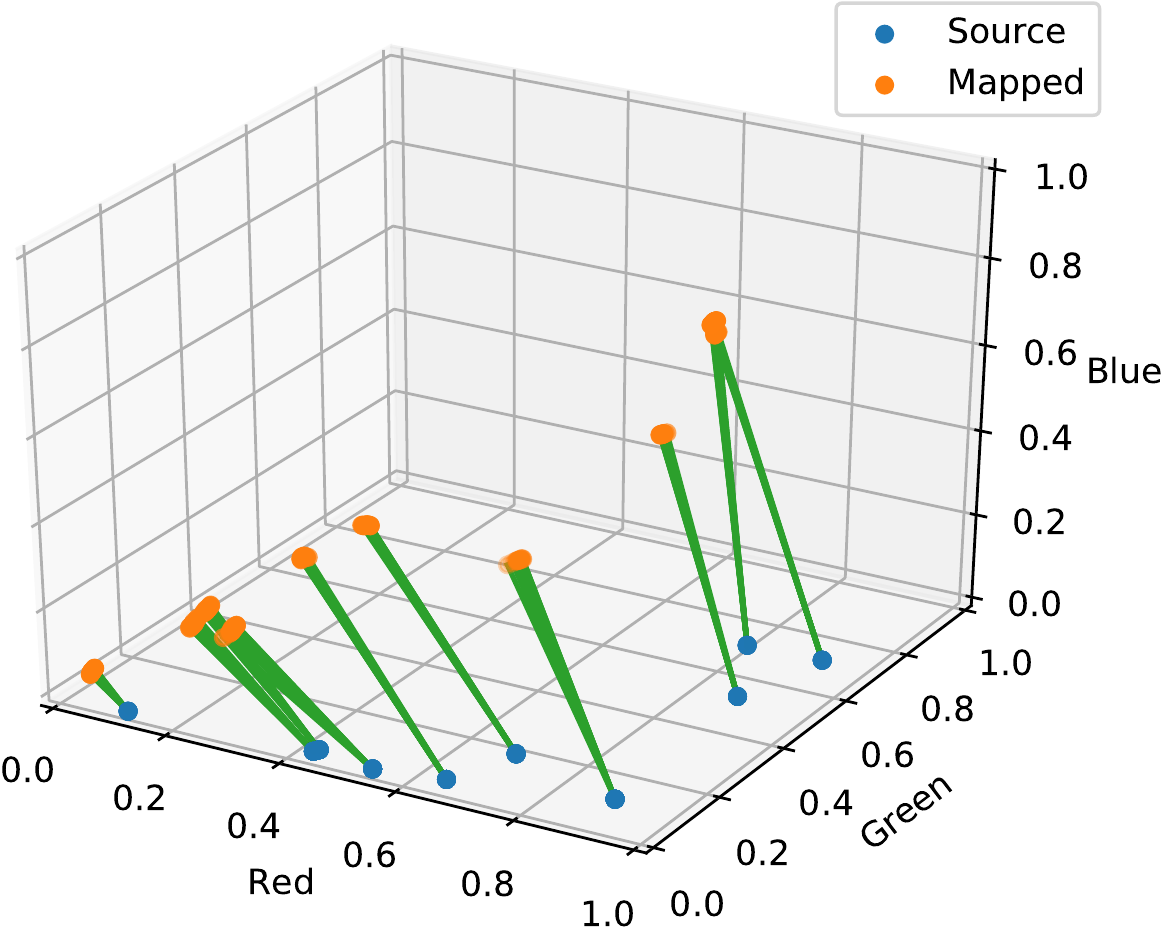}
		\caption{Monge solver}
	\end{subfigure}
	\begin{subfigure}[b]{0.33\linewidth}
		\centering
		\includegraphics[width=0.7\linewidth]{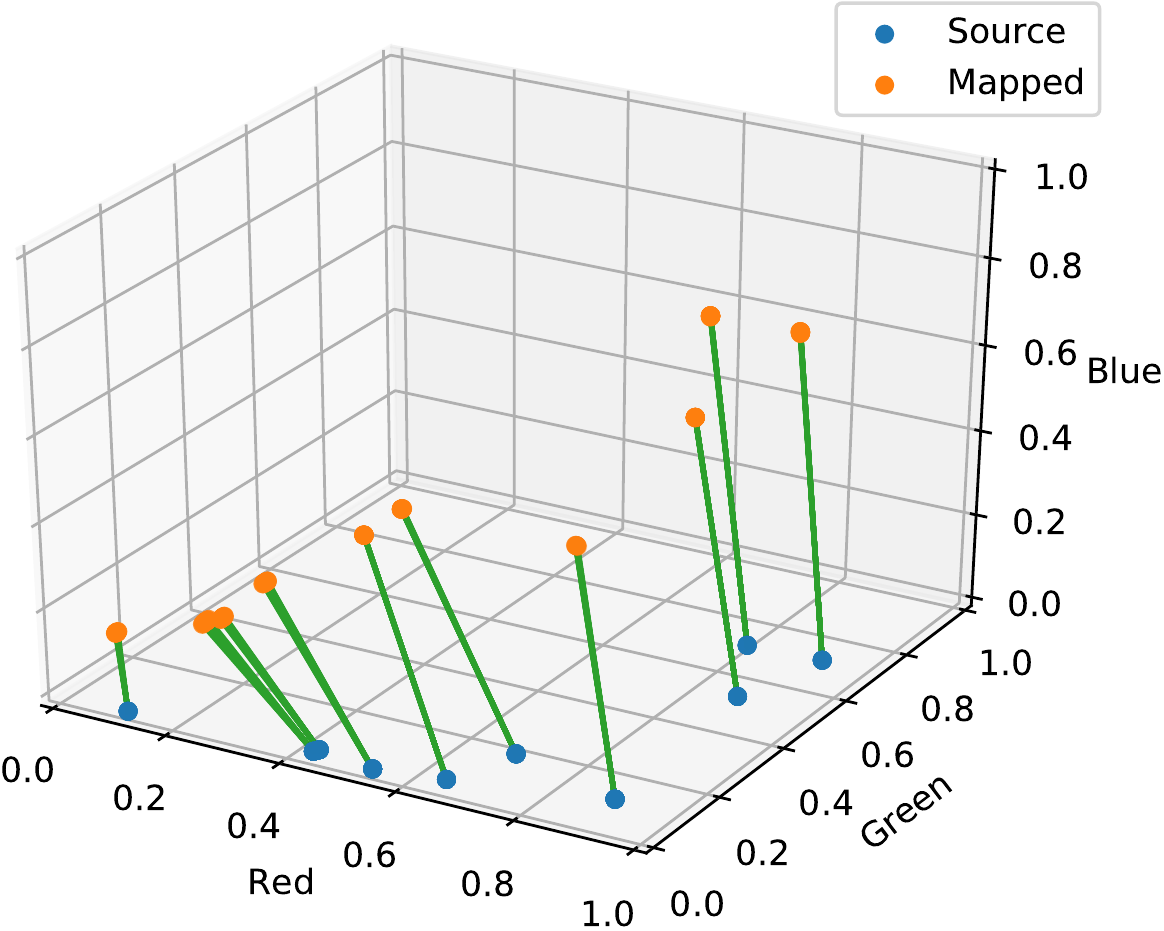}
		\caption{Bijection solver}
	\end{subfigure}
	\caption{(a) Source image (b) Target image (c-g) Transfer results of ROT, BOT, Kantorovich solver, Monge solver and Bijection solver, respectively (h-j) Mapping learned by Kantorovich solver, Monge solver and Bijection solver, respectively}
\end{figure}

\clearpage
\begin{figure}[!h]
	\begin{subfigure}[b]{0.5\linewidth}
		\begin{subfigure}[b]{0.5\linewidth}
			\centering
			\includegraphics[width=0.9\linewidth]{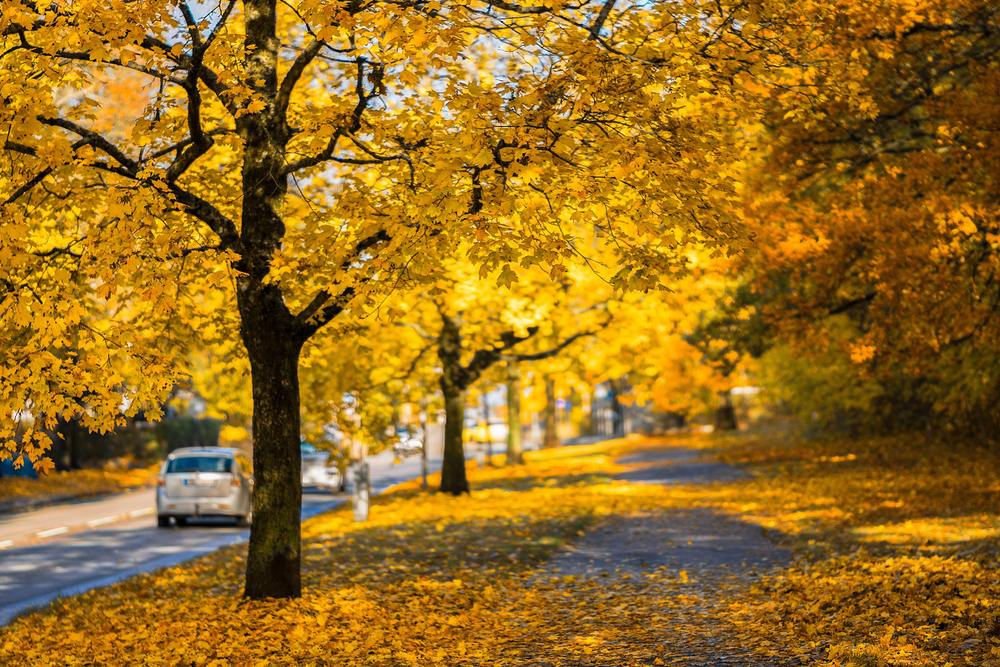} 
		\end{subfigure}%% 
		\begin{subfigure}[b]{0.5\linewidth}
			\centering
			\includegraphics[width=0.9\linewidth]{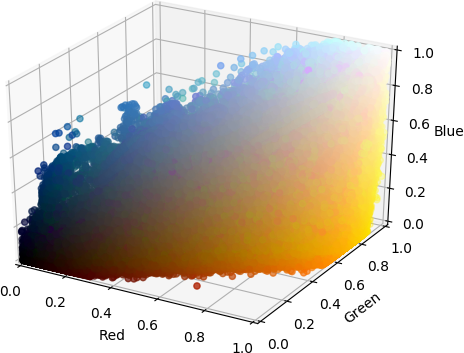} 
		\end{subfigure}%% 
		\caption{Source image} 
	\end{subfigure}
	\begin{subfigure}[b]{0.5\linewidth}
		\begin{subfigure}[b]{0.5\linewidth}
			\centering
			\includegraphics[width=0.9\linewidth]{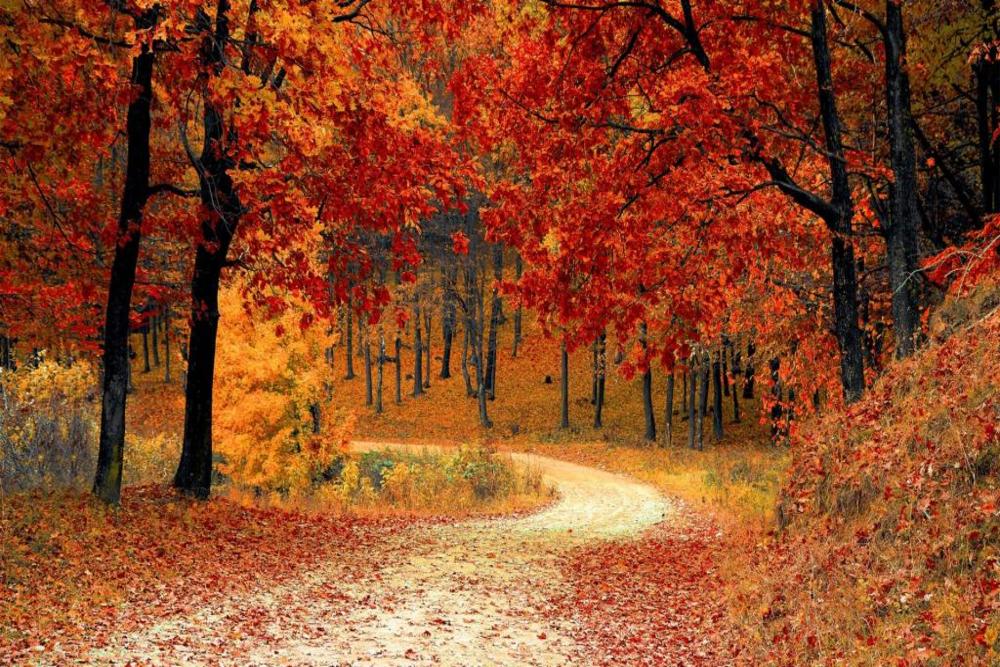} 
		\end{subfigure}%% 
		\begin{subfigure}[b]{0.5\linewidth}
			\centering
			\includegraphics[width=0.9\linewidth]{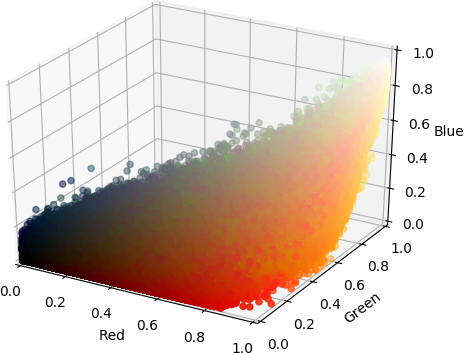} 
		\end{subfigure}%% 
		\caption{Target image} 
	\end{subfigure}
	\begin{subfigure}[b]{0.2\linewidth}
		\centering
		\includegraphics[width=0.99\linewidth]{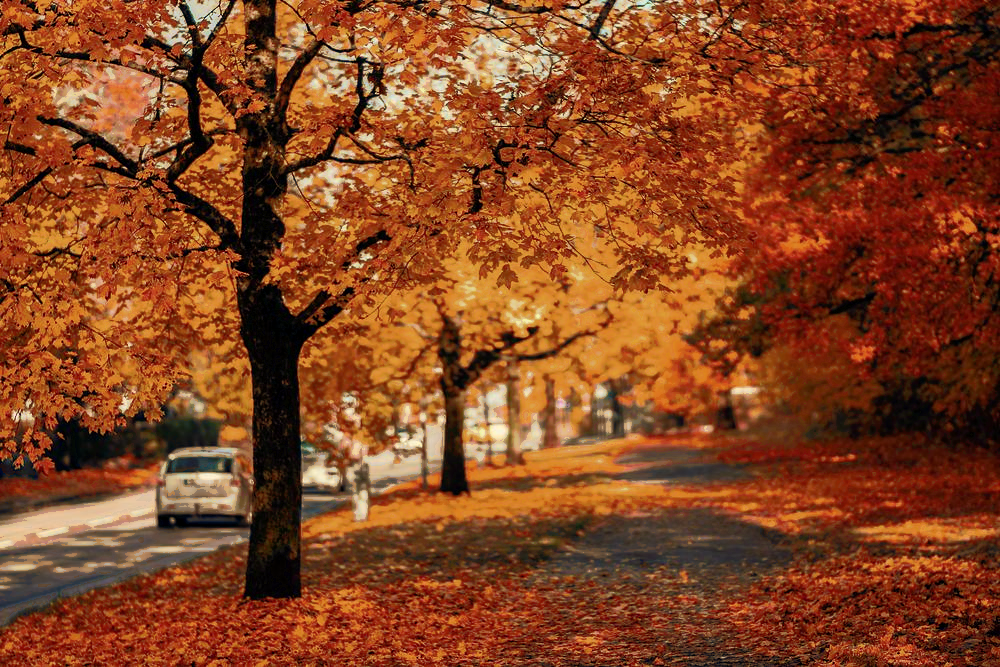} 
	\end{subfigure}%% 
	\begin{subfigure}[b]{0.2\linewidth}
		\centering
		\includegraphics[width=0.99\linewidth]{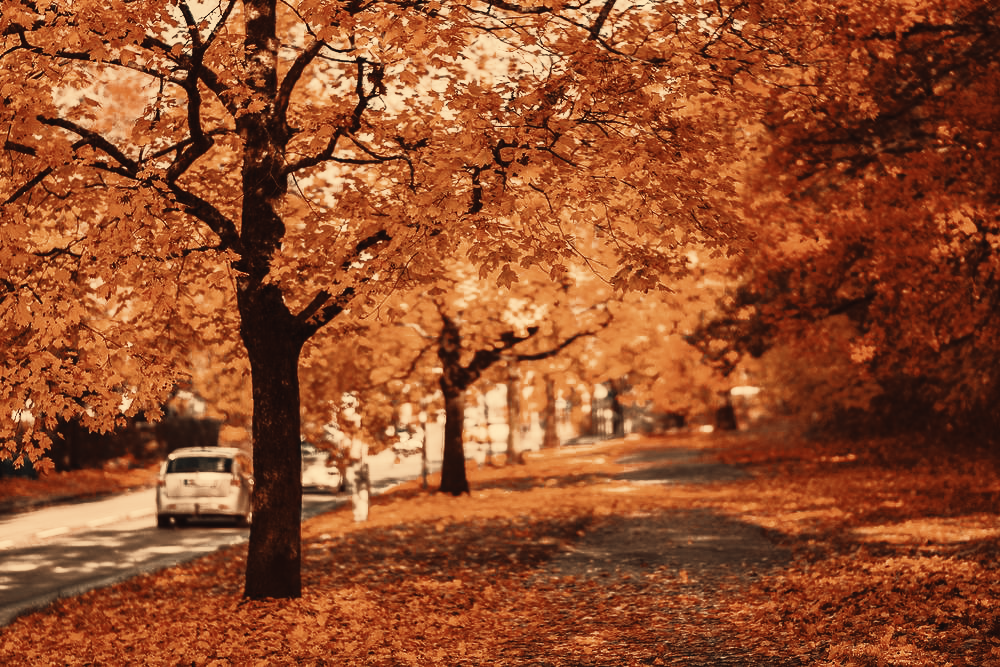} 
	\end{subfigure}%% 
	\begin{subfigure}[b]{0.2\linewidth}
		\centering
		\includegraphics[width=0.99\linewidth]{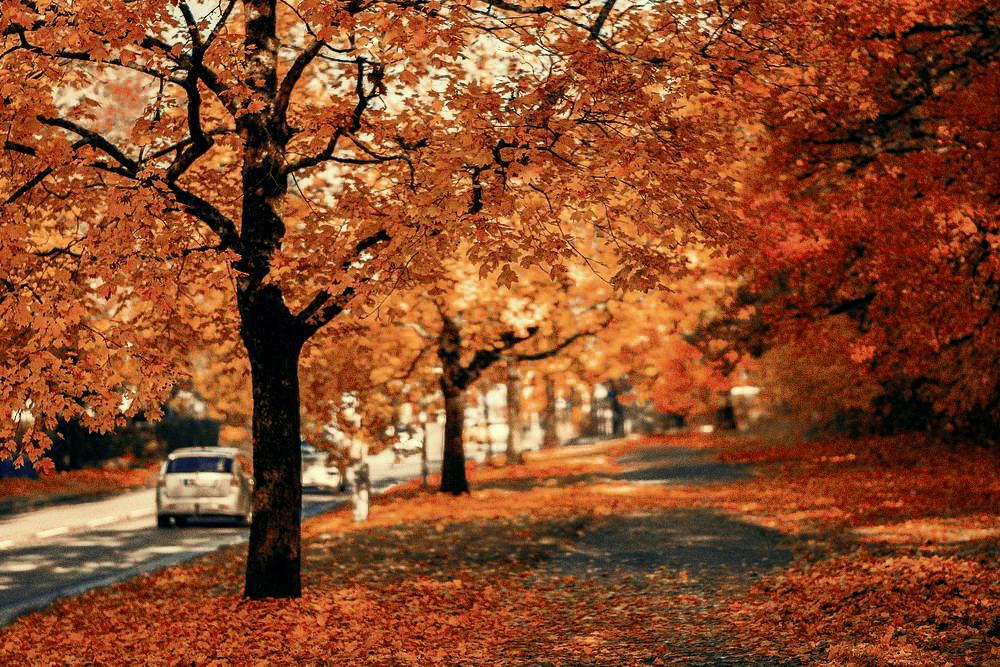} 
	\end{subfigure}%% 
	\begin{subfigure}[b]{0.2\linewidth}
		\centering
		\includegraphics[width=0.99\linewidth]{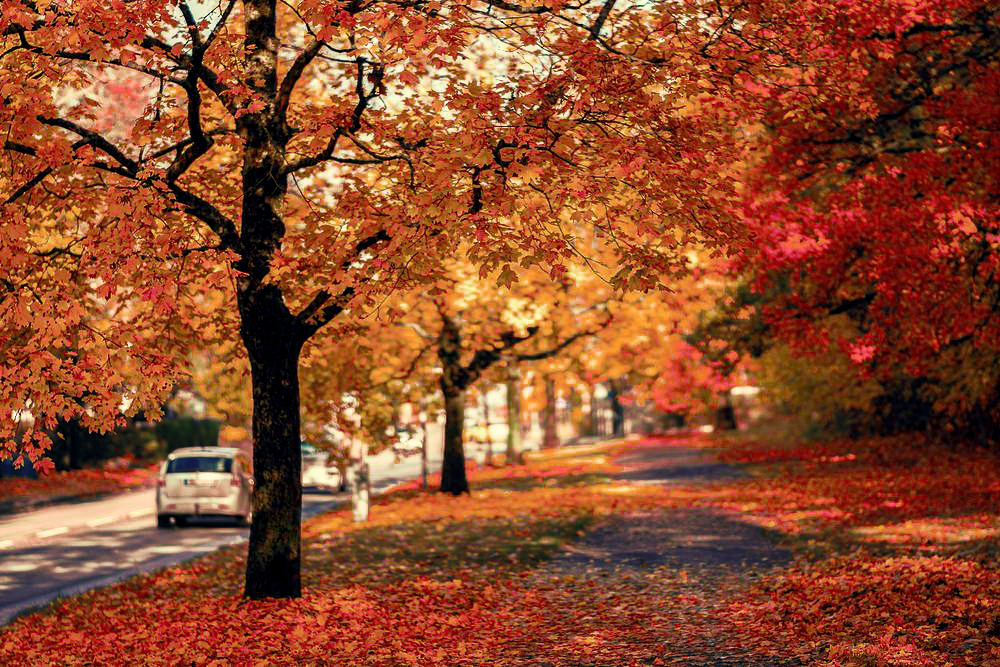} 
	\end{subfigure}%% 
	\begin{subfigure}[b]{0.2\linewidth}
		\centering
		\includegraphics[width=0.99\linewidth]{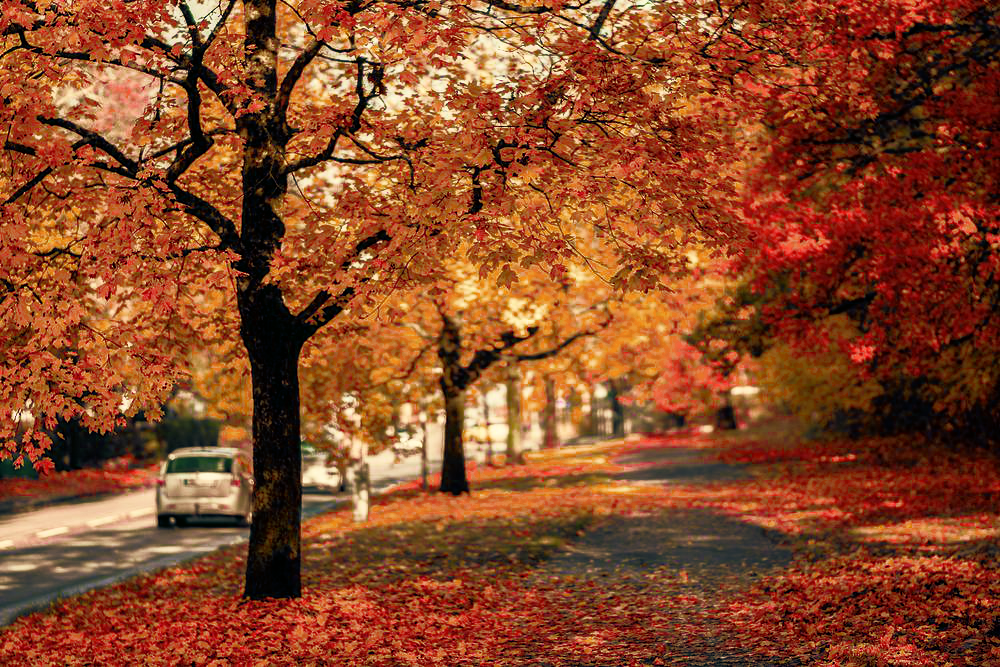} 
	\end{subfigure}%% 
	
	\begin{subfigure}[b]{0.2\linewidth}
		\centering
		\includegraphics[width=0.99\linewidth]{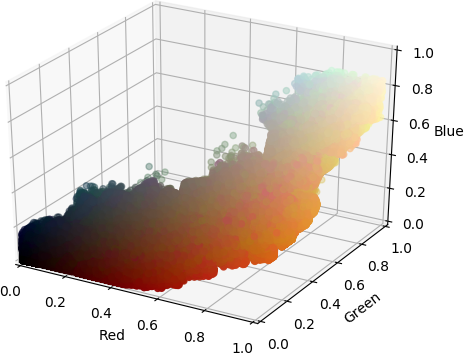} 
		\caption{ROT} 
	\end{subfigure}%% 
	\begin{subfigure}[b]{0.2\linewidth}
		\centering
		\includegraphics[width=0.99\linewidth]{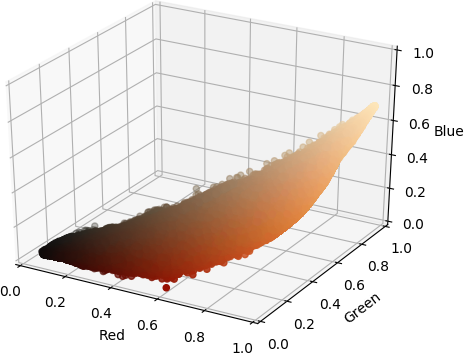} 
		\caption{BOT} 
	\end{subfigure}%% 
	\begin{subfigure}[b]{0.2\linewidth}
		\centering
		\includegraphics[width=0.99\linewidth]{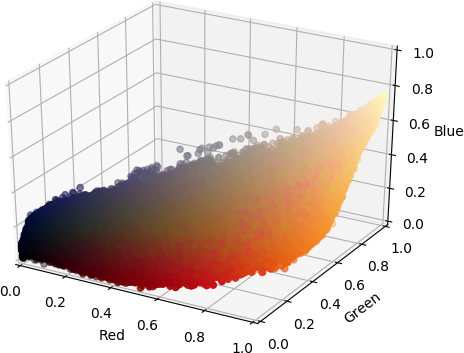} 
		\caption{Kantorovich} 
	\end{subfigure}%% 
	\begin{subfigure}[b]{0.2\linewidth}
		\centering
		\includegraphics[width=0.99\linewidth]{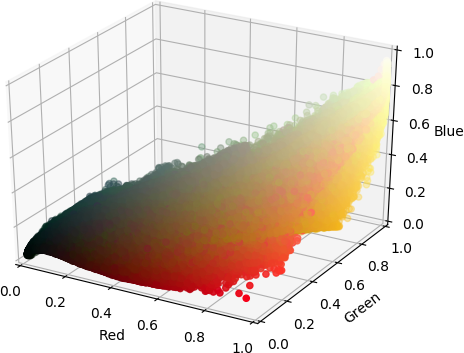} 
		\caption{Monge} 
	\end{subfigure}%% 
	\begin{subfigure}[b]{0.2\linewidth}
		\centering
		\includegraphics[width=0.99\linewidth]{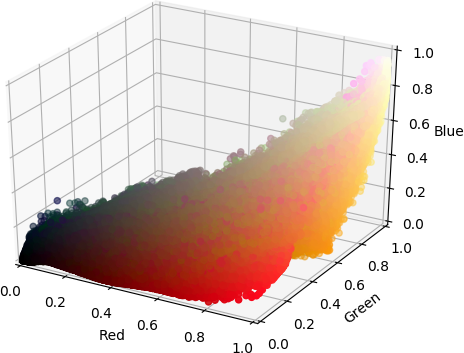} 
		\caption{Bijection} 
	\end{subfigure}%% 
	
	\begin{subfigure}[b]{0.33\linewidth}
		\centering
		\includegraphics[width=0.7\linewidth]{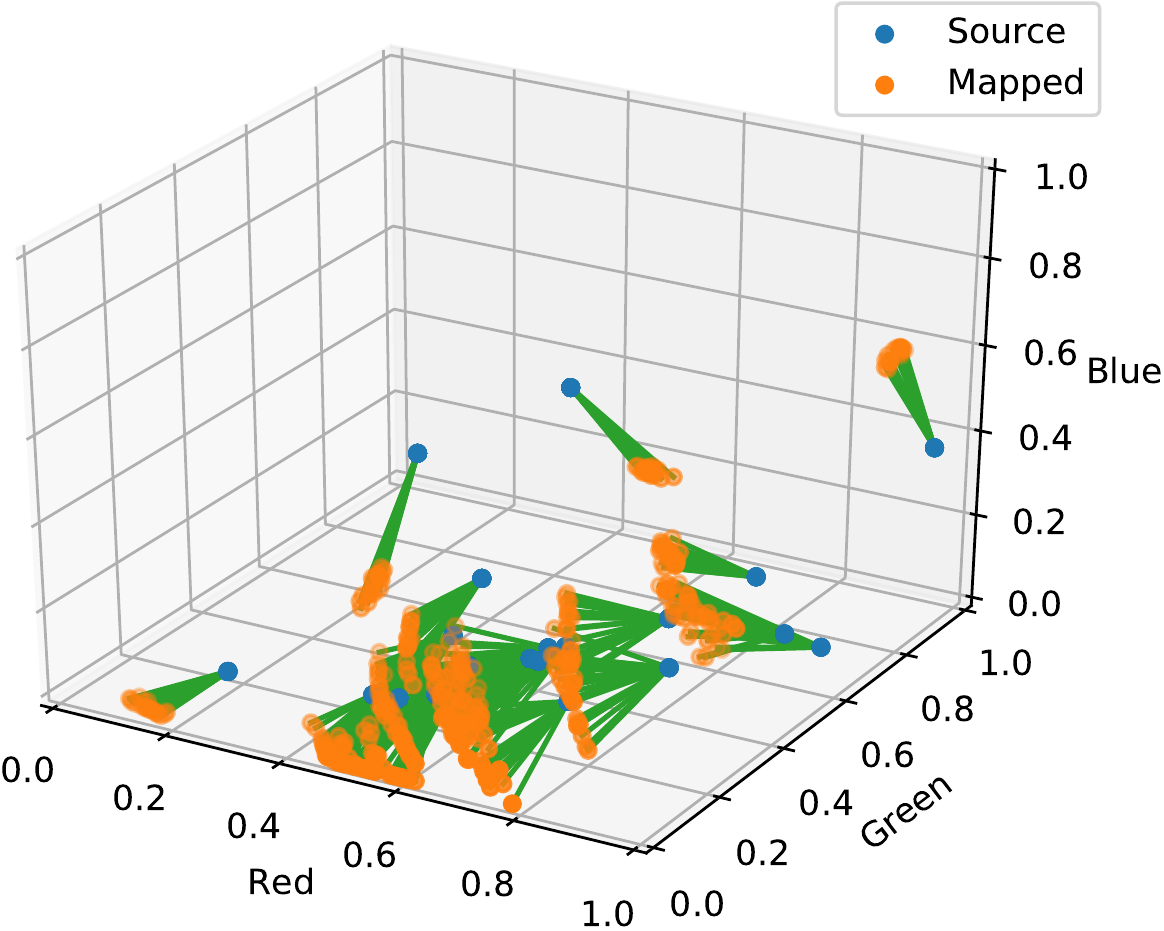} 
		\caption{Kantorovich solver} 
	\end{subfigure}%% 
	\begin{subfigure}[b]{0.33\linewidth}
		\centering
		\includegraphics[width=0.7\linewidth]{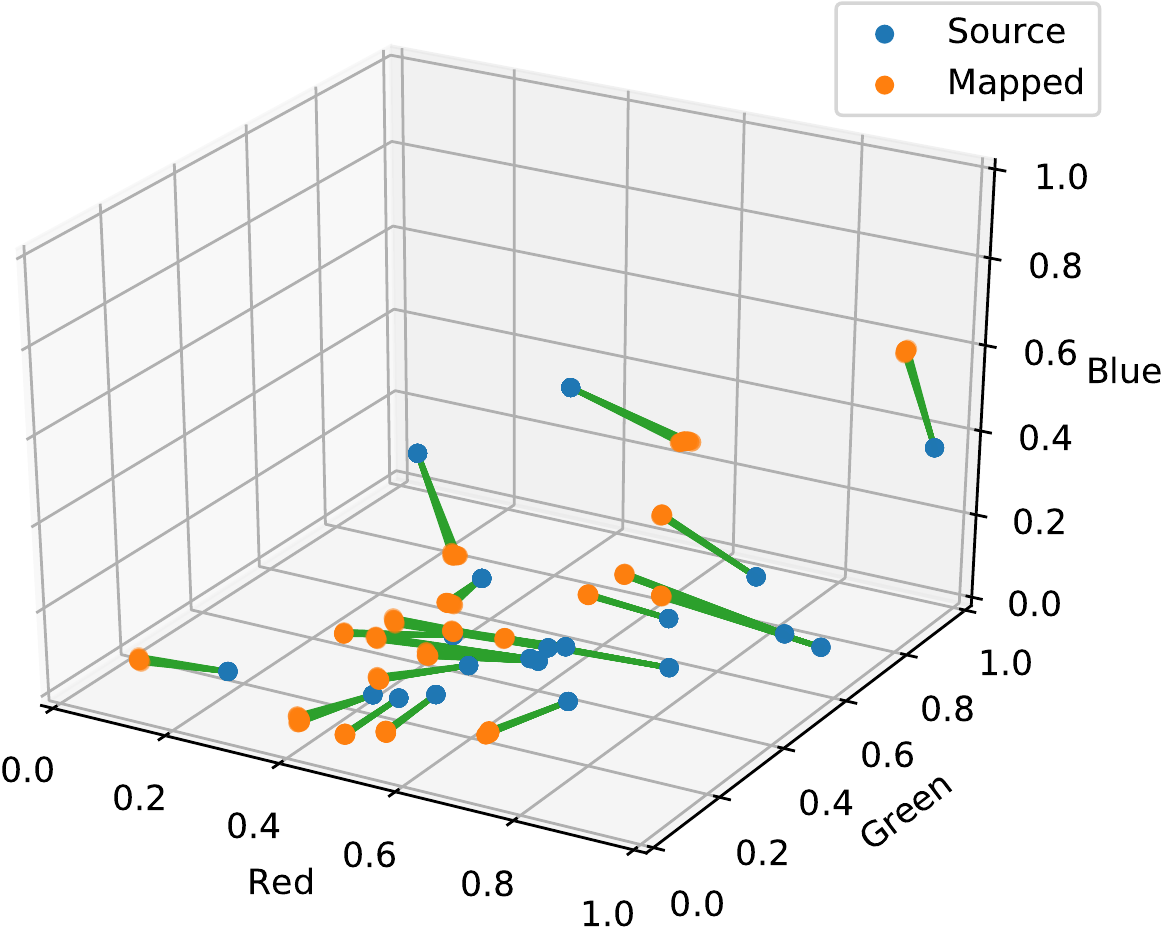}
		\caption{Monge solver}
	\end{subfigure}
	\begin{subfigure}[b]{0.33\linewidth}
		\centering
		\includegraphics[width=0.7\linewidth]{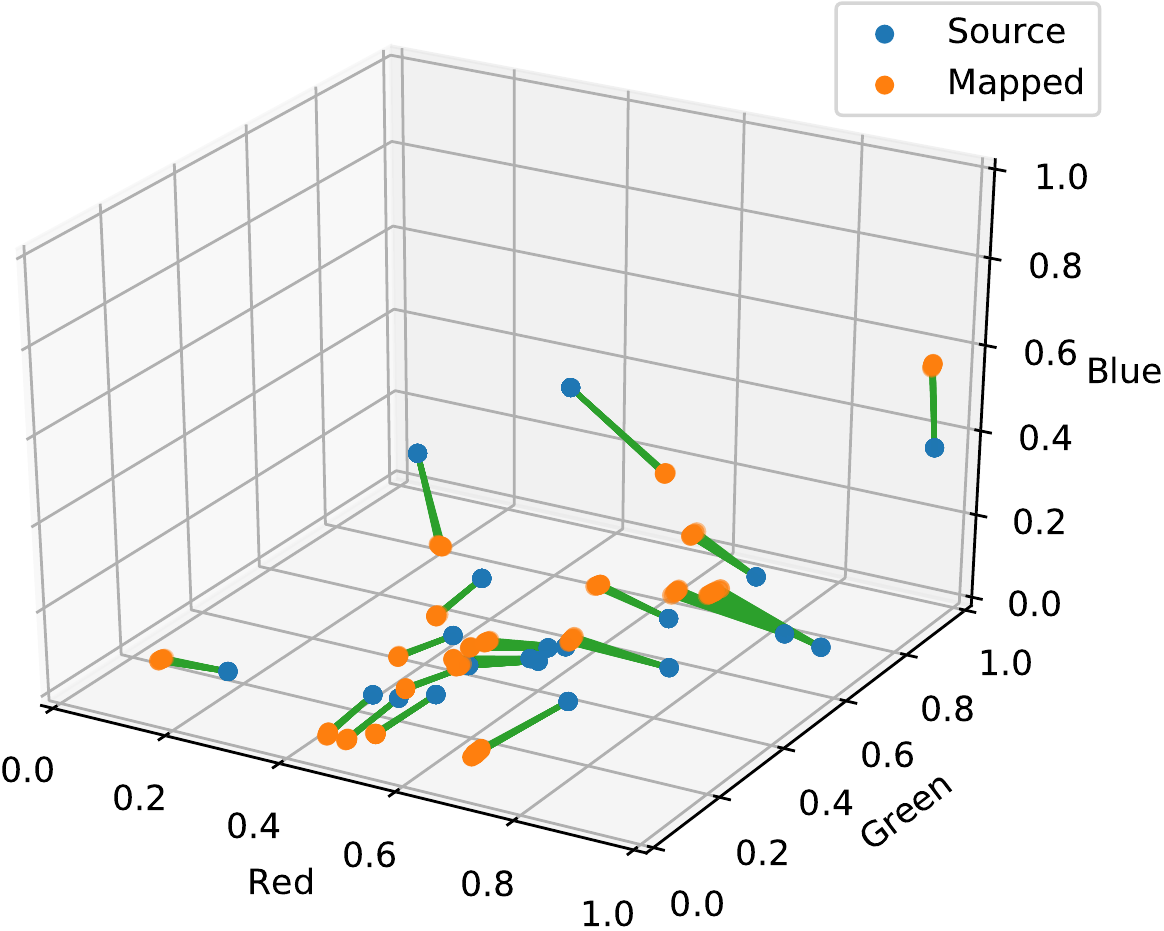}
		\caption{Bijection solver}
	\end{subfigure}
	\caption{(a) Source image (b) Target image (c-g) Transfer results of ROT, BOT, Kantorovich solver, Monge solver and Bijection solver, respectively (h-j) Mapping learned by Kantorovich solver, Monge solver and Bijection solver, respectively}
	\vspace{-30pt}
\end{figure}

\section{Experiment Methods}

\subsection{Domain Adaptation}
The data from the source domain are denoted as $\left<x_i, l_i\right>$, and the data from the target domain are denoted as $y_j$ without labels. 
A classifier $C_{x}$ in the source domain is pre-trained with $\{\left<x_i, l_i\right>\}_{i=1}^{n_x}$ and fixed hereafter.

To adapt the learned classifier to the target domain, we would like to learn an optimal mapping 
that maps the samples from the source domain to the samples from the target domain. Follow the common choice \cite{cycada,spot}, we define the cost function to be the difference between the label of the source sample and the label prediction of the translated target sample by $C_{x}$ and use the cross-entropy $\mathcal{H}$ to measure the difference: 
\begin{equation}
c(x, y) = \mathcal{H}(C_x(y), l), 
\end{equation}
where $y$ is the mapped sample of $x$, $l$ denotes the class label of $x$. 

After the mapping is learned stably, a classifier $C_y$ in the target domain is introduced and the cost function is changed to the one as follow:
\begin{equation}
c(x, y) = \mathcal{H}(C_y(y), l),
\end{equation}
where $C_{y}$ is trained together with $G_{xy}$. The whole training procedure is provided in Appendix \ref{algorithm_da}. 
We train the model with the training set of both datasets (the labels in the target dataset are omitted) and test on the target test set. 

We perform domain adaptation between four digit image datasets: MNIST \cite{mnist}, USPS \cite{usps}, SVHN \cite{svhn}, MNISTM \cite{mnistm}, each consists of images of digits 0-9.
Both MNIST and MNISTM consist of 60000 training images and 10000 test images which are of the size 28$\times$28.
USPS consists of 7291 training images and 2007 test images, which are of the size 16$\times$16. 
SVHN consists of 73257 training images and 26032 test images, which are of the size 32$\times$32. Images in MNIST and USPS are 1-channel, while images in SVHN and MNISTM are 3-channel.

\subsection{Image-to-Image Translation}
In the experiments of edges2handbags, we design the cost function as the $L_2$ norm between feature maps extracted through different convolution kernels for edge detection.
Specifically, we adopt the following two kernels:
\begin{equation*}
K_1 = \begin{bmatrix} 
-1 & 0 & 1 \\
-2 & 0 & 2 \\
-1 & 0 & 1	 
\end{bmatrix},
K_2 = \begin{bmatrix} 
1 & 2 & 1 \\
0 & 0 & 0 \\
-1 & -2 & -1	 
\end{bmatrix},
\end{equation*}
for detecting edges in two different directions.
Additionally, the cost function can be formulated as follow:
\begin{equation*}
c(x, y) = \sum_{k=1}^2\sum_{c=1}^3|||K_k * x_c| - |K_k * y_c|||_2,
\end{equation*}
where $*$ denotes the convolution operator, $|X|$ denotes a matrix with $[|X|]_{ij} = |X_{ij}|$.

In the experiments of handbags2shoes, we adopt the cost function as the mean squared distance between the average color vectors. This cost function can be formulated as follow:
\begin{equation*}
c(x, y) = \frac{1}{3} \sum_{c=1}^3 (Avg(x_c)-Avg(y_c))^2,
\end{equation*}
where $Avg(X)$ denotes the average of all elements of matrix $X$.

\subsection{Color Transfer}

For color transfer, the source and target distributions are the 3D color histograms of the source image $X$ and target image $Y$. We solve the OT problem between these two distributions. After the optimal mapping is learned, we apply the optimal mapping on each pixel of the source image $X$ and thereby obtain the transferred image. In this way, we transfer the color style of $Y$ to $X$, or in other words, impose the color histogram of $Y$ on $X$, and thus achieve the task of color transfer. We adopt the squared Euclidean distance $c(x, y) = \| x - y \| ^ 2$ as the cost function. 

\section{Network Architecture \& Hyperparameters}
\label{hyperparameters}

\subsection{Toy Experiments}
\label{hyperparameters_toy}
In toy experiments, generators and critics are all parameterized by multi-layer fully connected neural networks. Table \ref{toy_generator_generator} and Table \ref{toy_generator_critic} show the network architectures for generators and critics.

We set $\lambda_{gan_{xy}}=\lambda_{gan_{yx}}=1$, $\lambda_{gp_{xy}}=\lambda_{gp_{yx}}=0.1$, $\lambda_{cycle}=1$, $n_{critic}=5$, $lr=0.0001$ and the batch size is 100. The independent noise $z$ is sampled from the 2-dimension uniform distribution $\mathcal{U}[-1, 1]$. 

\begin{table}
	\vspace{-20pt}
	\centering
	\caption{The network architecture of generators for toy experiments and color transfer. For toy experiments, $d_{in}=2$; for color transfer, $d_{in}=3$}
	\vspace{3pt}
	\begin{tabular}{lll}
		\toprule[2pt]
		Input: & $x \in \mathbb{R}^{d_{in}}, z \in \mathbb{R}^{d_{in}}$ & \\
		\hline
		Linear: & [$2d_{in}$, 1024] & LeakyReLU \\
		\hline 
		Linear: & [1024, 1024] & LeakyReLU  \\
		\hline
		Linear: & [1024, $d_{in}$] \\
		\bottomrule[2pt]
	\end{tabular}
	\label{toy_generator_generator}
\end{table}

\begin{table}
	\vspace{-40pt}
	\centering
	\caption{The network architecture of critics for toy experiments and color transfer. For toy experiments, $d_{in}=2$; for color transfer, $d_{in}=3$}
	\vspace{3pt}
	\begin{tabular}{lll}
		\toprule[2pt]
		Input: & $x \in \mathbb{R}^{d_{in}}$ & \\
		\hline
		Linear: & [$d_{in}$, ~~1024] & LeakyReLU \\
		\hline 
		Linear: & [1024, 1024] & LeakyReLU  \\
		\hline
		Linear: & [1024, 1] \\
		\bottomrule[2pt]
	\end{tabular}
	\label{toy_generator_critic}
	\vspace{-10pt}
\end{table}

\subsection{Domain Adaptation}
\label{hyperparameters_da}

For MNIST-to-USPS, USPS-to-MNIST, and MNIST-to-MNISTM: images are resized to 64$\times$64, 
critics are implemented as 6-layer CNN and 
Classifiers are implemented as LeNet-like CNN. 
For SVHN-to-MNIST, images are resized to 32$\times$32, 
critics are implemented as 5-layer CNN and 
classifiers are adopted as the WideResNet \cite{wideresnet}. 
Generators are all implemented as 6-block ResNets \cite{resnet}. 

Table \ref{da_generator_resnet6} shows the network architecture of generators for DA. 
Table \ref{da_discriminator_image64} shows the network architecture of critics for DA. 	
For MNIST-to-USPS and USPS-to-MNIST $s=64$, $c=1$. For MNIST-to-MNISTM $s=64$, $c=3$. For SVHN-MNIST, $s=32$, $c=3$. 	
Table \ref{da_classifier_lenet} shows the network architecture of classifiers for DA except SVHN-to-MNIST, for which we adopt the WideResNet \cite{wideresnet}. 

We set $\lambda_{gan}=1$, $\lambda_{gp}=10$, $n_{critic}=3$, $lr=0.0002$ and batch size as 32. $\lambda_{cycle}$ is typically set within $[100, 1000]$. The independent noise $z$ is sampled from the 10-dimension uniform distribution $\mathcal{U}[-1, 1]$.

\begin{table}
	\vspace{-20pt}
	\centering
	\caption{The network architecture of generators for DA.}
	\vspace{3pt}
	\begin{tabular}{lll}
		\toprule[2pt]
		Input: & $z \in \mathbb{R}^{10}$ &\\
		\hline
		& Parameters & Activation\\
		\hline
		Linear: & [10, $s\times s\times$1] & BN, ReLU \\
		\hline
		Concat: & $x \in \mathbb{R}^{s\times s\times c}$ \\
		\hline 
		Conv: & [$c$+1, 64, ~k=7, s=1, p=3] & IN, ReLU \\
		\hline
		Conv: & [64, ~~128, k=3, s=2, p=1] & IN, ReLU \\
		\hline
		Conv: & [128, ~256, k=3, s=2, p=1] & IN, ReLU \\
		\hline 
		ResBlock: & 6 blocks \\
		\hline 
		Deconv: & [256, 128, k=3, s=2, p=1] & IN, ReLU \\
		\hline
		Deconv: & [128, 64, ~~k=3, s=2, p=1] & IN, ReLU \\
		\hline
		Conv: & [64, ~~$c$, ~~~k=7, s=1, p=3] & Tanh \\
		\bottomrule[2pt]
	\end{tabular}
	\label{da_generator_resnet6}
\end{table}

\begin{table}
	\vspace{-40pt}
	\centering
	\caption{The network architecture of critics for DA. Conv[$s$=64] denotes a Conv layer which exists if $s$=64.}
	\vspace{3pt}
	\begin{tabular}{lll}
		\toprule[2pt]
		Input: & $x \in \mathbb{R}^{s\times s\times c}$ & \\
		\hline
		& Parameters & Activation \\
		\hline
		Conv: & [$c$,~~~~64,~~k=4, s=2, p=1] & LeakyReLU \\
		\hline
		Conv: & [64,~~128, k=4, s=2, p=1] & LeakyReLU \\
		\hline
		Conv: & [128, 256, k=4, s=2, p=1] & LeakyReLU \\
		\hline 
		Conv: & [256, 512, k=4, s=2, p=1] & LeakyReLU \\
		\hline 
		Conv[$s$=64]: & [512, 512, k=4, s=2, p=1] & LeakyReLU \\
		\hline 
		Conv: & [512, 1,~~~~k=4, s=2, p=1] \\
		\bottomrule[2pt]
	\end{tabular}
	\label{da_discriminator_image64}
	\vspace{-10pt}
\end{table}

\begin{table}
	\vspace{-30pt}
	\centering
	\caption{The network architecture of classifiers for DA.}
	\vspace{3pt}
	\begin{tabular}{lll}
		\toprule[2pt]
		Input: & $x \in \mathbb{R}^{64\times 64\times c}$ & \\
		\hline
		& Parameters & Activation \\
		\hline
		Conv: & [$c$,~~~32, k=5, s=1, p=2] & ReLU, MaxPool(2,2) \\
		\hline
		Conv: & [32, 48, k=5, s=1, p=2] & ReLU, MaxPool(2,2) \\
		\hline
		Linear: & [12288, 100] & ReLU \\
		\hline 
		Linear: & [100, 100] & ReLU  \\
		\hline
		Linear: & [100, 10] \\
		\bottomrule[2pt]
	\end{tabular}
	\label{da_classifier_lenet}
	\vspace{-20pt}
\end{table}

\subsection{Image-to-Image Translation}
\label{hyperparameters_image_translation}

Inputs and outputs of all tasks are of size 64$\times$64$\times$3. 	
For Edges-to-Handbags, 
generators adopt the architecture of autoencoder \cite{autoencoder} and 
critics are implemented as 6-layer CNN. 
For Handbags-to-Shoes, 
generators are implemented as 8-block ResNets and 
critics are implemented as 5-block ResNets. 

Table \ref{image2image_translation_generator_ae64} shows the network architecture of the generators for image-to-image translation on edges2handbags, where the architecture of the critics is shows in Table \ref{da_discriminator_image64} with $s=64$.
Table \ref{image2image_translation_generator_snres} and Table \ref{image2image_translation_discriminator_snres} show the network architecture of generators and critics for image-to-image translation on handbags2shoes. The ResBlock is the same as the one in WGAN-GP \cite{wgangp}. 

Other hyperparameters are the same as the ones used for experiments of domain adaptation.

\begin{table}
	\vspace{-20pt}
	\centering
	\caption{The network architecture of generators for image-to-image translation on edges2handbags.}
	\vspace{3pt}
	\begin{tabular}{lll}
		\toprule[2pt]
		Input: & $z \in \mathbb{R}^{10}$ & \\
		\hline
		& Parameters & Activation \\
		\hline
		Linear: & [10, 64 $\times$64$\times$1] & BN, ReLU \\
		\hline
		Concat: & $x \in \mathbb{R}^{64\times64\times3}$ \\
		\hline 
		Conv: & [4,~~~~64,~~k=4, s=2, p=1] & LeakyReLU \\
		\hline
		Conv: & [64,~~128, k=4, s=2, p=1] & IN,LeakyReLU \\
		\hline
		Conv: & [128, 256, k=4, s=2, p=1] & IN,LeakyReLU \\
		\hline 
		Conv: & [256, 512, k=4, s=2, p=1] & IN,LeakyReLU \\
		\hline 
		Conv: & [512, 512, k=4, s=2, p=1] & IN,LeakyReLU \\
		\hline 
		Conv: & [512, 512, k=4, s=2, p=1] & ReLU \\
		\hline 
		Deconv: & [512, 512, k=4, s=2, p=1] & IN,ReLU \\
		\hline
		Deconv: & [512, 512, k=4, s=2, p=1] & IN,ReLU \\
		\hline
		Deconv: & [512, 256, k=4, s=2, p=1] & IN,ReLU \\
		\hline 
		Deconv: & [256, 128, k=4, s=2, p=1] & IN,ReLU \\
		\hline 
		Deconv: & [128, 64,~~k=4, s=2, p=1] & IN,ReLU \\
		\hline 
		Deconv: & [64,~~3,~~~~k=4, s=2, p=1] & Tanh \\
		\bottomrule[2pt]
	\end{tabular}
	\label{image2image_translation_generator_ae64}
\end{table}

\begin{table}
	\vspace{-40pt}
	\centering
	\caption{The network architecture of generators for image-to-image translation on handbags2shoes.}
	\vspace{3pt}
	\begin{tabular}{lll}
		\toprule[2pt]
		Input: & $z \in \mathbb{R}^{10}$ & \\
		\hline
		& Parameters & Activation \\
		\hline
		Linear: & [10, 64 $\times$64$\times$1] & BN, ReLU \\
		\hline
		Concat: & $x \in \mathbb{R}^{64\times64\times3}$ \\
		\hline 
		ResBlock down & channel = 64 \\
		\hline 
		ResBlock down & channel = 128 \\
		\hline 
		ResBlock down & channel = 256 \\
		\hline 
		ResBlock down & channel = 512 & ReLU \\
		\hline 
		ResBlock up & channel = 256 & \\
		\hline 
		ResBlock up & channel = 128 & \\
		\hline 
		ResBlock up & channel = 64 & \\
		\hline 
		ResBlock up & channel = 32 & BN,ReLU \\
		\hline
		Conv: & [32, 3, k=3, s=1, p=1] & Tanh \\
		\bottomrule[2pt]
	\end{tabular}
	\label{image2image_translation_generator_snres}
\end{table}

\begin{table}
	%		\vspace{-10pt}
	\centering
	\caption{The network architecture of critics for image-to-image translation on handbags2shoes.}
	\vspace{3pt}
	\begin{tabular}{lll}
		\toprule[2pt]
		Input: & $x \in \mathbb{R}^{64\times64\times3}$ \\
		\hline 
		& Parameters & Activation \\
		\hline
		ResBlock down & channel = 64 \\
		\hline 
		ResBlock down & channel = 128 \\
		\hline 
		ResBlock down & channel = 256 \\
		\hline 
		ResBlock down & channel = 512 \\
		\hline 
		ResBlock down & channel = 1024 & ReLU \\
		\hline 
		Linear: & [1024, 1] \\
		\bottomrule[2pt]
	\end{tabular}
	\label{image2image_translation_discriminator_snres}
\end{table}

\subsection{Color Transfer}
\label{hyperparameters_color_transfer}

Network Architectures and hyperparameters are the same as the ones in toy experiments except that the dimensions of source sample, target sample and independent noise $z$ are changed to 3. Table \ref{toy_generator_generator} and Table \ref{toy_generator_critic} show the network architectures for generators and critics.

\clearpage

\section{Algorithm for Training Kantorovich Solver/Monge Solver/Optimal Bijection Solver}

\begin{breakablealgorithm}
	\caption{Stochastic Gradient Algorithm for K-solver/M-solver/B-solver}
	\label{training algorithm}
	\begin{algorithmic}[1]
		\REQUIRE Source distribution $\mu$; target distribution $\nu$; independent noise distribution $p(z)$; cost function $c$; generator networks $G_{xy}$, $G_{yx}$ and critic networks $D_{x}$, $D_{y}$ with parameters $\theta_{xy}$, $\theta_{yx}$, $\omega_{x}$ and $\omega_{y}$ respectively; coefficients $\lambda_{cycle_{\mu}}$, $\lambda_{cycle_{\nu}}$, $\lambda_{gan_{xy}}$, $\lambda_{gan_{yx}}$, $\lambda_{gp_{xy}}$, $\lambda_{gp_{yx}}$; parameters of Adam $\alpha$, $\beta_1$, $\beta_2$; batch size $m$; number of critic iterations per generator iteration $n_{critic}$
		\WHILE{not converged}
		\FOR{$t=1$ \textbf{to} $n_{critic}$}
		\FOR{$i=1$ \textbf{to} $m$}
		\STATE sample $x$, $y$, $z_x$, $z_y$, $\epsilon_x$, $\epsilon_y$ from $\mu$,  $\nu$, $p(z)$, $p(z)$, $U[0,1]$, $U[0,1]$ respectively
		\STATE $y' \gets G_{xy}(x, z_x)$, $x' \gets G_{yx}(y, z_y)$
		\STATE $\tilde{y} \gets \epsilon_y y+(1-\epsilon_y)y'$, $\tilde{x} \gets \epsilon_x x+(1-\epsilon_x)x'$
		\STATE $L_i \gets D_{y}(y') - D_{y}(y) + \lambda_{gp_{xy}} (\|\nabla_{\tilde{y}}D_{y}(\tilde{y}) \|_2 - 1)^2 + D_{x}(x') - D_{x}(x) + \lambda_{gp_{yx}} (\|\nabla_{\tilde{x}}D_{x}(\tilde{x}) \|_2 - 1)^2$
		\ENDFOR
		\STATE $\omega_{y} \gets$ Adam$(\nabla_{\omega_y} \frac{1}{m} \sum_{i=1}^m L_i, \alpha, \beta_1, \beta_2)$
		\STATE $\omega_{x} \gets$ Adam$(\nabla_{\omega_x} \frac{1}{m} \sum_{i=1}^m L_i, \alpha, \beta_1, \beta_2)$
		\ENDFOR
		\FOR{$i=1$ \textbf{to} $m$}
		\STATE sample $x$, $y$, $z_x$, $z_y$ from $\mu$,  $\nu$, $p(z)$, $p(z)$ respectively
		\STATE $y' \gets G_{xy}(x, z_x)$, $x' \gets G_{yx}(y, z_y)$
		\STATE $\hat{y} \gets G_{xy}(x',z_x)$, $\hat{x} \gets G_{yx}(y', z_y)$
		\STATE $L_i \gets c(x, y') + \lambda_{cycle_{\nu}}\|\hat{y} - y\|_2 + \lambda_{cycle_{\mu}}\|\hat{x} - x\|_2 - \lambda_{gan_{xy}} D_{y}(y') - \lambda_{gan_{yx}} D_{x}(x')$
		\ENDFOR
		\STATE $\theta_{xy} \gets$ Adam$(\nabla_{\theta_{xy}} \frac{1}{m} \sum_{i=1}^m L_i, \alpha, \beta_1, \beta_2)$
		\STATE $\theta_{yx} \gets$ Adam$(\nabla_{\theta_{yx}} \frac{1}{m} \sum_{i=1}^m L_i, \alpha, \beta_1, \beta_2)$
		\ENDWHILE
	\end{algorithmic}
\end{breakablealgorithm}

\newpage
\section{Algorithm for Domain Adaptation}
\label{algorithm_da}

\begin{breakablealgorithm}
	\caption{Stochastic Gradient Algorithm for Domain Adaptation}
	\label{training algorithm for DA}
	\begin{algorithmic}[1]
		\REQUIRE Datasets $\{\left< x_i, l_i \right>\}_{i=1}^{n_x}$, $\{\left< y_i \right>\}_{i=1}^{n_y}$; independent noise distribution $p(z)$; generator networks $G_{xy}$, $G_{yx}$, classifier networks $C_{x}$, $C_{y}$ and critic networks $D_{x}$, $D_{y}$ with parameters $\theta_{xy}$, $\theta_{yx}$, $\psi_x$, $\psi_y$, $\omega_{x}$ and $\omega_{y}$ respectively; coefficients $\lambda_{cycle_{\mu}}$, $\lambda_{cycle_{\nu}}$, $\lambda_{gan_{xy}}$, $\lambda_{gan_{yx}}$, $\lambda_{gp_{xy}}$, $\lambda_{gp_{yx}}$; parameters of Adam $\alpha$, $\beta_1$, $\beta_2$; batch size $m$; number of critic iterations of per generator iteration $n_{critic}$
		\WHILE{not converged}
		\FOR{$i=1$ \textbf{to} $m$}
		\STATE sample $\left< x, l \right>$ from $\{\left< x_i, l_i \right>\}_{i=1}^{n_x}$
		\STATE $L_i = \mathcal{H}(C_{x}(x), l)$
		\ENDFOR
		\STATE $\psi_x \gets$ Adam $(\frac{1}{m} \sum_{i=1}^m L_i, \alpha, \beta_1, \beta_2)$
		\ENDWHILE
		\STATE
		
		\WHILE{not converged}
		\FOR{$t=1$ \textbf{to} $n_{critic}$}
		\FOR{$i=1$ \textbf{to} $m$}
		\STATE sample $\left< x, l \right>$, $y$, $z_x$, $z_y$, $\epsilon_x$, $\epsilon_y$ from $\{\left< x_i, l_i \right>\}_{i=1}^{n_x}$, $\{\left< y_i \right>\}_{i=1}^{n_y}$, $p(z)$, $p(z)$, $U[0,1]$, $U[0,1]$ respectively
		\STATE $y' \gets G_{xy}(x, z_x)$, $\tilde{y} \gets \epsilon_y y+(1-\epsilon_y)y'$
		\STATE $x' \gets G_{yx}(y, z_y)$, $\tilde{x} \gets \epsilon_x x+(1-\epsilon_x)x'$
		\STATE $L_i \gets D_{y}(y') - D_{y}(y) + \lambda_{gp_{xy}} (||\nabla_{\tilde{y}}D_{y}(\tilde{y}) ||_2 - 1)^2 + D_{x}(x') - D_{x}(x) + \lambda_{gp_{yx}} (||\nabla_{\tilde{x}}D_{x}(\tilde{x}) ||_2 - 1)^2$
		\ENDFOR
		\STATE $\omega_{y} \gets$ Adam$(\nabla_{\omega_y} \frac{1}{m} \sum_{i=1}^m L_i, \alpha, \beta_1, \beta_2)$
		\STATE $\omega_{x} \gets$ Adam$(\nabla_{\omega_x} \frac{1}{m} \sum_{i=1}^m L_i, \alpha, \beta_1, \beta_2)$
		\ENDFOR
		\FOR{$i=1$ \textbf{to} $m$}
		\STATE sample $\left< x, l \right>$, $y$, $z_x$, $z_y$ from $\{\left< x_i, l_i \right>\}_{i=1}^{n_x}$, $\{\left< y_i \right>\}_{i=1}^{n_y}$, $p(z)$, $p(z)$ respectively
		\STATE $y' \gets G_{xy}(x, z_x)$, $x' \gets G_{yx}(y, z_y)$
		\STATE $\hat{y} \gets G_{xy}(x', z_x)$, $\hat{x} \gets G_{yx}(y', z_y)$
		\STATE $L_i \gets\mathcal{H}(C_x(y'), l) + \lambda_{cycle_{\nu}}||\hat{y} - y||_2 + \lambda_{cycle_{\mu}}||\hat{x} - x||_2 - \lambda_{gan_{xy}} D_{y}(y') - \lambda_{gan_{yx}} D_{x}(x')$
		\ENDFOR
		\STATE $\theta_{xy} \gets$ Adam$(\nabla_{\theta_{xy}} \frac{1}{m} \sum_{i=1}^m L_i, \alpha, \beta_1, \beta_2)$
		\STATE $\theta_{yx} \gets$ Adam$(\nabla_{\theta_{yx}} \frac{1}{m} \sum_{i=1}^m L_i, \alpha, \beta_1, \beta_2)$
		\ENDWHILE
		\STATE
		
		\WHILE{not converged}
		\FOR{$t=1$ \textbf{to} $n_{critic}$}
		\FOR{$i=1$ \textbf{to} $m$}
		\STATE sample $\left< x, l \right>$, $y$, $z_x$, $z_y$, $\epsilon_x$, $\epsilon_y$ from $\{\left< x_i, l_i \right>\}_{i=1}^{n_x}$, $\{\left< y_i \right>\}_{i=1}^{n_y}$, $p(z)$, $p(z)$, $U[0,1]$, $U[0,1]$ respectively
		\STATE $y' \gets G_{xy}(x, z_x)$, $\tilde{y} \gets \epsilon_y y+(1-\epsilon_y)y'$
		\STATE $x' \gets G_{yx}(y, z_y)$, $\tilde{x} \gets \epsilon_x x+(1-\epsilon_x)x'$
		\STATE $L_i \gets D_{y}(y') - D_{y}(y) + \lambda_{gp_{xy}} (||\nabla_{\tilde{y}}D_{y}(\tilde{y}) ||_2 - 1)^2 + D_{x}(x') - D_{x}(x) + \lambda_{gp_{yx}} (||\nabla_{\tilde{x}}D_{x}(\tilde{x}) ||_2 - 1)^2$
		\ENDFOR
		\STATE $\omega_{y} \gets$ Adam$(\nabla_{\omega_y} \frac{1}{m} \sum_{i=1}^m L_i, \alpha, \beta_1, \beta_2)$
		\STATE $\omega_{x} \gets$ Adam$(\nabla_{\omega_x} \frac{1}{m} \sum_{i=1}^m L_i, \alpha, \beta_1, \beta_2)$
		\ENDFOR
		\FOR{$i=1$ \textbf{to} $m$}
		\STATE sample $\left< x, l \right>$, $y$, $z_x$, $z_y$ from $\{\left< x_i, l_i \right>\}_{i=1}^{n_x}$, $\{\left< y_i \right>\}_{i=1}^{n_y}$, $p(z)$, $p(z)$ respectively
		\STATE $y' \gets G_{xy}(x, z_x)$, $x' \gets G_{yx}(y, z_y)$
		\STATE $\hat{y} \gets G_{xy}(x', z_x)$, $\hat{x} \gets G_{yx}(y', z_y)$
		\STATE $L_i \gets\mathcal{H}(C_y(y'), l) + \lambda_{cycle_{\nu}}||\hat{y} - y||_2 + \lambda_{cycle_{\mu}}||\hat{x} - x||_2 - \lambda_{gan_{xy}} D_{y}(y') - \lambda_{gan_{yx}} D_{x}(x')$
		\ENDFOR
		\STATE $\theta_{xy} \gets$ Adam$(\nabla_{\theta_{xy}} \frac{1}{m} \sum_{i=1}^m L_i, \alpha, \beta_1, \beta_2)$
		\STATE $\theta_{yx} \gets$ Adam$(\nabla_{\theta_{yx}} \frac{1}{m} \sum_{i=1}^m L_i, \alpha, \beta_1, \beta_2)$
		\STATE $\psi_y \gets$ Adam$(\nabla_{\psi_{y}} \frac{1}{m} \sum_{i=1}^m L_i, \alpha, \beta_1, \beta_2)$
		\ENDWHILE
	\end{algorithmic}
\end{breakablealgorithm}

\end{document}